\documentclass[journal]{IEEEtran}
\usepackage[T1]{fontenc}
\usepackage{cite}
\usepackage{amsmath,amssymb,amsfonts}
\usepackage{algorithmic}
\usepackage{algorithm}
\usepackage{graphicx}
\usepackage{textcomp}
\usepackage{xcolor}
\usepackage{color}
\usepackage{amsthm}
\usepackage{cite}
\usepackage{caption}
\usepackage{booktabs}
\usepackage{makecell}
\usepackage{threeparttable}
\usepackage{subfig}
\usepackage[numbers,sort&compress]{natbib}
\ifCLASSINFOpdf
\else
\fi
\begin{document}
%
\title{Decentralized Federated Learning: Balancing Communication and Computing Costs}
%
%
%

\author{Wei Liu, Li Chen, and Wenyi Zhang,~\IEEEmembership{Senior Member,~IEEE}
}
\maketitle

\begin{abstract}
	Decentralized stochastic gradient descent (SGD) is a driving engine for decentralized federated learning (DFL). The performance of decentralized SGD is jointly influenced by inter-node communications and local updates. In this paper, we propose a general DFL framework, which implements both multiple local updates and multiple inter-node communications periodically, to strike a balance between communication efficiency and model consensus. 
	It can provide a general decentralized SGD analytical framework. We establish strong convergence guarantees for the proposed DFL algorithm without the assumption of convex objectives. The convergence rate of DFL can be optimized to achieve the balance of communication and computing costs under constrained resources. For improving communication efficiency of DFL, compressed communication is further introduced to the proposed DFL as a new scheme, named DFL with compressed communication (C-DFL). The proposed C-DFL exhibits linear convergence for strongly convex objectives. Experiment results based on MNIST and CIFAR-10 datasets illustrate the superiority of DFL over traditional decentralized SGD methods and show that C-DFL further enhances communication efficiency. 
\end{abstract}

\begin{IEEEkeywords}
	Compressed communication, decentralized machine learning, federated learning.
\end{IEEEkeywords}

%
\IEEEpeerreviewmaketitle

\section{Introduction}

Recently, thanks to the explosive growth of the powerful individual computing devices worldwide and the rapid advancement of Internet of Things (IoT), data generated at device terminals are experiencing an exponential increase \cite{chiang2016fog}. Data-driven machine learning is developing into an attractive technique, which makes full use of the tremendous data for predictions and decisions of future events. As a promising
data-driven machine learning variant, Federated Learning (FL) provides a communication-efficient approach for processing voluminous distributed data and grows in popularity nowadays. 

As a critical technique of distributed machine learning, FL was first proposed in \cite{mcmahan2017communication} under a centralized form, where edge clients perform local model training in parallel and a central server aggregates the trained model parameters from the edge without transmission of raw data from edge clients to the central server. For convergence analysis of FL, the work in \cite{wang2019adaptive} provided a theoretical convergence bound of FL based on gradient-descent. The area of FL has been studied in the literature from both theoretical and applied perspectives \cite{kairouz2019advances,li2020federated,konevcny2016federated}. Because of the property of transmitting model parameters instead of user data and the distributed network structure that an arbitrary number of edge nodes are coordinated through one central server, 
FL plays a key
role in supporting privacy protection of user data and deploying in complex environment with massive intelligent terminals access to the network center.

In order to improve learning performance with guaranteeing communication efficiency of FL,
extended methods for executing FL have been proposed and developed. As a fundamental means of achieving FL, parallel SGD has been studied and developed in previous works \cite{dean2012large} \cite{lian2015asynchronous}.
In original parallel SGD \cite{yu2019parallel}, each round consists of one local update (computing) step and one global aggregation (communication) step. 
To improve communication efficiency of distributed learning,
Federated Averaging (FedAvg), as a variant of parallel SGD, was proposed in \cite{mcmahan2017communication}. Multiple local updates are executed on each client between two communication steps in FedAvg. 
FedAvg was generalized to non-i.i.d. cases  \cite{li2019convergence,sattler2019robust,wang2020optimizing}; momentum FL was proposed for convergence acceleration \cite{liu2020accelerating}; FedAvg with a Proximal Term (FedProx) was studied for handling heterogeneous federated learning environment \cite{li2018federated}; federated multi-task learning was proposed to improve learning performance under unbalance data \cite{smith2017federated}; robust FL handles the noise effect during model transmission in wireless communication \cite{ang2020robust}.

The aforementioned centralized FL framework requires a central server for data aggregation. As an aggregation center of information from clients, the server could potentially incur a single point of failure, such as suffering from high bandwidth and device battery costs \cite{vanhaesebrouck2017decentralized}. It may also become the busiest node incurring a high communication cost when a large number of clients are deployed in a federated learning system \cite{lian2017can}. 
When there is no central server, the DFL paradigm is rapidly gaining popularity, where edge clients exchange their local models or gradient information with their neighboring clients to achieve model consensus. This avoids a single point of failure, and reduces the communication traffic jam due to the central server \cite{hu2019decentralized}.
DFL can be traced back to decentralized optimization \cite{nedic2009distributed,wu2017decentralized,wei2012distributed,li2019communication,duchi2011dual,yang2019byrdie,zhang2021newton,lu2020computation} and decentralized SGD \cite{yuan2016convergence,sirb2018decentralized,lian2017can,jiang2017collaborative,wang2019cooperative,li2019communicationn} where gossip-based model averaging was applied for model consensus evaluation \cite{tsianos2012communication}.
For decentralized optimization, subgradient descent method \cite{nedic2009distributed}, decentralized consensus optimization with asynchrony and delays \cite{wu2017decentralized}, alternating direction method of multipliers (ADMM) \cite{wei2012distributed} \cite{li2019communication}, dual averaging method \cite{duchi2011dual}, Byzantine-resilient distributed
coordinate descent \cite{yang2019byrdie} and Newton tracking method \cite{zhang2021newton} were studied to solve the decentralized optimization problem, and convex optimization problem with a sum of smooth convex functions and a non-smooth regularization term was solved in \cite{lu2020computation}.
As a driving engine, decentralized SGD provides a means of realizing DFL. A convergence analysis of decentralized gradient descent was provided
in \cite{yuan2016convergence}; the convergence of decentralized approach with delayed gradient information was analyzed \cite{sirb2018decentralized}; a comparison between decentralized SGD and centralized ones was made, illustrating the advantage of decentralized approaches \cite{lian2017can}.
Decentralized SGD performing one local update step and one inter-node communication step in a round (D-SGD) was proposed and analyzed in \cite{jiang2017collaborative}. 
In order to improve communication efficiency of decentralized SGD, the work in \cite{wang2019cooperative} first proposed a framework named Cooperative SGD (C-SGD), where the edge clients perform multiple local updates before exchanging the model parameters.
The subsequent work in \cite{li2019communicationn} studied how to efficiently combine local updates and decentralized SGD based on the decentralized framework by performing several local updates and several decentralized SGDs. 

Note that
decentralized SGD in the literature mainly focuses on improving communication efficiency by performing multiple local update steps and one inter-node communication step in a round \cite{wang2019cooperative}.  Besides communication efficiency, model consensus also has an important influence on the convergence performance of DFL. Model consensus, which means the removal of discrepancy of model parameter averagings among neighboring nodes, improves the convergence performance of DFL by executing multiple inter-node communications in the DFL framework. However, to the best of our knowledge, a joint consideration of communication efficiency and model consensus in DFL framework has not been addressed. Only considering communication efficiency incurs inferior model consensus, and vice versa. 
Motivated by the above observation, we consider a unified DFL framework where each node performs \textit{multiple} local updates and \textit{multiple} inter-node communications in a round alternately. The proposed framework can balance communication efficiency and model consensus with convergence guarantee.
We prove the global convergence and derive a convergence bound of the proposed DFL, which indicates that multiple inter-node communications effectively improve the convergence performance. 
In order to improve communication efficiency of DFL, we further introduce compression into DFL, and the DFL framework with compressed communication is termed C-DFL. The convergence rate of C-DFL is derived and linear convergence is exhibited for strongly convex objectives.   
The main contributions of this paper are summarized as follows:
\begin{itemize} 
	\item \textbf{Design of DFL:} To consider the joint influence of communication efficiency and model consensus, we design a novel DFL framework to balance the allocation of communication and computation resources for optimizing the convergence performance. In DFL, each node performs multiple local update steps and multiple inter-node communication steps in a round. 
	
	\item \textbf{Convergence Analysis of DFL:} We establish the convergence of DFL and derive a convergence bound that incorporates network topology and the allocation  of computation and communication steps in a round. It shows that the convergence performance of DFL outperforms that of traditional frameworks.
	
	\item \textbf{C-DFL Design and Convergence:} In order to improve communication efficiency of DFL, we further introduce compressed communication to DFL as a new scheme, termed C-DFL, where inter-node communication is achieved with compressed information to reduce the communication overhead.
	For C-DFL, the convergence performance is based on the synergistic influence of compression and allocation of communication and computation steps in a round. We provide a convergence analysis of C-DFL and establish its linear convergence property.
	
	
\end{itemize} 
\begin{figure}[!t]
	\centering
	\includegraphics[scale=0.3]{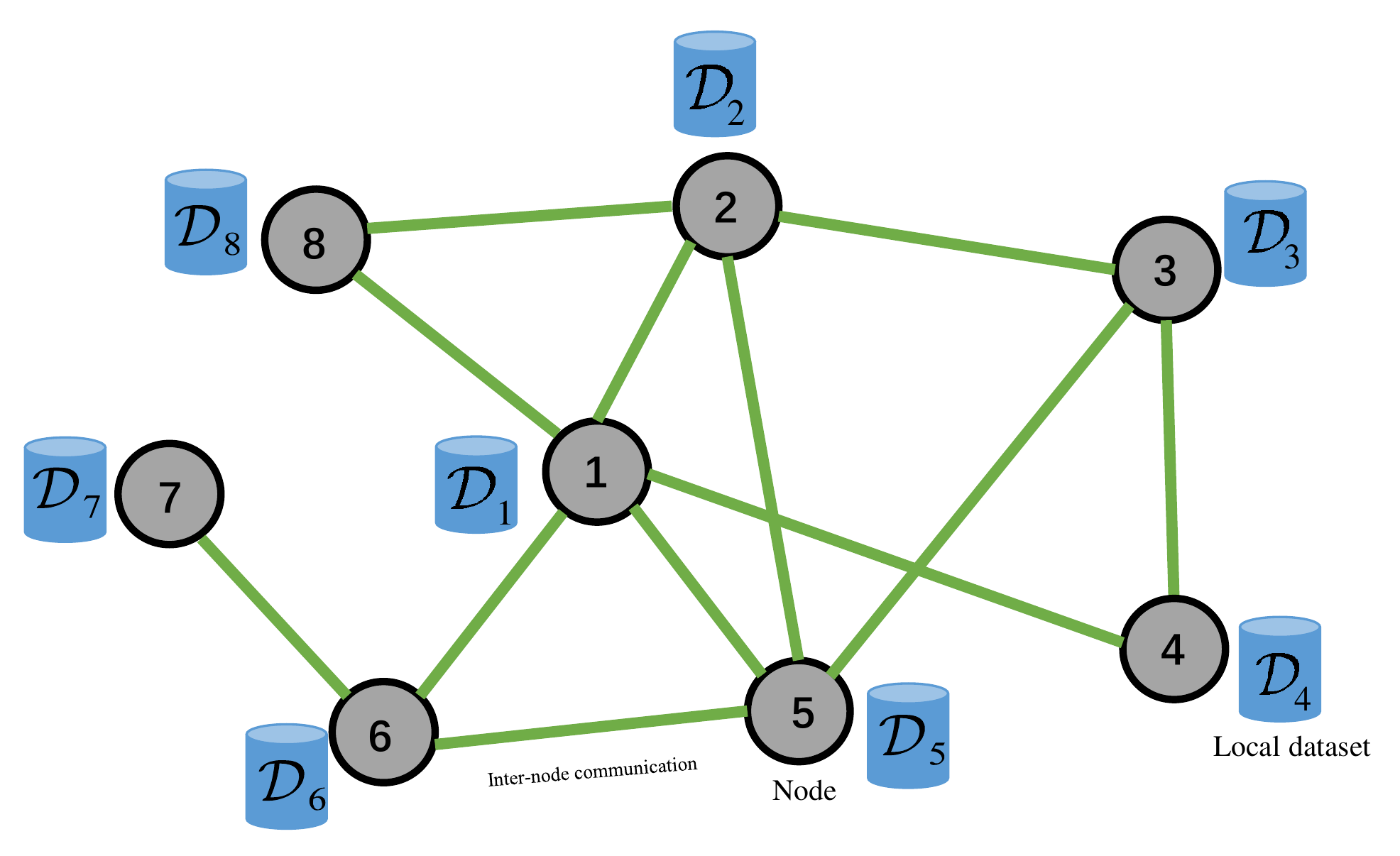}        
	\caption{System model of decentralized federated learning}
	\label{fig1}
\end{figure}

The remaining part of this paper is organized as follows. We introduce the system model for solving the distributed optimization problem in DFL in Section \uppercase\expandafter{\romannumeral2}. 
The design of DFL is described in detail in Section \uppercase\expandafter{\romannumeral3} and the convergence analysis of DFL is presented in Section \uppercase\expandafter{\romannumeral4}. Then we introduce C-DFL and present its convergence analysis in Section \uppercase\expandafter{\romannumeral5}. We present simulation results in Section \uppercase\expandafter{\romannumeral6}. The conclusion is given in Section \uppercase\expandafter{\romannumeral7}.

\subsection{Notations}
In this paper, we use $\textbf{1}$ to denote vector $\left[1, 1, ..., 1\right]^\top$, and define
the consensus matrix $\textbf{J}\triangleq \textbf{1}\textbf{1}^\top/(\textbf{1}^\top \textbf{1})$, which means under the topology represented by $\textbf{J}$, DFL can realize model consensus. All vectors in this paper are column vectors. We assume that the decentralized network contains $N$ nodes.
So we have $\textbf{1}\in \mathbb{R}^N$ and $\textbf{J}\in \mathbb{R}^{N\times N}$. 
We use $\left\|\cdot\right\|$, $\rm\left\|\cdot\right\|_{F}$, $\rm\left\|\cdot\right\|_{op}$ and $\|\cdot\|_2$ to denote the $\ell_2$ vector norm, the Frobenius matrix norm, the operator norm and the $\ell_2$ matrix norm, respectively.

\section{System Model}
Considering a general system model as illustrated in Fig. \ref{fig1}, 
we describe the decentralized federated learning structure. The system model is made of 
$N$ edge nodes.
The $N$ nodes have distributed datasets $\mathcal{D}_1,\mathcal{D}_2,...,\mathcal{D}_i,...,\mathcal{D}_N$ with $\mathcal{D}_i$ owned by node $i$. 
We use $\mathcal{D}=\{\mathcal{D}_1, \mathcal{D}_2,
..., \mathcal{D}_N\}$ to denote the global dataset of all nodes.
Assuming $\mathcal{D}_i\cap\mathcal{D}_j=\emptyset$ for $i\neq j$, we define $D\triangleq|\mathcal{D}|$ and $D_i\triangleq|\mathcal{D}_i|$ for $i=1,2,...,N$, where $|\cdot|$ denotes the size of set, and use $\textbf{w}\in\mathbb{R}^{d}$
to denote the model parameter. 

Then we introduce the loss function of the decentralized network as follows. Let $(\textbf{x}_j,y_j)$ denote a data sample, where $\textbf{x}_j$ is regarded as an input of the learning model and $y_j$ is the label of the input $\textbf{x}_j$, which denotes the desired output of the model. 
The same learning models are embedded in all different nodes. 
We use $f(\cdot)$ to denote the loss function based on one data sample. 
Furthermore, we
present the definition of loss function based on local dataset $\mathcal{D}_i$ and the global dataset $\mathcal{D}$, which are denoted by $F_i(\cdot)$ and $F(\cdot)$, respectively.
We define
\begin{align*}
&F_i(\textbf{w})\triangleq\frac{1}{D_i}\sum_{j\in D_i}f(\textbf{w},\textbf{x}_j,y_j),\\
&F(\textbf{w})\triangleq\frac{1}{D}\sum_{i=1}^{N}D_i F_i(\textbf{w}).
\end{align*}

In order to learn about the characteristic information hidden in the distributed datasets,
the purpose of DFL is to obtain optimal model parameters of all nodes. 
We present the learning problem of DFL. 
The targeted optimization problem is the minimization of the global loss function to get the optimal parameter $\textbf{w}^*$ as follows:
\begin{align}
\label{min_problem}\textbf{w}^*=\mathop {\arg \min}_{\textbf{w}\in\mathbb{R}^d} F(\textbf{w}).
\end{align}
Because of the intrinsic complexity of most machine learning models and common training datasets, finding a closed-form solution of the optimization problem \eqref{min_problem} is usually intractable.
Thus, gradient-based methods are generally used to solve this problem.
We let $\textbf{w}_t$ denote the model parameter at the $t$-th iteration step, $\eta$ denote the learning rate, $\xi^{(i)}_{t}\subset\mathcal{D}_i$ denote the mini-batches randomly sampled on $\mathcal{D}_i$, and $$g(\textbf{w},\xi^{(i)})\triangleq\frac{1}{|\xi^{(i)}|}\sum_{(\textbf{x}_j,y_j)\in\xi^{(i)}}\nabla f(\textbf{w},\textbf{x}_j,y_j)$$ be the stochastic gradient of node $i$.
The iteration of a gradient-based method can be expressed as
\begin{align}
\label{update_rule}\textbf{w}_{t+1}=\textbf{w}_{t}-\eta\left[\frac{1}{N}\sum_{i=1}^{N}g(\textbf{w}_t,\xi^{(i)}_t)\right].	
\end{align}
We will use $g(\textbf{w})$ to substitute $g(\textbf{w},\xi^{(i)})$ in the remaining part of the paper for convenience.

Finally, the targeted optimization problem can be collaboratively solved by local updates and model averagings. In a local update step, each node uses the gradient-based method to update its local model.
In a model averaging step, each node communicates its local model parameters with its neighboring nodes in parallel. 


\section{Design of DFL}\label{section3}
In this section, we introduce the design of DFL to solve the optimization problem as presented in \eqref{min_problem}. Note that there is no central server in the system model of DFL as shown in Fig. \ref{fig1}. We first elaborate the motivation of our work. Then we present
the design of DFL in detail. Finally, we introduce D-SGD and C-SGD as the special cases of DFL.

\begin{figure}[!t]
	\centering
	\includegraphics[scale=0.3]{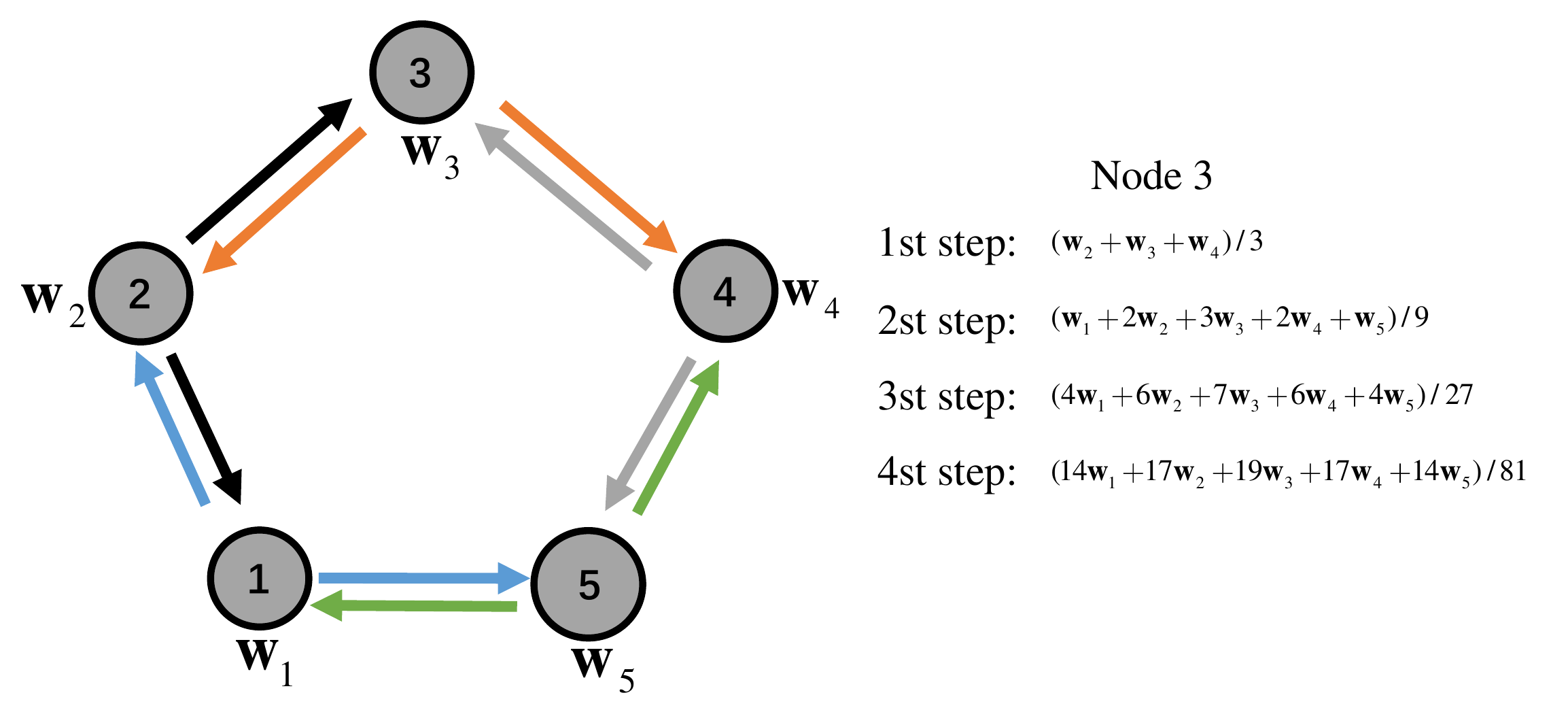}        
	\caption{Illustration of DFL with 5 nodes. Focusing on Node $3$, we can find its model parameter evolution with communication steps.}
	\label{fig2}
\end{figure}
\begin{figure}[!t]
	\centering
	\includegraphics[scale=0.3]{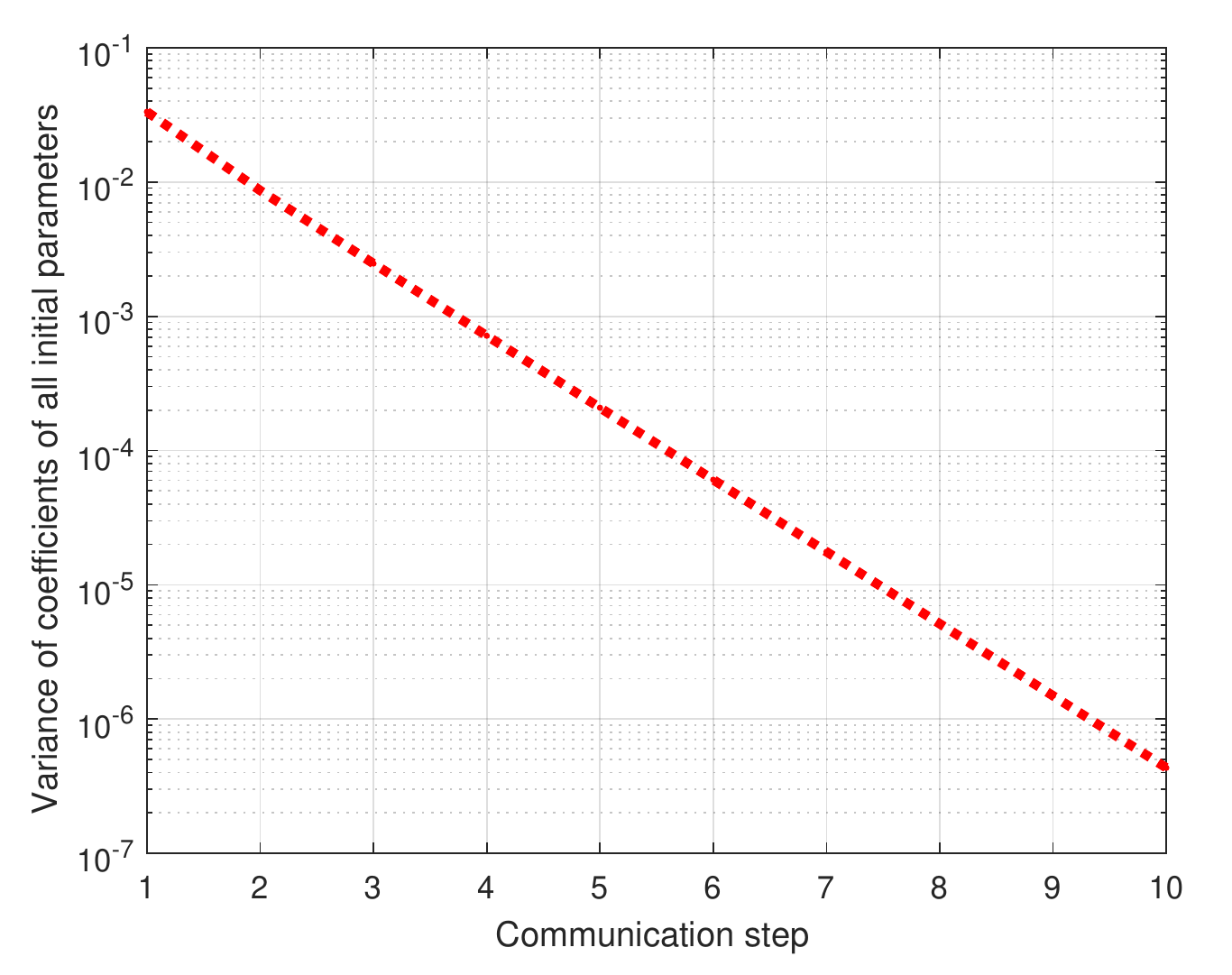}        
	\caption{Illustration of how the variance of coefficients of all initial parameters decreases monotonically with communication step at Node $3$ in Fig. 2.}
	\label{fig3}
\end{figure}
\subsection{Motivation}

Considering the system model,
D-SGD \cite{lian2017can} \cite{jiang2017collaborative}, where each edge node performs one local update and one inter-node communication alternately in a round, was proposed to solve the problem in \eqref{min_problem}. It will incur a low communication efficiency.
To improve communication efficiency of the model, C-SGD \cite{wang2019cooperative} was proposed, where each node performs multiple local updates and one inter-node communication. 
However, in C-SGD, communication efficiency is merely considered by performing multiple local updates in a round. It can incur local drift and inferior model consensus because successive gradient updates in a round make the descent direction locally optimal (but not globally optimal). Furthermore, merely considering model consensus by performing multiple inter-node communications will lead to a high communication cost. 
Therefore,
the existing works do not jointly consider communication efficiency and model consensus on decentralized learning system. We propose to
consider that each node performs multiple inter-node communications between two stages of successive local update steps.

Motivated by the above observation, we describe a general framework, termed DFL, where each node performs multiple local updates and multiple inter-node communications in a round. 
DFL can balance communication efficiency and model consensus by collaboratively allocating computing and communication steps. 
We illustrate the effectiveness of multiple inter-node communications in model consensus as follows.


Without loss of generality,  we set a simplified decentralized network for illustration. As shown in Fig. \ref{fig2}, the model contains five nodes which form a ring topology. After local updates, all nodes get their initial parameters $\textbf{w}_1, ..., \textbf{w}_5$. 
Fig. \ref{fig3} shows the evolution trajectory of coefficients variance of initial parameters with communication steps at Node $3$. We find that the variance monotonically decreases with communication steps, which means more communications better approach the average model and enhance the model consensus. A theoretical explanation will be provided in Proposition \ref{pro1}. 



\begin{figure}[!t]
	\centering
	\includegraphics[scale=0.28]{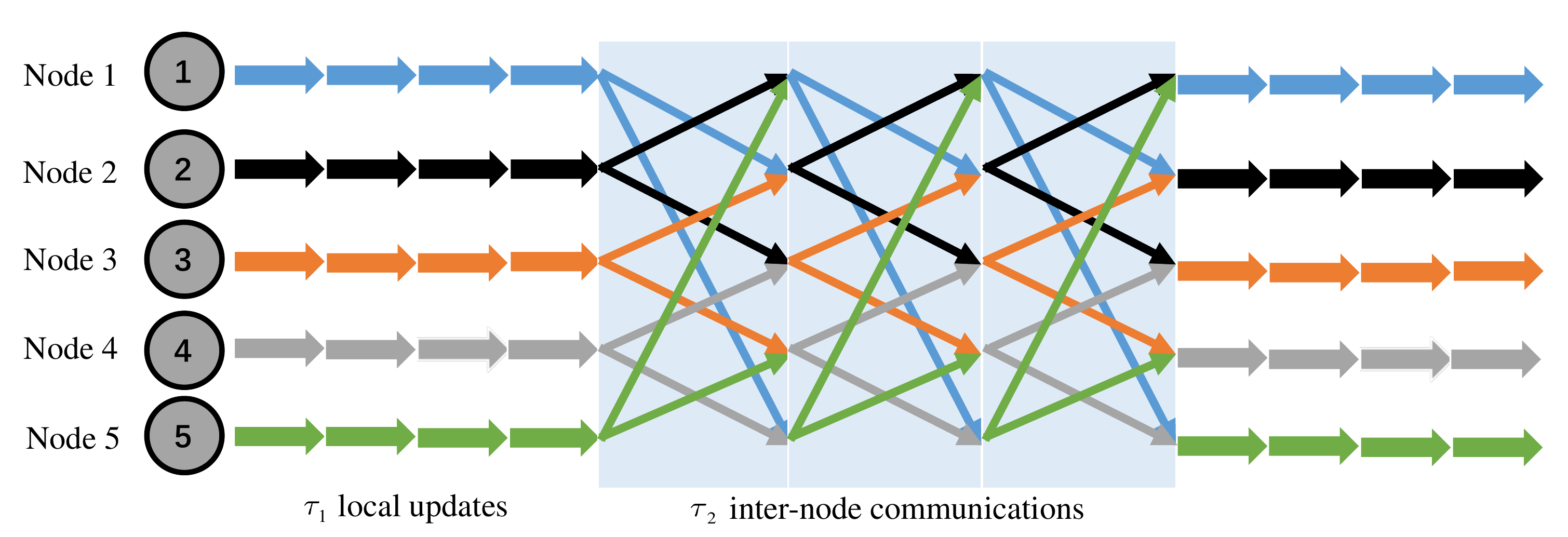}        
	\caption{Illustration of DFL learning strategy. In the decentralized model with 5 nodes in Fig. 2, DFL performs $\tau_2$ inter-node communication steps after $\tau_1$ local update steps.}
	\label{fig4}
\end{figure}

\subsection{DFL}

The DFL learning strategy consists of two stages, which are local update and inter-node communication (communication with neighboring nodes), respectively. Firstly, in local update stage, each node computes the current gradient of its local loss function and updates the model parameter by performing SGD multiple times in parallel.  After finishing local update, each node performs multiple inter-node communications. In inter-node communication stage, each node sends its local model parameters to the nodes connected with it, and receives the parameters from its neighboring nodes. Then model averaging based on the received parameters is performed by the node to obtain the updated parameters. The whole learning process is the alternation of the two stages.

We define $\textbf{C}\in\mathbb{R}^{N\times N}$ as the confusion matrix which captures the network topology, and it is doubly stochastic, i.e., $\textbf{C}\textbf{1}=\textbf{1},\textbf{C}^{\top}=\textbf{C}$. Its element
$c_{ji}$ denotes the contribution of node $j$ in model averaging at node $i$. 
We let DFL perform $\tau_1$ local updates in parallel at each node during the local updates stage. After $\tau_1$ local updates, each node needs to perform  $\tau_2$ inter-node communications. We call $\tau_1$ \textit{the computation frequency} and $\tau_2$ \textit{the communication frequency} in the sequel.
In Fig. \ref{fig4}, we illustrate the learning strategy of DFL. 
Then we define $\tau\triangleq\tau_1+\tau_2$ in the DFL framework and we name the combination of $\tau_1$ local updates and $\tau_2$ inter-node communications as an iteration \textit{round}. Thus the $k$-th iteration round is $[(k-1)\tau, k\tau)$ for $k=1, 2, ...$. We divide an iteration round into $[(k-1)\tau, (k-1)\tau+\tau_1)$
and $[(k-1)\tau+\tau_1, k\tau)$ and define $[k]_1\triangleq[(k-1)\tau, (k-1)\tau+\tau_1)$ and $[k]_2\triangleq[(k-1)\tau+\tau_1, k\tau)$ for convenience.
Thus round $[k]_1$ and $[k]_2$ correspond to the iterative period of local update and inter-node communication, respectively.
We use $\textbf{w}^{(i)}_{t}$ and $g(\textbf{w}^{(i)}_{t})$ to denote the model parameter and the gradient of node $i$ at the $t$-th iteration step, respectively.

\newcommand{\tabincell}[2]{
	\begin{tabular}{@{}#1@{}}
		#2
	\end{tabular}
}
\begin{table*}[!t]  
	\caption{Comparison of different distributed SGD methods}
	\centering
	\label{table-comparison}
	\renewcommand\arraystretch{1.5}
	\begin{threeparttable}
		\begin{tabular}{ccccc}  
			
			\toprule[1pt]   
			
			\textbf{Method} & \tabincell{c}{\textbf{Steps in a round}\\ \textbf{(local update steps, inter-node communication steps)}} & \tabincell{c}{\textbf{Local drift due to local}\\ \textbf{computation}} & \tabincell{c}{\textbf{Local drift due to inter-node }\\ \textbf{communication}} & \textbf{Central server} \\  
			
			\midrule   
			
			FL \cite{wang2019adaptive}&  ($\tau$, Null)\tnote{1} & \multicolumn{1}{c}{High} & \multicolumn{1}{c}{Null\tnote{2}}&Required\\  
			
			D-SGD \cite{jiang2017collaborative}&  ($1$, $1$)  & \multicolumn{1}{c}{Low} & \multicolumn{1}{c}{High}&Not needed\\    
			
			C-SGD \cite{wang2019cooperative}& ($\tau$, $1$) & \multicolumn{1}{c}{High}& \multicolumn{1}{c}{High}&Not needed\\
			
			DFL & ($\tau_1$, $\tau_2$)& \multicolumn{1}{c}{High}& \multicolumn{1}{c}{Low}&Not needed\\
			
			\bottomrule[1pt]  
			
		\end{tabular}
		\begin{tablenotes}
			\footnotesize
			\item[1] In FL, there is no inter-node communication. A central parameter server is used for collecting local models.
			\item[2] FL achieves model consensus by taking a weighted average of all local models. Thus, local drift from inter-node communication dose not exist.
		\end{tablenotes}
	\end{threeparttable}
\end{table*}

\textbf{The learning strategy} of DFL can be described as follows.

\textbf{Local update:} When $t\in[k]_1$, local update is performed at each node in parallel. At node $i$, the updating rule can be expressed as 
$$\textbf{w}_{t+1}^{(i)}=\textbf{w}_{t}^{(i)}-\eta g(\textbf{w}_{t}^{(i)}),$$
where $\eta$ is the learning rate.
Note that each node performs SGD individually based on its local dataset.

\textbf{Inter-node communication:} When $t\in[k]_2$, inter-node communication is performed. Each node 
communicates with the connected nodes to exchange model parameters. After the parameters from all the connected nodes are received, model averaging
is performed by
$$\textbf{w}_{t+1}^{(i)}=\sum_{j=1}^{N}c_{ji}\textbf{w}_{t}^{(j)}.$$

For convenience, we rewrite the learning strategy into matrix form.
We first define \textit{the model parameter matrix} $\textbf{X}_t$ and \textit{the gradient matrix} $\textbf{G}_t\in \mathbb{R}^{d\times N}$ as follows:
\begin{align*}
\textbf{X}_t&=\left[\textbf{w}^{(1)}_{t}, \textbf{w}^{(2)}_{t} ..., \textbf{w}^{(N)}_{t}\right],\\
\textbf{G}_t&=\left[g(\textbf{w}^{(1)}_{t}), ..., g(\textbf{w}^{(N)}_{t})\right].
\end{align*}
Therefore, the learning strategy in matrix form can be rewritten as follows.

\textbf{Local update:} $\textbf{X}_{t+1}=\textbf{X}_{t}-\eta \textbf{G}_{t}$ for $t\in[k]_1$.

\textbf{Inter-node communication:} $\textbf{X}_{t+1}=\textbf{X}_{t}\textbf{C}$ for $t\in[k]_2$.

We further combine the two update rules into one to facilitate the following theoretical analysis.
Before describing the update rule of DFL, we define two time-varying matrices $\textbf{C}_t$, $\textbf{G}^{\prime}_t$ as follows:
\begin{align}
\label{W-t}
\textbf{C}_t&\triangleq
\begin{cases}
\textbf{I},&t\in[k]_1 \\
\textbf{C},&t\in[k]_2,
\end{cases}\\
\label{G-t}
\textbf{G}^{\prime}_t&\triangleq
\begin{cases}
\textbf{G}_{t},&t\in[k]_1 \\
\textbf{0},&t\in[k]_2,
\end{cases}
\end{align}
where $\textbf{I}\in\mathbb{R}^{N\times N}$ is the identity matrix and $\textbf{0}\in\mathbb{R}^{d\times N}$ is the zero matrix. According to the definitions of \eqref{W-t} and \eqref{G-t},
\textbf{the global learning strategy} of DFL can be described as 
\begin{align}
\label{global-update-rule}
\textbf{X}_{t+1}=(\textbf{X}_t-\eta \textbf{G}^{\prime}_t)\textbf{C}_t.
\end{align}

Assuming that all nodes have limited communication and computation 
resources, we use a finite number $T$ to denote the total iteration steps. Then the corresponding number of the iteration round is denoted by $K$ with $K=\lfloor\frac{T}{\tau}\rfloor$\footnote{Note that we set the condition $T-K\tau\leq\tau_1$ in this paper for simplified analysis. The other case $T-K\tau>\tau_1$ can be similarly treated based on the former.}. We present the pseudo-code of DFL in Algorithm \ref{alg1}.

A comparison among different distributed SGD methods is shown in Table \ref{table-comparison}. From this table, we can find that DFL is beneficial for alleviating local drift caused by inter-node communication.
\begin{algorithm}[!t]
	\caption{\textbf{DFL}} 
	\label{alg1}
	\begin{algorithmic}[1]
		\REQUIRE ~~\\ 
		Learning rate $\eta$,
		total number of steps $T$,
		computation frequency $\tau_1$ and communication frequency $\tau_2$ in an iteration round
		\STATE Set the initial value of $\textbf{X}_0$. 
		\label{ code:fram:extract }
		\FOR{$t=1,2,...,K\tau$}
		\label{code:fram:trainbase}
		\IF{$t\in [k]_1$ where $k=1,...,K$} 
		\STATE	$\textbf{X}_{t+1}=\textbf{X}_{t}-\eta \textbf{G}_{t}$ \hfill$//$\textit{local update}
		\ELSE 
		\STATE $\textbf{X}_{t+1}=\textbf{X}_{t}\textbf{C}$ \hfill$//$\textit{inter-node communication}
		\ENDIF
		\ENDFOR
		\FOR{$t=K\tau+1,...,T$}
		\STATE $\textbf{X}_{t+1}=\textbf{X}_{t}-\eta \textbf{G}_{t}$
		\ENDFOR
	\end{algorithmic}
\end{algorithm}

\subsection{Special Cases of DFL}
When performing inter-node communication only once, DFL can degenerate into two special cases, which are D-SGD and C-SGD, respectively. 

\subsubsection{D-SGD}\label{D-SGD-title}

D-SGD algorithm was developed and studied for solving the decentralized learning problem without aggregating local models at a central server  \cite{lian2017can} \cite{jiang2017collaborative}.
The network topology of D-SGD is captured by a confusion matrix $\textbf{C}\in\mathbb{R}^{N\times N}$ whose $(j,i)$-th element $c_{ji}$ represents the contribution of node $j$ 
in model averaging at node $i$. As shown in Fig. \ref{A}, nodes of D-SGD execute \textit{local update} and \textit{inter-node communication} once
alternately.
The learning strategy of D-SGD consists of the two steps as follows.

\begin{figure}[!t]
	\centering
	\subfloat[D-SGD]{\includegraphics[scale=0.34]{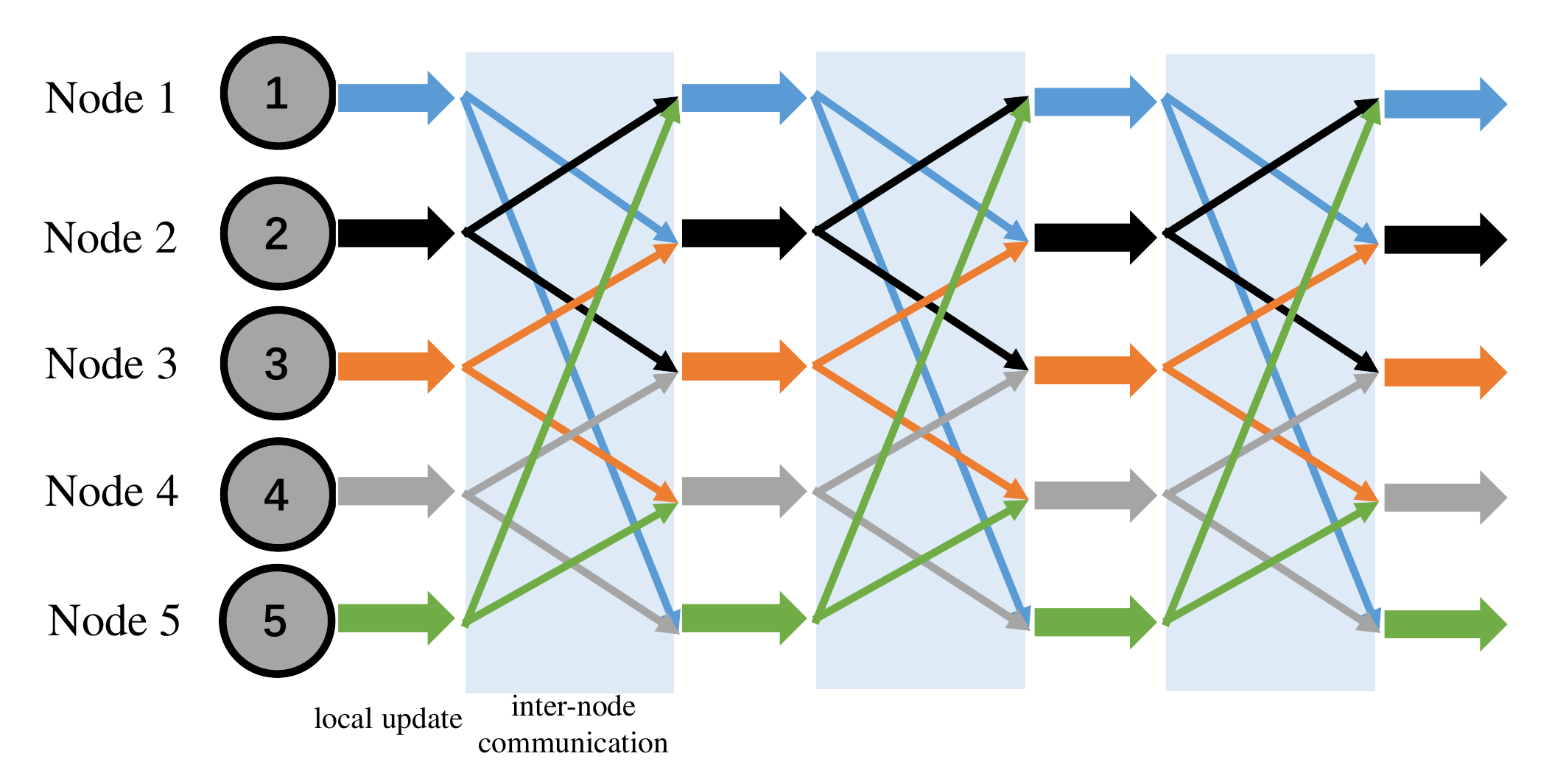}\label{A}}
	\\
	\subfloat[C-SGD]{\includegraphics[scale=0.34]{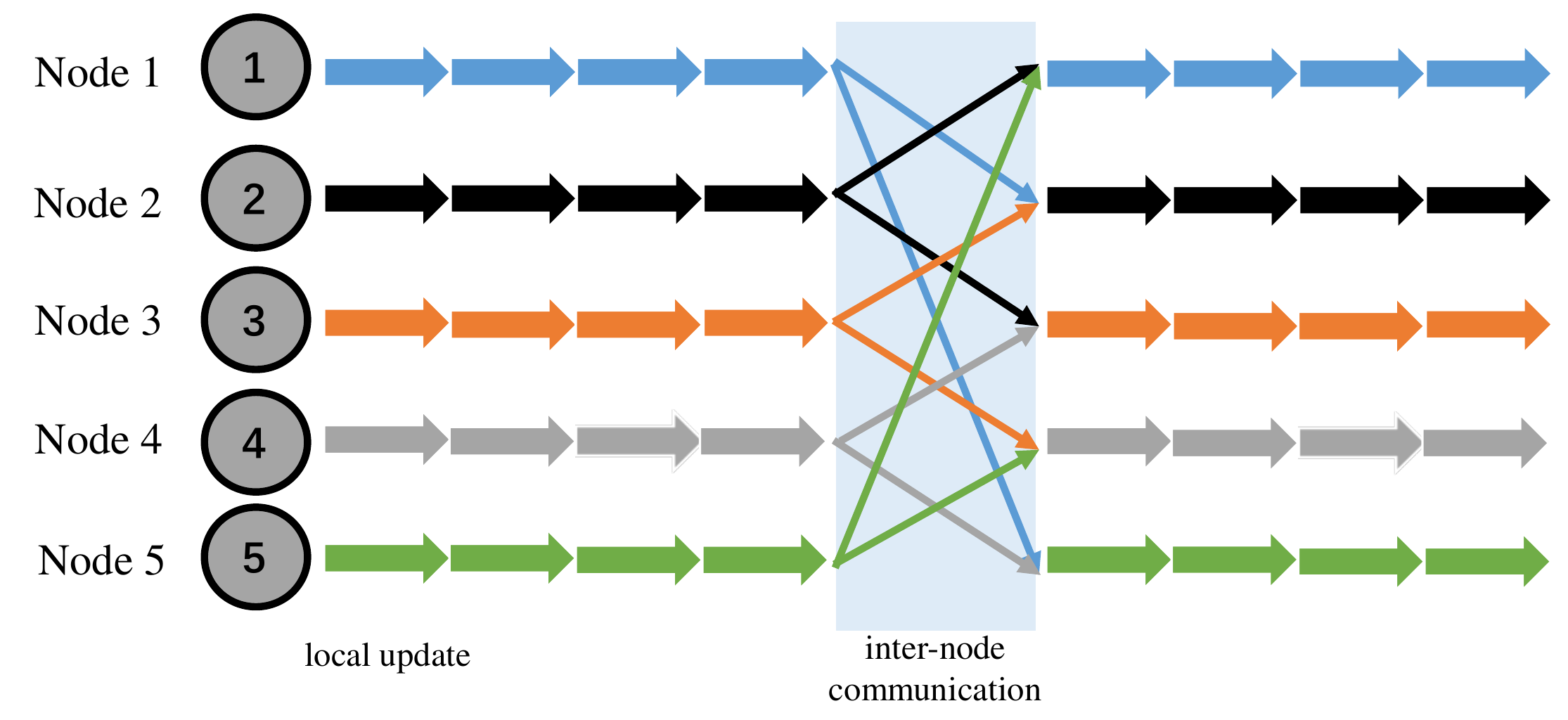}\label{B}}
	\caption{Learning strategies of D-SGD and C-SGD}
	\label{SGD-fig}
\end{figure}

\textbf{Inter-node communication:} For each node $i$, the neighboring nodes that are connected with node $i$ send their current model parameters to node $i$ at the $t$-th iteration step. Then node $i$ takes a weighted average of the received parameters and the local parameter by
\begin{align}
\label{D-SGD}\textbf{w}^{(i)}_{t+\frac{1}{2}}=\sum_{j=1}^{N}c_{ji}\textbf{w}^{(j)}_{t},     
\end{align}
where $\textbf{w}^{(i)}_{t+\frac{1}{2}}$ denotes the parameter of model averaging after inter-node communication.
Note that only the nodes connected with node $i$ contribute to the model average of node $i$. If node $j$ ($j=1,...,N$) is not connected with node $i$, then $c_{ji}=0$; otherwise, $c_{ji}>0$. Of course, $c_{ii}>0$ is necessary.

\textbf{Local update:} After inter-node communication, each node $i$ performs a local update as follows:
\begin{align}
\label{D-SGD-1}\textbf{w}_{t+1}^{(i)}=\textbf{w}^{(i)}_{t+\frac{1}{2}}-\eta g(\textbf{w}^{(i)}_{t}),
\end{align}
where $\eta$ is the learning rate.
The local update follows the SGD method.

At the $t$-th iteration step, \textit{the model parameter matrix} and \textit{the gradient matrix} are denoted by 
$\textbf{X}_t,\textbf{G}_t\in \mathbb{R}^{d\times N}$, respectively.
Then according to \eqref{D-SGD} and \eqref{D-SGD-1}, we have \textbf{the global learning strategy} of D-SGD by combining the two steps as follows:
\begin{align}\label{glo-d-sgd}
\textbf{X}_{t+1}=\textbf{X}_t \textbf{C}-\eta \textbf{G}_t.
\end{align}

\subsubsection{C-SGD}\label{C-SGD-title}
The work \cite{wang2019cooperative} proposed C-SGD which allows nodes to perform multiple local updates and inter-node communication once to improve communication efficiency. 
At $t$-th iteration step, all nodes have different model parameters $\textbf{w}^{(1)}_{t}, ..., \textbf{w}^{(N)}_{t}\in \mathbb{R}^d$. The network topology of C-SGD is captured by the confusion matrix $\textbf{C}\in\mathbb{R}^{N\times N}$ again. The learning strategy of C-SGD has similar steps as D-SGD. As shown in  Fig. \ref{B}, C-SGD performs an inter-node communication after $\tau$ local updates. 
The learning strategy of C-SGD is presented as follows.

\textbf{Local update:} For each node $i$, when the iteration index $t$ satisfies $t \bmod \tau\neq0$, a local update is performed according to
\begin{align}
\textbf{w}_{t+1}^{(i)}=\textbf{w}_{t}^{(i)}-\eta g(\textbf{w}^{(i)}_{t}),
\end{align}
where $\eta$ is the learning rate.

\textbf{Inter-node communication:} After the stage of local update, each node performs inter-node communication with its connected nodes. At the $t$-th
iteration step satisfying $t\bmod\tau=0$,
node $i$ receives the model parameters from its neighbors, and computes a weighted averaging as follows:
\begin{align}
\textbf{w}^{(i)}_{t+1}=\sum_{j=1}^{N}c_{ji}(\textbf{w}_{t}^{(i)}-\eta g(\textbf{w}^{(i)}_{t})).
\end{align} 

Then we present the matrix format of the C-SGD learning process.
The time-varying confusion matrix $\textbf{C}_t\in\mathbb{R}^{N\times N}$ is defined as 
\begin{equation*}
\textbf{C}_t=
\begin{cases}
\textbf{C},&t\bmod\tau=0 \\
\textbf{I},&\mbox{otherwise,}
\end{cases}
\end{equation*}
where $\textbf{I}\in \mathbb{R}^{N\times N}$ is the identity matrix, which means there is no inter-node communication during the local update stage. Then, at the $t$-th iteration step, \textit{the model parameter matrix} and \textit{the gradient matrix} are denoted by 
$\textbf{X}_t,\textbf{G}_t\in \mathbb{R}^{d\times N}$, respectively. 
Thus, \textbf{the global learning strategy} of C-SGD can be rewritten as
\begin{align}
\label{C-SGD_rules}
\textbf{X}_{t+1}=(\textbf{X}_t-\eta \textbf{G}_t)\textbf{C}_t.
\end{align}

Note that the order of inter-node communication and local update differs in D-SGD and C-SGD, resulting in difference of the global learning strategy of the two cases as shown in \eqref{glo-d-sgd} and \eqref{C-SGD_rules}. We call \eqref{glo-d-sgd} communicate-then-compute and \eqref{C-SGD_rules} compute-then-communicate. In fact, communicate-then-compute and compute-then-communicate have equivalent convergence behavior, as revealed by our analysis in Section \ref{equivalentcon}.


\subsubsection{Convergence of D-SGD and C-SGD}\label{equivalentcon}
We present the learning strategies of the two methods as follows:
\begin{align}
\label{prior-communicatio}\text{D-SGD:}\quad\textbf{X}_{t+1}&=\textbf{X}_t \textbf{C}-\eta \textbf{G}_t, \\
\label{prior-computatio}\text{C-SGD:}\quad\textbf{X}_{t+1}&=(\textbf{X}_t -\eta \textbf{G}_t)\textbf{C}_t, 
\end{align} 
where \eqref{prior-communicatio} and \eqref{prior-computatio} denote the learning strategies of communicate-then-compute and compute-then-communicate, respectively.

For facilitating the analysis, we multiply $\textbf{1}/N$ on both sides of \eqref{prior-communicatio} and \eqref{prior-computatio} to obtain the format of model averaging. So we have 
\begin{align}
\label{prior-communicatio-1}\textbf{X}_{t+1}\frac{\textbf{1}}{N}&=\textbf{X}_t \textbf{C}\frac{\textbf{1}}{N}-\eta \textbf{G}_t\frac{\textbf{1}}{N},\\
\label{prior-computatio-1}\textbf{X}_{t+1}\frac{\textbf{1}}{N}&=(\textbf{X}_t -\eta \textbf{G}_t)\textbf{C}_t\frac{\textbf{1}}{N}.
\end{align}
Because the confusion matrix $\textbf{C}$ is doubly stochastic, we have $\textbf{C}\textbf{1}=\textbf{1}$. According to the definition of $\textbf{C}_t$ in C-SGD, we obtain $\textbf{C}_t\textbf{1}=\textbf{1}$. Therefore, we can rewrite \eqref{prior-communicatio-1} and \eqref{prior-computatio-1} as
\begin{align*}
\text{D-SGD:}\quad\textbf{u}_{t+1}=\textbf{u}_{t}-\eta\left[\frac{1}{N}\sum_{i=1}^{N}g(\textbf{w}_t^{(i)})\right],\\
\text{C-SGD:}\quad\textbf{u}_{t+1}=\textbf{u}_{t}-\eta\left[\frac{1}{N}\sum_{i=1}^{N}g(\textbf{w}_t^{(i)})\right],
\end{align*}
respectively, where we define $\textbf{u}_t\triangleq \textbf{X}_t\frac{\textbf{1}}{N}$. Consequently, we find out that the learning strategies of communicate-then-compute and compute-then-communicate have the same update rules based on the averaged model $\textbf{u}_t$. 
Thus, we can conclude that communicate-then-compute and compute-then-communicate lead to an equivalent convergence bound.

\section{Convergence Analysis}
In this section, we analyze the convergence of the proposed algorithm in the DFL framework and further study how 
the choice of $\tau_1$ and $\tau_2$ impacts the convergence bound.
\subsection{Preliminaries}
To facilitate the analysis, we make the following assumptions.

\newtheorem{assumption}{Assumption}
\begin{assumption}\label{ass1}
	we assume the following conditions:\\
	\textup{1)} $F_i(\textup{\textbf{w}})$ is $L$-smooth, i.e., $\|\nabla F_i(\textup{\textbf{x}})-\nabla F_i(\textup{\textbf{y}})\|\leq L\|\textup{\textbf{x}}-\textup{\textbf{y}}\|$ for some $L>0$ and any $\textup{\textbf{x}}$, $\textup{\textbf{y}},i$. Hence, we have $F(\textup{\textbf{w}})$ is $L$-smooth from the triangle inequality\textup{;}\\
	\textup{2)} $F_i(\textup{\textbf{w}})$ is $\mu$-strongly convex\textup{;}\\
	\textup{3)} $F(\textup{\textbf{w}})$ has a lower bound, i.e., $F(\textup{\textbf{w}})\geq F_{\textup{inf}}$ for some $F_{\textup{inf}}>0$\textup{;}\\
	\textup{4)} Gradient estimation is unbiased for stochastic mini-batch sampling, i.e., $\mathbb{E}_{\xi|\textup{\textbf{w}}}[g(\textup{\textbf{w}})]=\nabla F(\textup{\textbf{w}})$\textup{;}\\
	\textup{5)} Gradient estimation has a bounded variance, i.e., $\mathbb{E}_{\xi|\textup{\textbf{w}}}\|g(\textup{\textbf{w}})-\nabla F(\textup{\textbf{w}})\|^2\leq\sigma^2$ where $\sigma^2>0$. Furthermore, for any node $i$, $\mathbb{E}_{\xi|\textbf{\textup{w}}}\|g(\textup{\textbf{w}}^{(i)})-\nabla F_i(\textup{\textbf{w}})\|^2\leq \sigma_i^2$ and $\mathbb{E}_{\xi|\textbf{\textup{w}}}\|g(\textup{\textbf{w}}^{(i)})\|^2\leq G^2$. We use $\bar{\sigma}^2$ to denote the average of $\{\sigma_i^2\}$, i.e., $$\bar{\sigma}^2\triangleq\frac{1}{N}\sum_{i=1}^{N}\sigma_i^2;$$\\
	\textup{6)} $\textup{\textbf{C}}$ is a doubly stochastic matrix, which satisfies $\textup{\textbf{C}}\textup{\textbf{1}}=\textup{\textbf{1}},
	\textup{\textbf{C}}^\top=\textup{\textbf{C}}$. The largest eigenvalue of $\textup{\textbf{C}}$ is always $1$ and the other eigenvalues are strictly less than $1$, i.e., $\max\left\{|\lambda_2(\textup{\textbf{C}})|, |\lambda_N(\textup{\textbf{C}})|\right\}<\lambda_1(\textup{\textbf{C}})=1$, where $\lambda_i(\textup{\textbf{C}})$ denotes the $i$-th largest eigenvalue of $\textup{\textbf{C}}$. For convenience, we define $$\beta\triangleq\|\textup{\textbf{I}}-\textup{\textbf{C}}\|_2\in[0,2],$$
	which is a similar way to analyze communication-compressed D-SGD in \textup{\cite{koloskova2019decentralized}}.
\end{assumption}
\subsection{Update of Model Averaging}\label{model_averaging}
For convenience of analysis, we transform the update rule of DFL in \eqref{global-update-rule} into model averaging. By
multiplying $\textbf{1}/N$ on both sides of \eqref{global-update-rule}, we obtain 
\begin{align}
\textbf{X}_{t+1}\frac{\textbf{1}}{N}=\textbf{X}_{t}\frac{\textbf{1}}{N}-\eta \textbf{G}^{\prime}_{t}\frac{\textbf{1}}{N},
\end{align}
where $\textbf{C}_{t}$ is eliminated because of $\textbf{C}_t \textbf{1}=\textbf{1}$ from the definition of $\textbf{C}_t$ in \eqref{W-t}. 
Thus, we can obtain 
\begin{align}
\textbf{u}_{t+1}=\textbf{u}_{t}-\eta\left[\frac{1}{N}\sum_{i=1}^{N}g(\textbf{w}_t^{(i)})\right]
\end{align}
for $t\in[k]_1$ and $\textbf{u}_{t+1}=\textbf{u}_{t}$ for $t\in[k]_2$. This is because SGD algorithm is only 
performed in local updates. In the remaining analysis, we will concentrate on the convergence of the average model $\textbf{u}_t$,
which is a common way to analyze convergence under stochastic sampling setting of gradient descent \cite{yu2019parallel} \cite{lian2017can} \cite{stich2018local}. 

\newtheorem{remark}{Remark}

Because the global loss function $F(\textbf{w})$ can be non-convex in complex learning platform such as Convolutional Neural Network (CNN), SGD may converge to a local minimum or a saddle point. 
We use the expectation of the gradient norm average of all iteration steps as the indicator of convergence. Similar indicators 
have been used in \cite{jiang2018linear} \cite{yu2019linear}.
The algorithm is convergent if the following condition is satisfied.
\newtheorem{definition}{Definition}
\begin{definition}[\textbf{Algorithm Convergence}]\label{def1}
	The algorithm converges to a stationary point if 
	it achieves an $\epsilon$-suboptimal solution, i.e.,
	\begin{align}\label{defF}
	\mathbb{E}\left[\frac{1}{T}\sum_{t=1}^{T}\|\nabla F(\textup{\textbf{u}}_t)\|^2\right]\leq\epsilon.
	\end{align}
\end{definition} 

\subsection{Convergence Results}
Based on Assumption \ref{ass1} and Definition \ref{def1}, we have the 
following proposition which characterizes the convergence of DFL. 

\newtheorem{proposition}{Proposition}
\begin{proposition}[\textbf{Convergence of DFL}]\label{pro1}
	For DFL algorithm, let the total steps be $T$ and the iteration rounds be $\tau=\tau_1+\tau_2$, where $\tau_1$  and $\tau_2$ are the computation frequency and the communication frequency, respectively. Based on 
	Assumption \ref{ass1} and that all local models have the same initial point $\textup{\textbf{u}}_1$, if the learning rate $\eta$ satisfies 
	\begin{align}
	\label{eta}
	\eta L+\frac{\eta^2 L^2\tau}{1-\zeta^{\tau_2}}\left(\frac{2\tau_1\zeta^{2\tau_2}}{1+\zeta^{\tau_2}}+\frac{2\tau_1\zeta^{\tau_2}}{1-\zeta^{\tau_2}}+\tau-1
	\right)\leq 1,
	\end{align}
	where $\zeta=\max\left\{|\lambda_2(\textup{\textbf{C}})|, |\lambda_N(\textup{\textbf{C}})|\right\}$ with $0\leq \zeta\leq 1$, then the expectation of the gradient norm average after $T$ steps is bounded as follows:
	\begin{align}
	\label{upper-bound}\mathbb{E}\left[\frac{1}{T}\sum_{t=1}^{T}\|\nabla F(\textup{\textbf{u}}_t)\|^2\right]\leq\underbrace{\frac{2[F(\textup{\textbf{u}}_1)-F_{\textup{inf}}]}{\eta T}+\frac{\eta L \sigma^2}{N}}_{\mbox{\footnotesize synchronous SGD}}+\notag\\
	\underbrace{2\eta^2 L^2 \sigma^2 \left(\frac{\tau_1}{1-\zeta^{2\tau_2}}-1\right)}_{\mbox{\footnotesize local drift}}\\
	\label{inftyT}\stackrel{T\rightarrow\infty}{\longrightarrow}\frac{\eta L \sigma^2}{N}+2\eta^2 L^2 \sigma^2 \left(\frac{\tau_1}{1-\zeta^{2\tau_2}}-1\right),
	\end{align}
	where $\textup{\textbf{u}}_t=\textbf{X}_t\frac{\textup{\textbf{1}}}{N}$.
\end{proposition}
\begin{proof}
	The complete proof is presented in Appendix A. 
  We present the proof outline as follows.

	From Definition \ref{def1}, we first derive an upper bound of $\mathbb{E}\left[\frac{1}{T}\sum_{t=1}^{T}\|\nabla F(\textup{\textbf{u}}_t)\|^2\right]$, which is given by Lemma 1 in Appendix A as follows:
	\begin{align}\label{upperboundsdfl}
		\mathbb{E}\left[\frac{1}{T}\sum_{t=1}^{T}\|\nabla F(\textbf{\textup{u}}_t)\|^2\right]\leq \underbrace{\frac{2[F(\textbf{\textup{w}}_1)-F_{\textup{inf}}]}{\eta T}+\frac{\eta L \sigma^2}{N}}_{\mbox{\footnotesize synchronous SGD}}+\notag\\
		\underbrace{\frac{L^2}{T}\sum_{t=1}^{T}\frac{\mathbb{E}\|\textbf{\textup{X}}_t(\textbf{\textup{I}}-\textbf{\textup{J}})\|^2_{\textup{F}}}{N}}_{\mbox{\footnotesize local drift}}.
	\end{align}
	Then considering the second term of \eqref{upperboundsdfl}, we can find an upper bound of $\mathbb{E}\|\textbf{\textup{X}}_t(\textbf{\textup{I}}-\textbf{\textup{J}})\|^2_{\textup{F}}$, i.e.,
	\begin{align*}
	&\mathbb{E}\|\textbf{\textup{X}}_t(\textbf{\textup{I}}-\textbf{\textup{J}})\|^2_{\textup{F}}\\
	\leq&\underbrace{2\eta^2\mathbb{E}\left\|\sum_{r=0}^{j}(\textbf{Y}_r\!-\!\textbf{Q}_r)(\textbf{C}^{(j-r)\tau_2}\!-\!\textbf{J})\right\|^2_{\textup{F}}}_{A_1}\\
	&+\underbrace{2\eta^2\mathbb{E}\left\|\sum_{r=0}^{j}\textbf{Q}_r(\textbf{C}^{(j-r)\tau_2}\!-\!\textbf{J})\right\|^2_{\textup{F}}}_{A_2}.
	\end{align*}
	Finally, we decompose $\mathbb{E}\|\textbf{\textup{X}}_t(\textbf{\textup{I}}-\textbf{\textup{J}})\|^2_{\textup{F}}$ into two parts, denoted as $A_1$ and $A_2$, respectively. 
	For $A_1$, we have
	\begin{align*}
	&\sum_{t=1}^{T}A_1\leq\\
	&2\eta^2 N\sigma^2 T\left(\frac{\tau_1}{1-\zeta^{2\tau_2}}-1 \right)+\frac{2\eta^2\beta\tau}{1-\zeta^{2\tau_2}}\sum_{r=0}^{\frac{T}{\tau}-1}\sum_{s=r\tau+1}^{r\tau+\tau_1}\|\nabla F(\textbf{X}_s)\|^2_{\textup{F}};
	\end{align*}
	and for $A_2$, we have
	\begin{align*}
	&\sum_{t=1}^{T}A_2\leq\\
	&\frac{\eta^2\tau}{1-\zeta^{\tau_2}}\!\!\left(\frac{2\tau_1\zeta^{2\tau_2}}{1+\zeta^{\tau_2}}\!\!+\!\!\frac{2\tau_1\zeta^{\tau_2}}{1-\zeta^{\tau_2}}\!\!+\!\tau\!-\!1\!
	\right)\!\!\sum_{r=0}^{\frac{T}{\tau}-1}\!\sum_{s=r\tau+1}^{r\tau+\tau_1}\!\!\mathbb{E}\|\nabla F(\textbf{X}_{s})\|^2_{\textup{F}}.
	\end{align*}
	From the upper bounds of $A_1$ and $A_2$, we finish the proof of Proposition \ref{pro1}.

	Note that the assumption of $\mu$-strong convexity in Assumption \ref{ass1} is not necessary for the convergence bound in \eqref{upper-bound}.
\end{proof}

Focusing on the right side of the inequality \eqref{upper-bound},
we find out that the convergence upper bound of DFL consists of two parts as indicated by the under braces. The first two terms are the original error bound led by synchronous SGD \cite{bottou2018optimization}, where one local update follows one model averaging covering all nodes, not just the neighbors. The synchronous SGD also corresponds to a special DFL case with $\tau_1=1,\tau_2\to\infty$.
The last term is the error bound brought by local drift, which reflects the effect of performing multiple local updates and insufficient model averaging. The consensus model can not be acquired due to finite communication frequency $\tau_2$. Focusing on \eqref{inftyT}, we can find out that DFL algorithm converges to a stable and static point as $T\rightarrow\infty$, regardless of network topology. The network topology parameter given by $\zeta$ and the frequencies of computation and communication (i.e., $\tau_1$ and $\tau_2$, respectively) will affect the final convergence upper bound.

We further analyze how the computation frequency $\tau_1$, the communication frequency $\tau_2$ and the network topology parameter $\zeta$ affect the convergence of DFL.
\begin{remark}[\textbf{Impact of $\tau_1,\tau_2$}]\label{remark1}
	\textup{
	Focusing on the last term in \eqref{upper-bound}, we can find how the computation frequency $\tau_1$ and the communication frequency $\tau_2$ affect the convergence of DFL. Specifically, the bound increases with $\tau_1$ and 
	decreases with $\tau_2$ monotonically. The former is because more local updates lead to larger difference between the local model and the previous average model over neighboring nodes. Thus the local drift is intensified with $\tau_1$. The latter is because more inter-node communications can facilitate model consensus and make local model approach the ideal average model. Thus the local drift is mitigated by increasing $\tau_2$. DFL with  $\tau_2=1$ is equivalent to C-SGD, which has a worse convergence rate compared to DFL with $\tau_2>1$. The local drift will tend to zero under the ideal condition $\tau_1=1$ and $\tau_2\to\infty$, which is the synchronous SGD.  In contrast, the local drift will become the worst when $\tau_1$ is large and $\tau_2=1$.}
\end{remark}
\newtheorem{corollary}{Corollary}
\begin{corollary}\label{cor1}
	If $\tau_1=1,\tau_2\to\infty$, DFL achieves the optimal convergence bound without local drift which is
	\begin{align}
	\mathbb{E}\left[\frac{1}{T}\sum_{t=1}^{T}\|\nabla F(\textup{\textbf{u}}_t)\|^2\right]\leq\frac{2[F(\textup{\textbf{u}}_1)-F_{\textup{inf}}]}{\eta T}+\frac{\eta L \sigma^2}{N}.
	\end{align}
\end{corollary}
\begin{proof}
	It directly follows from \eqref{upper-bound}.
\end{proof}
\begin{remark}[\textbf{Impact of $\zeta$}]\label{remark2}
	\textup{
	From \eqref{upper-bound}, the local drift will monotonically increase with $\zeta$. 
	As we defined, $\zeta$ denotes the second largest eigenvalue of $\textup{\textbf{C}}$ and it reflects the exchange rate of all local models during a single inter-node communication step. Thus a larger $\zeta$ means a sparser matrix, resulting to a worse model averaging and aggravating the local drift.
	When each node can not communicate with any other nodes in the network, the confusion matrix $\textup{\textbf{C}}=\textup{\textbf{I}}$ and $\zeta=1$. There does not exist model averaging and it will cause the worst local drift.  When each node can communicate with all nodes in the network, the confusion matrix $\textup{\textbf{C}}=\textup{\textbf{J}}$ and $\zeta=0$. Model consensus is achieved and the local drift brought by insufficient model averaging is eliminated.}
\end{remark}
\begin{corollary}\label{cor2}
	If $\zeta=0,\textup{\textbf{C}}=\textup{\textbf{J}}$, DFL achieved a convergence bound without local drift caused by insufficient model averaging, which is expressed as
	\begin{align}
	\mathbb{E}\left[\frac{1}{T}\sum_{t=1}^{T}\|\nabla F(\textup{\textbf{u}}_t)\|^2\right]\leq\frac{2[F(\textup{\textbf{u}}_1)-F_{\textup{inf}}]}{\eta T}+\frac{\eta L \sigma^2}{N}+\notag\\
	2\eta^2 L^2 \sigma^2 \left(\tau_1-1\right).
	\end{align}
\end{corollary}
\begin{proof}
	It directly follows from \eqref{upper-bound}.
\end{proof}

\section{Communication-efficient DFL based on compressed communication}
In this section, we further study DFL based on a compressed communication theme to reduce the communication overhead when considering multiple inter-node communication steps in a round. The learning strategy of DFL with compressed communication is introduced and the algorithm is developed. Then we present the convergence analysis of DFL with compressed communication.

\subsection{The Learning Strategy of C-DFL}\label{con-csdfl}
In order to further improve communication efficiency of DFL, we apply the compressed inter-node communication scheme to DFL algorithm. Focusing on the inter-node communication stage, for each node $i$, the gossip type algorithm \cite{xiao2003fast} can exchange the model parameters by iterations of the form
\begin{align}
\textbf{w}_{t+1}^{(i)}=\textbf{w}_{t}^{(i)}+\gamma\sum_{j=1}^{N}c_{ji}\Delta_t^{(ji)},
\end{align}
where $\gamma\in(0,1]$ denotes the consensus step size and $\Delta_t^{(ji)}\in \mathbb{R}^{d}$ denotes a vector sent from node $j$ to node $i$ in the $t$-th iteration step. Note that if  $\gamma=1$ and  $\Delta_t^{ji}=\textbf{w}_t^{(j)}-\textbf{w}_t^{(i)}$, we have $\textbf{w}_{t+1}^{(i)}=\sum_{j=1}^{N}c_{ji}\textbf{w}_t^{j}$, which is consistent with the inter-node communication rule of DFL. We apply the CHOCO-G gossip scheme to DFL with compressed communication to remove the compression noise as $t\to \infty$ and guaranteeing model average-consensus \cite{koloskova2019decentralized}. The scheme is presented as
\begin{align}
\Delta_t^{(ji)}&=\hat{\textbf{w}}_t^{(j)}-\hat{\textbf{w}}_t^{(i)},\\
\hat{\textbf{w}}_{t+1}^{(j)}&=\hat{\textbf{w}}_{t}^{(j)}+Q(\textbf{w}_{t+1}^{(j)}-\hat{\textbf{w}}_t^{(j)}),
\end{align}
where $Q$ is a compression operator and $\hat{\textbf{w}}_t^{(i)}\in\mathbb{R}^d$ denotes additional variables which can be obtained and stored by all neighboring nodes (including node $i$ itself) of node $i$. 

We introduce some examples of compression strategies \cite{koloskova2019decentralized}. 
\begin{itemize}
	\item \textit{Sparsification}: We randomly select $k$ out of $d$ coordinates ($\textup{rand}_k$), or the $k$ coordinates with highest magnitude values ($\textup{top}_k$), which gives $\delta=\frac{k}{d}$.
	\item \textit{Rescaled unbiased estimators}: Assuming $\mathbb{E}_Q Q(\textbf{x})=\textbf{x}$, $\forall \textbf{x}\in \mathbb{R}^d$ and $\mathbb{E}_Q \|Q(\textbf{x})\|^2\leq c \|\textbf{x}\|^2$, we then have $Q'(\textbf{x})=\frac{1}{c}Q(\textbf{x})$ with $\delta=\frac{1}{c}$.
	\item \textit{Randomized gossip}: Supposing $Q(\textbf{x})=\textbf{x}$ with probability $p\in(0,1]$ and $Q(\textbf{x})=0$ otherwise, we have $\delta=p$. 
	\item \textit{Random quantization}: For the quantization levels $s\in\mathbb{N}_{+}$ and $c=(1+\min\{d/s^2,\sqrt{d}/s\})$, the quantization operator can be formulated as $$\textup{qsgd}_s(x)=\frac{\textup{sign}(\textbf{x})\cdot\|\textbf{x}\|}{sc}\cdot\left\lfloor s\frac{|x|}{\|\textbf{x}\|}+\xi\right\rfloor,$$ where the random variable $\xi\thicksim_{\textup{u.a.r.}}[0,1]^d$ satisfies \eqref{Q} with $\delta=\frac{1}{c}$.
\end{itemize}
Note that $\delta$ is the compression rate, which reflects the compression quality for the compression operator $Q$. 

We apply the CHOCO-G scheme to inter-node communication of DFL for $t\in[k]_2$, and local update for $t\in[k]_1$ remains unchanged. The proposed communication-efficient algorithm is named DFL with compressed communication (C-DFL). C-DFL is presented in Algorithm \ref{alg2}. Note that in C-DFL, we set $T=K\tau$ for convenience. According to Algorithm \ref{alg2}, local update step with SGD is in line \ref{localupdatesgd}, iterative update of inter-node communication step is in line \ref{iterationupadte}, compression is performed in line \ref{compression}, the exchange of compressed information is in line \ref{infoexchange}, and the iteration of additional variables is in line \ref{additionalvariableupdate}.

Because multiple inter-node communications with compression and multiple local updates are performed in a round, C-DFL is a more general analytical framework compared with CHOCO-SGD proposed in \cite{koloskova2019decentralized}.
\begin{algorithm}[!t]
	\caption{\textbf{C-DFL}} 
	\label{alg2}
	\begin{algorithmic}[1]
		\REQUIRE ~~\\ 
		Total number of steps $T$, computation frequency $\tau_1$ and communication frequency $\tau_2$ in an iteration round, initial values $\textbf{w}_0^{(i)}$ for all nodes, the learning rate $\eta_k$ changed with $k$, the consensus step size $\gamma$ 
	
		\STATE Set the initial value $\hat{\textbf{w}}_0^{i}=\textbf{0}$ $\forall i$.
		\label{ code:fram:extract }
		\FOR{$t=1,2,...,K\tau$ in parallel for all nodes $i$}
		\label{code:fram:trainbase}
		\IF{$t\in [k]_1$ where $k=1,...,K$} 
		\STATE	$\textbf{w}_{t+1}^{(i)}=\textbf{w}_{t}^{(i)}-\eta_k g(\textbf{w}_{t}^{(i)})$ 
		\label{localupdatesgd}
		\ELSE 
		\STATE $\textbf{w}_{t+1}^{(i)}=\textbf{w}_{t}^{(i)}+\gamma\sum_{j=1}^{N}c_{ji}(\hat{\textbf{w}}_t^{(j)}-\hat{\textbf{w}}_t^{(i)})$  \label{iterationupadte}
		\STATE $\textbf{q}_t^{(i)}=Q(\textbf{w}_{t+1}^{(i)}-\hat{\textbf{w}}_t^{(i)})$
		\label{compression}
		\ENDIF
		\FOR{neighboring nodes $j$ of node $i$ (including node $i$ itself)}
		\STATE Send $\textbf{q}_t^{(i)}$ and receive $\textbf{q}_t^{(j)}$ \label{infoexchange}
		\STATE $\hat{\textbf{w}}_{t+1}^{(j)}=\hat{\textbf{w}}_{t}^{(j)}+\textbf{q}_t^{(j)}$ \label{additionalvariableupdate}
		\ENDFOR
		\ENDFOR
	\end{algorithmic}
\end{algorithm}

\subsection{Convergence Analysis of C-DFL}
In this subsection, we analyze the convergence of Algorithm \ref{alg2}. We make an assumption about the compression operator $Q$ to measure the compression quality of different methods.
\begin{assumption}[Compression operator]\label{ass2}
	We assume that $\forall \textbf{\textup{x}}\in \mathbb{R}^d$, the compression operator $Q$: $\mathbb{R}^d\to \mathbb{R}^d$ satisfies
	\begin{align}\label{Q}
		\mathbb{E}_{Q}\|Q(\textbf{\textup{x}})-\textbf{\textup{x}}\|^2\leq (1-\delta)\|\textbf{\textup{x}}\|^2,
	\end{align}
	for the compression ratio $\delta>0$, where $\mathbb{E}_Q$ denotes the expectation about the internal randomness of operator $Q$. 
\end{assumption}

We define $\textbf{u}^{(k)}\triangleq \textbf{u}_{k\tau}$, which is the average of model parameters after $k$ iteration rounds. With Assumption \ref{ass1} qualifying the loss function $F_i(\cdot)$ and Assumption \ref{ass2} bounding the discrepancy of compression, we obtain the following proposition which states the convergence of C-DFL.
\begin{proposition}[\textbf{Convergence of C-DFL}]\label{convergenceofCSDFL}
	Under Assumption \ref{ass1} and \ref{ass2}, 
	if the learning rate $\eta_k=\frac{4}{\mu(a+k)}$ for parameter $a\geq 16\kappa$, consensus step size $\gamma=\frac{\rho^2\delta}{16\rho+\rho^2+4\beta^2 +2\rho\beta^2-8\rho\delta}$, and the spectral gap $\rho=1-\max\left\{|\lambda_2(\textup{\textbf{C}})|, |\lambda_N(\textup{\textbf{C}})|\right\}\in(0,1]$, then $\mathbb{E}[F(\textbf{\textup{w}}^{(T)}_{avg})]$ is bounded as follows: 
	\begin{align}\label{convergence-csdfl}
		\mathbb{E}[F(\textbf{\textup{w}}^{(T)}_{avg})]\!\!-\!\!F^*\!\!=&\mathcal{O}\left(\!\frac{\tau_1\bar{\sigma}^2}{\mu K N}\!\!\right)\!\!+\!\!\mathcal{O}\left(\!\frac{\kappa\tau_1^4 G^2 (1-p)^{\tau_2}}{\mu K^2[1\!-\!(1+\theta)(1-p)^{\tau_2}]}\!\right)\notag\\
		&+\mathcal{O}\left(\frac{\kappa^3 G^2\tau_1}{K^3\mu}\right),
	\end{align}
	where $\textbf{\textup{w}}^{(T)}_{avg}=\frac{1}{S_K}\sum_{k=0}^{K-1}w_k \textup{\textbf{u}}^{(k)}$ for $w_k=(a+k)^2$, $S_K=\sum_{k=0}^{K-1}w_k\geq \frac{1}{3}K^3$, $\theta<\frac{p}{1-p}$ with $p=\frac{\rho^2\delta}{82}$, and $\kappa=\frac{L}{\mu}$. Note that $F^*$ is the global minimum of $F(\cdot)$.
\end{proposition}
\begin{proof} The complete proof is presented in Appendix B. 
We give the proof outline as follows.

	We first have 
	\begin{align}\label{outline2}
	&\|\textbf{u}^{(k+1)}-\textbf{u}^*\|^2\notag\\
	=&\left\|\textbf{u}^{(k)}-\textbf{u}^*-\eta_{k}\sum_{j=0}^{\tau_1-1}\frac{1}{N}\sum_{q=1}^{N}g(\textbf{w}_{k\tau+j}^{(q)})\right\|^2\notag\\
	=&\left\|\textbf{u}^{(k)}-\textbf{u}^*-\sum_{j=0}^{\tau_1-1}\frac{\eta_{k}}{N}\sum_{i=1}^{N}\nabla F_i(\textbf{w}_{k\tau+j}^{(i)})+\right.\notag\\
	&\left.\sum_{j=0}^{\tau_1-1}\frac{\eta_{k}}{N}\sum_{i=1}^{N}\nabla F_i(\textbf{w}_{k\tau+j}^{(i)})
	-\sum_{j=0}^{\tau_1-1}\frac{\eta_{k}}{N}\sum_{q=1}^{N}g(\textbf{w}_{k\tau+j}^{(q)})\right\|^2\notag\\	
	\leq&\left\|\textbf{u}^{(k)}-\textbf{u}^*-\sum_{j=0}^{\tau_1-1}\frac{\eta_{k}}{N}\sum_{i=1}^{N}\nabla F_i(\textbf{w}_{k\tau+j}^{(i)})\right\|^2+\eta_{k}^2\tau_1^2\frac{\bar{\sigma}^2}{N}.
	\end{align}
	Then considering the first term of \eqref{outline2}, we obtain
	\begin{align}\label{outline2-1}
	&\left\|\textbf{u}^{(k)}-\textbf{u}^*-\sum_{j=0}^{\tau_1-1}\frac{\eta_{k}}{N}\sum_{i=1}^{N}\nabla F_i(\textbf{w}_{k\tau+j}^{(i)})\right\|^2\notag\\
	=&\|\textbf{u}^{(k)}-\textbf{u}^*\|^2+\eta_{k}^2\underbrace{\left\|\sum_{j=0}^{\tau_1-1}\frac{1}{N}\sum_{i=1}^{N}\nabla F_i(\textbf{w}_{k\tau+j}^{(i)})\right\|^2}_{\mbox{\footnotesize $T_1$}}\notag\\
	&-\underbrace{2\eta_{k}\left\langle\textbf{u}^{(k)}-\textbf{u}^*,\sum_{j=0}^{\tau_1-1}\frac{1}{N}\sum_{i=1}^{N}\nabla F_i(\textbf{w}_{k\tau+j}^{(i)})\right\rangle}_{\mbox{\footnotesize $T_2$}}.
	\end{align}
    Considering $T_1$ in \eqref{outline2-1}, we have
    \begin{align*}
    &T_1\notag\\
    \leq&\frac{2L^2}{N}\tau_1\sum_{j=0}^{\tau_1-1}\sum_{i=1}^{N}\|\textbf{w}_{k\tau+j}^{(i)}-\textbf{u}_{k\tau+j}\|^2+4L\tau_1^2 \alpha_k(F(\textbf{u}^{(k)})-F^*);
    \end{align*} 
    and considering $T_2$ in \eqref{outline2-1}, we obtain 
    \begin{align*}
    &-\frac{1}{\eta_{k}}T_2\notag\\
    \leq&-2\tau_1(F(\textbf{u}^{(k)})-F^*)+\frac{L+\mu}{N}\sum_{j=0}^{\tau_1-1}\sum_{i=1}^{N}\|\textbf{u}^{(k)}-\textbf{w}_{k\tau+j}^{(i)}\|^2\notag\\
    &-\frac{\mu\tau_1}{2}\|\textbf{u}^{(k)}-\textbf{u}^*\|^2.
    \end{align*}
    We can further get
    \begin{align*}
    &\|\textbf{u}^{(k)}-\textbf{w}_{k\tau+j}^{(i)}\|^2\notag\\
    \leq&2\|\textbf{u}_{k\tau+j}-\textbf{w}_{k\tau+j}^{(i)}\|^2+\frac{4\eta_{k}^2 L^2 j}{N}\sum_{p=0}^{j-1}\sum_{i=1}^{N}\|\textbf{w}_{k\tau+p}^{(i)}-\textbf{u}_{k\tau+p}\|^2+\notag\\
    &8\eta_{k}^2 L j^2 \alpha_k(F(\textbf{u}^{(k)})-F^*)+\frac{\eta_{k}^2 j^2 \bar{\sigma}^2}{N}.
    \end{align*}
    Then putting everything together, we obtain Lemma 8 in Appendix B, where a upper bound of $\mathbb{E}\|\textbf{u}_{(k+1)\tau}-\textbf{u}^*\|^2$ is derived as 
	\begin{align}\label{subcon}
		&\mathbb{E}_{\xi_{k\tau}^{(1)},...,\xi_{k\tau}^{(N)}}\|\textbf{\textup{u}}^{(k+1)}-\textbf{\textup{u}}^{*}\|^2\leq\left(1-\frac{\mu\eta_{k}\tau_1}{2}\right)\|\textbf{\textup{u}}^{(k)}-\textbf{\textup{u}}^{*}\|^2 +\notag \\
		&\eta_k^2\tau_1^2\frac{\bar{\sigma}^2}{N} + \frac{(L+\mu)\eta_k^3\bar{\sigma}^2(\tau_1-1)(2\tau_1-1)\tau_1}{6N} +\notag\\ &\!(\!F(\textbf{\textup{u}}^{(k)})\!\!-\!F^*\!)\!\!\left[\!4L\eta_k^2\tau_1^2\alpha_k\!\!-\!\!2\eta_k\tau_1\!\!+\!\!\frac{4}{3}(L\!\!+\!\!\mu)\eta_{k}^3L\alpha_k\!(\tau_1\!\!-\!\!1)\!(2\tau_1\!\!-\!\!1)\tau_1\!\!\right]\notag\\
		&\!+\! \!\frac{2L^2\eta_k^2\tau_1}{N}\!\Phi_k \!\!+\!\! \frac{2(L\!\!+\!\!\mu)\eta_k}{N}\Phi_k\!\!+\!\!\frac{(L\!\!+\!\!\mu)\eta_k}{N}\!\!\sum_{j=0}^{\tau_1-1}\sum_{i=1}^{N}\frac{4\eta_k^2L^2j^2}{N}\Phi_{k,j},
	\end{align}
	where $\Phi_k=\sum_{j=0}^{\tau_1-1}\sum_{i=1}^{N}\|\textbf{\textup{u}}_{k\tau+j}-\textbf{\textup{w}}_{k\tau+j}^{(i)}\|^2$ and $\Phi_{k,j}=\sum_{p=0}^{j-1}\sum_{i=1}^{N}\|\textbf{\textup{u}}_{k\tau+p}-\textbf{\textup{w}}_{k\tau+p}^{(i)}\|^2$.
    Then upper bounds of $\Phi_k$ and $\Phi_{k,j}$ are given in Lemma 11 of Appendix B.
    
    Lemma 12 \citep[Lemma 3.4]{stich2018local} gives an upper bound about $F(\textbf{u}_{k\tau})$ because \eqref{subcon} matches the inequality of Lemma 12. Therefore, we can prove the convergence of C-DFL as presented in Lemma 13. Finally, Proposition \ref{convergenceofCSDFL} follows from Lemma 13  using the inequality $\mathbb{E}\mu\|\textbf{u}_0-\textbf{u}^*\|\leq 2G$ derived in \citep[Lemma 2]{rakhlin2012making}.
\end{proof}

When $K$ and $\bar{\sigma}$ are sufficiently large, the last two terms in \eqref{convergence-csdfl} are negligible compared to $\mathcal{O}\left(\frac{\tau_1\bar{\sigma}^2}{\mu K N}\right)$. Furthermore, when $\tau_1=1$, this dominant term becomes $\mathcal{O}\left(\frac{\bar{\sigma}^2}{\mu K N}\right)$, which is consistent with the result of CHOCO-SGD proposed in \cite{koloskova2019decentralized}. This also shows the linear convergence rate of decentralized SGD methods. Regarding the computation frequency $\tau_1$, it is easy to find that a larger value of $\tau_1$ leads to a larger convergence upper bound. Regarding the communication frequency $\tau_2$, the second term $\mathcal{O}\left(\!\!\frac{\kappa\tau_1^4 G^2 (1-p)^{\tau_2}}{\mu K^2[1-(1+\theta)(1-p)^{\tau_2}]}\!\!\right)$ decreases with $\tau_2$. The result verifies the impact analysis of $\tau_1$ and $\tau_2$ in Remark \ref{remark1}. 
The second term in \eqref{convergence-csdfl} decreases with $\delta$, which illustrates that a smaller compression rate lets the convergence get worse when improving the communication efficiency by compression.

\section{Simulation and discussion}
In this section, we study the DFL and C-DFL frameworks with embedding CNN into each node, and evaluate the convergence performance of DFL and communication efficiency of C-DFL based on MNIST and CIFAR-10 datasets. 

\subsection{Simulation Setup}
Based on Torch module of Python environment, we build a decentralized FL framework where each node can only communicate with its connected nodes. Because sparse topologies are able to improve the convergence of DFL significantly, we use ring topology and quasi-ring topology for experiments, and set 10 nodes in the two ring topologies as shown in Fig. \ref{ring-torus}. For the two topologies, each node updates its model parameter by taking an average of parameters from the connected neighbors and itself. 
Note that for ring topology, the second largest eigenvalue $\zeta=0.87$, and for quasi-ring topology, $\zeta=0.85$. 
We use CNN to train the local model and choose the cross-entropy as the loss function. The specific CNN structure is presented in Appendix C. 
\begin{figure}[!t]
	\centering
	\includegraphics[scale=0.2]{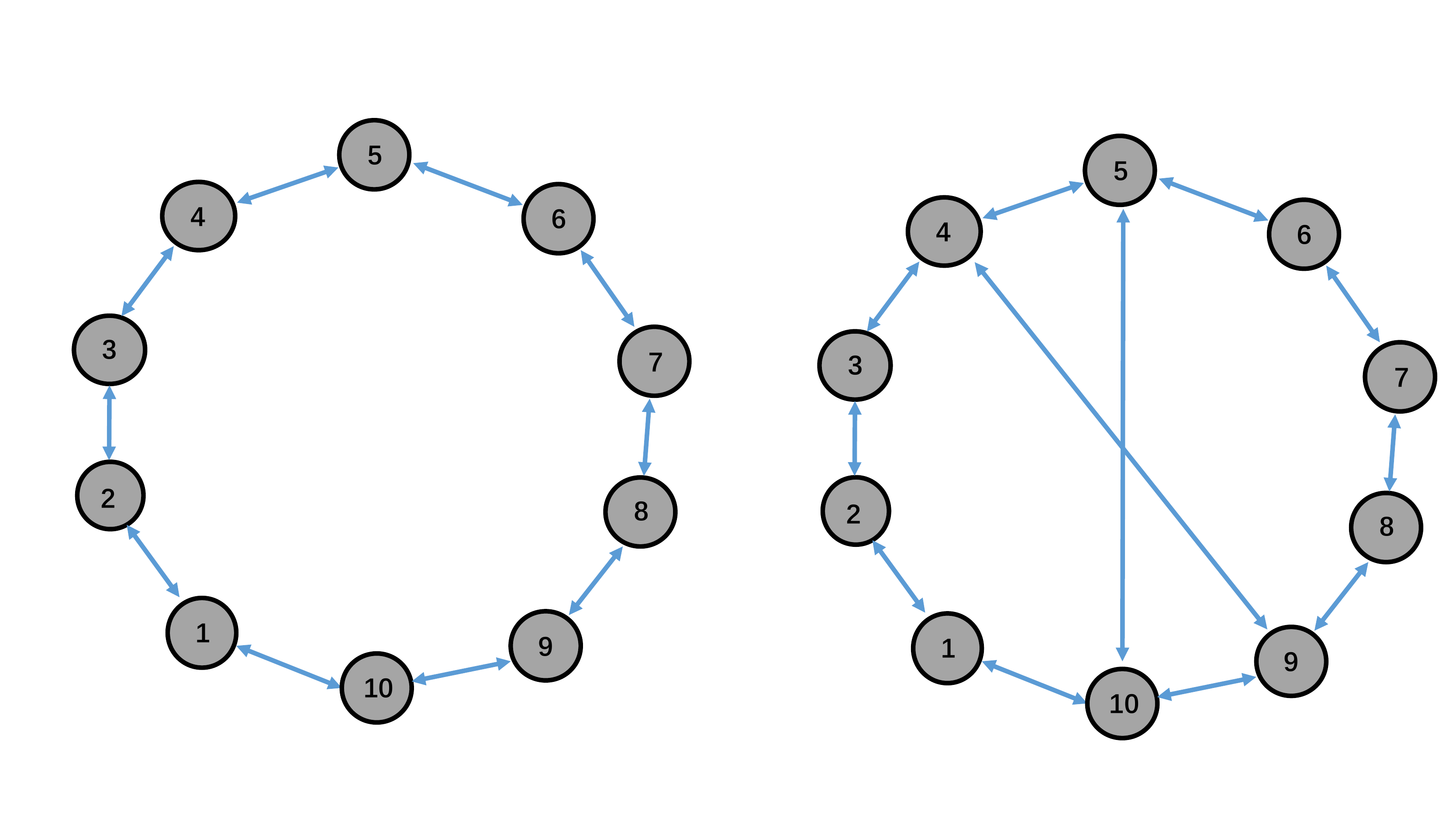}        
	\caption{Ring topology (left) and Quasi-ring topology (right)}
	\label{ring-torus}
\end{figure}

In our experiments, we use MNIST \cite{lecun1998gradient} and CIFAR-10 \cite{krizhevsky2009learning} dataset for CNN training and testing. 
Considering actual application scenarios of DFL where different devices have different distribution of data, statistical heterogeneity is a main feature for simulation.
So we set that the distribution of the training data samples is non-i.i.d. and the distribution of the testing data is i.i.d. 

For experimental setup, 
SGD is used for local update steps.
At the beginning of learning, the same initialization is applied with $\textbf{w}_0^{(1)}=\textbf{w}_0^{(2)}=\cdots=\textbf{w}_0^{(N)}$. In the experiments of MNIST, the learning rate is set as $\eta=0.002$. In the experiments of CIFAR-10, the learning rate is set as $\eta=0.008$.

\subsection{Simulation Evaluation}
In this subsection, we simulate DFL algorithm with CNN and verify the improved convergence of DFL compared with C-SGD. Then we explore the effect of the parameters $\tau_1$, $\tau_2$ and $\zeta$ on the convergence performance. Furthermore, we verify the convergence of C-DFL and illustrate the communication efficiency under different compressed communication themes.

\begin{figure*}[!t]
	\centering
	\subfloat[Ring topology trained with MNIST]{\includegraphics[scale=0.3]{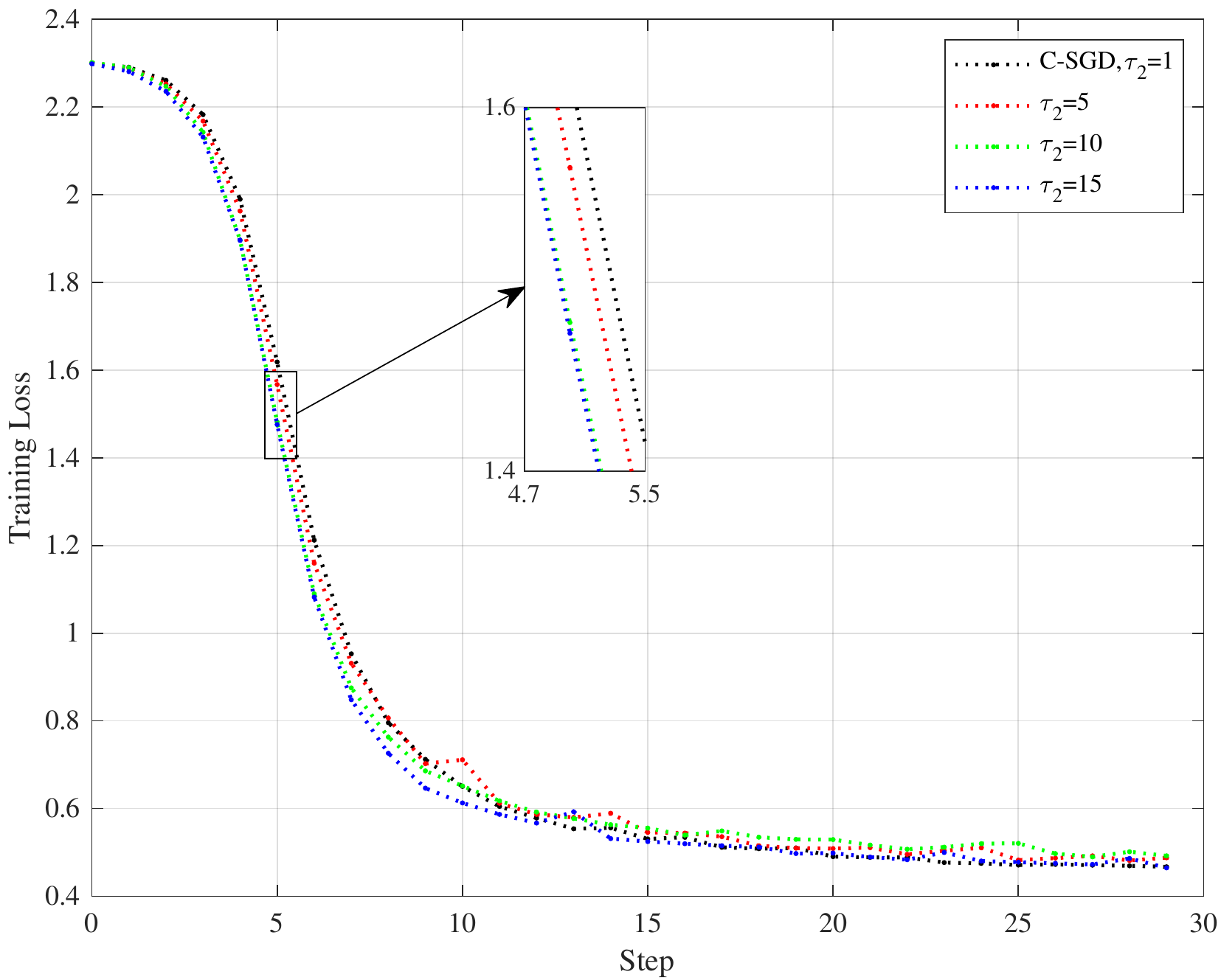}\label{mnist_loss}}
	\subfloat[Ring topology trained with MNIST]{\includegraphics[scale=0.3]{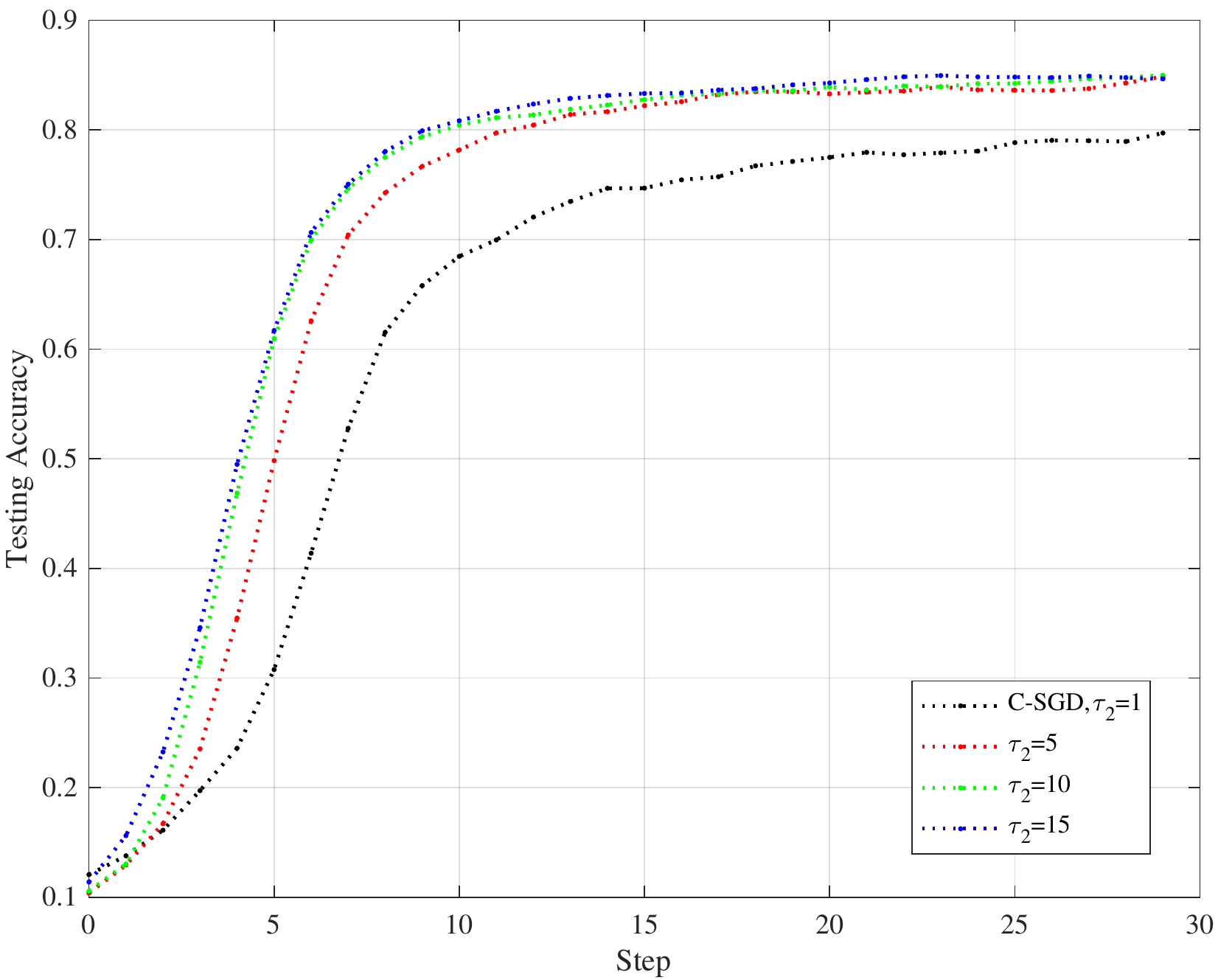}\label{mnist_acc}}
	\subfloat[Ring topology trained with CIFAR-10]{\includegraphics[scale=0.3]{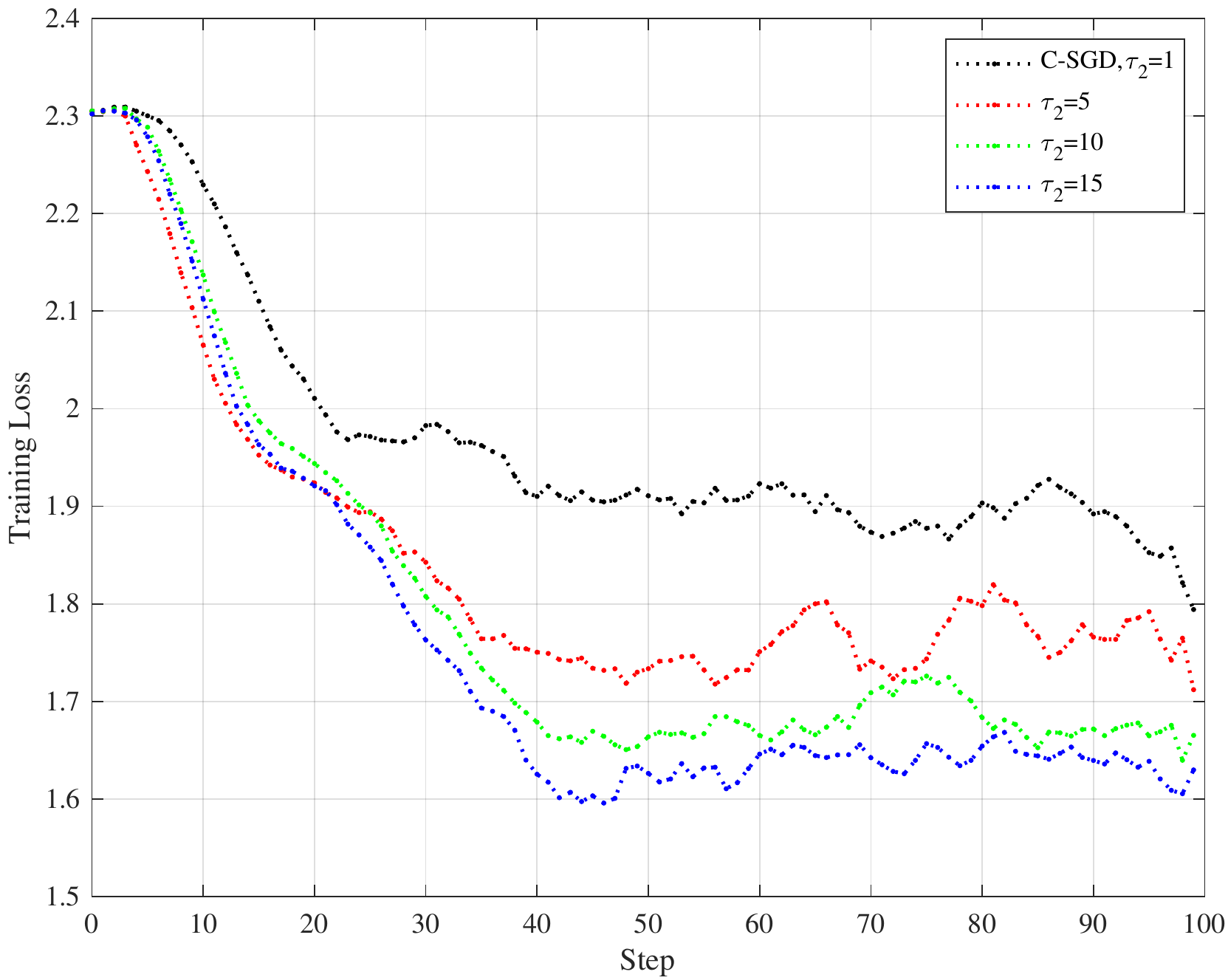}\label{cifar_loss}}
	\\
	\subfloat[Ring topology trained with CIFAR-10]{\includegraphics[scale=0.3]{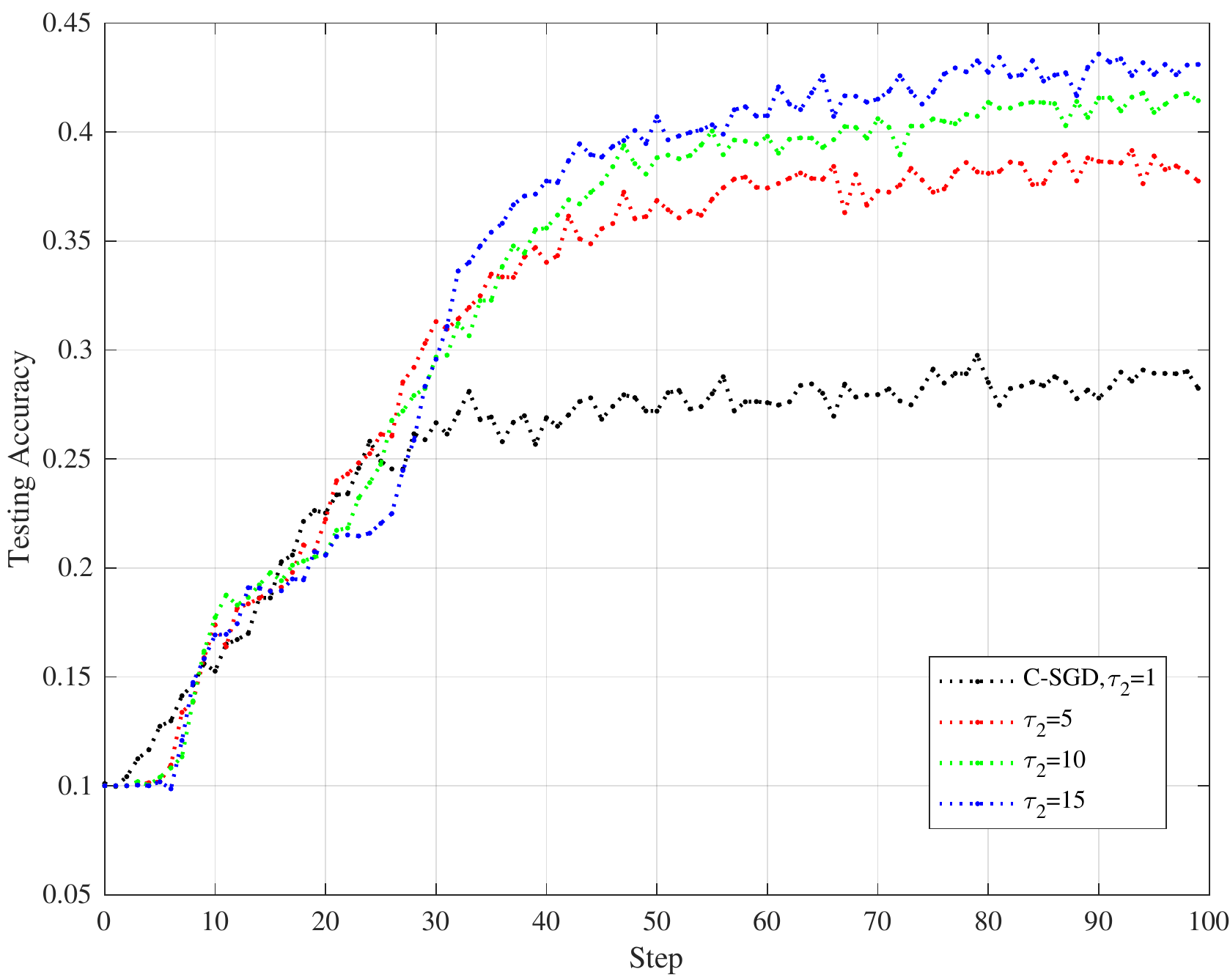}\label{cifar_acc}}
	\subfloat[Quasi-ring topology trained with MNIST]{\includegraphics[scale=0.3]{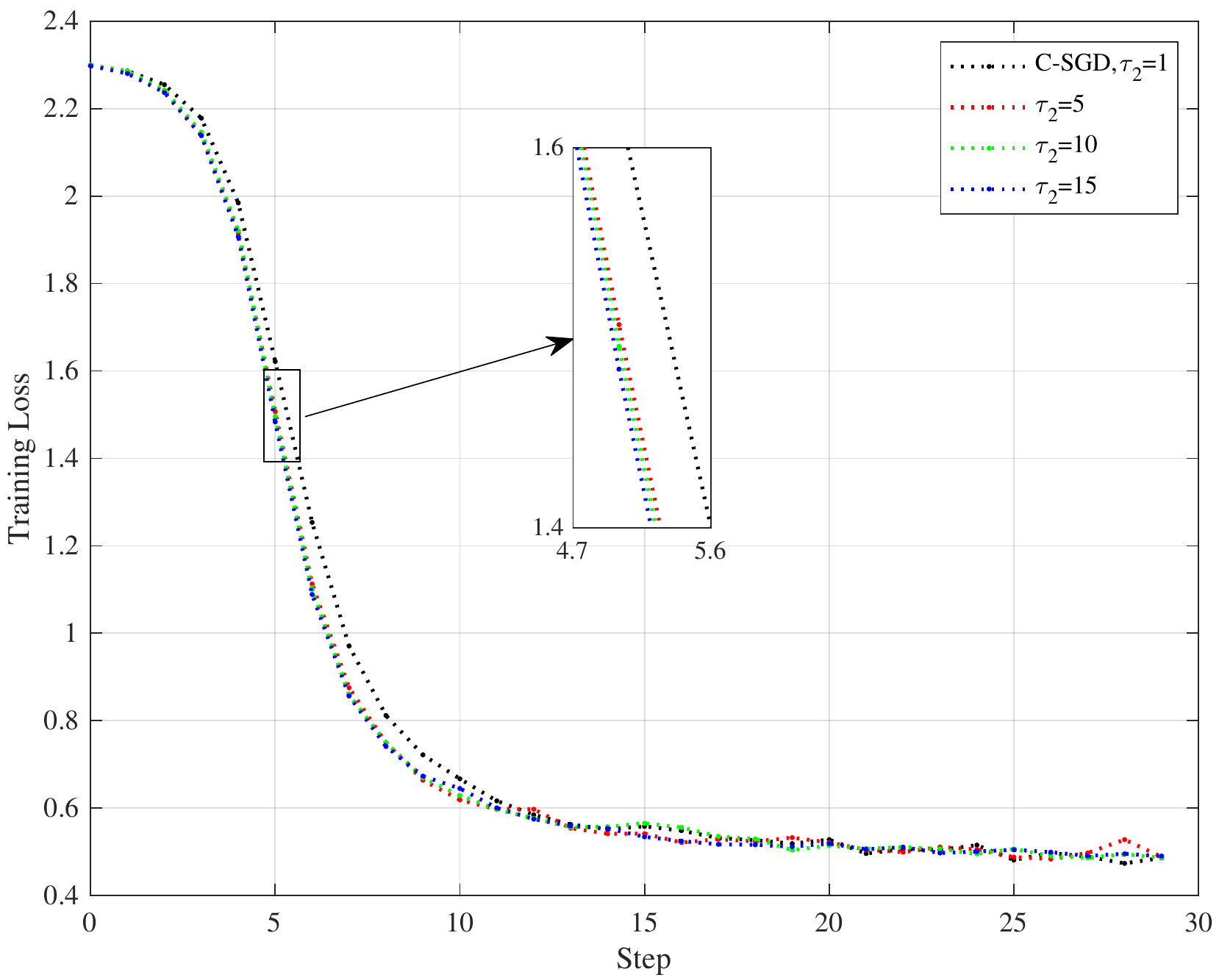}\label{torus-loss}}
	\subfloat[Quasi-ring topology trained with MNIST]{\includegraphics[scale=0.3]{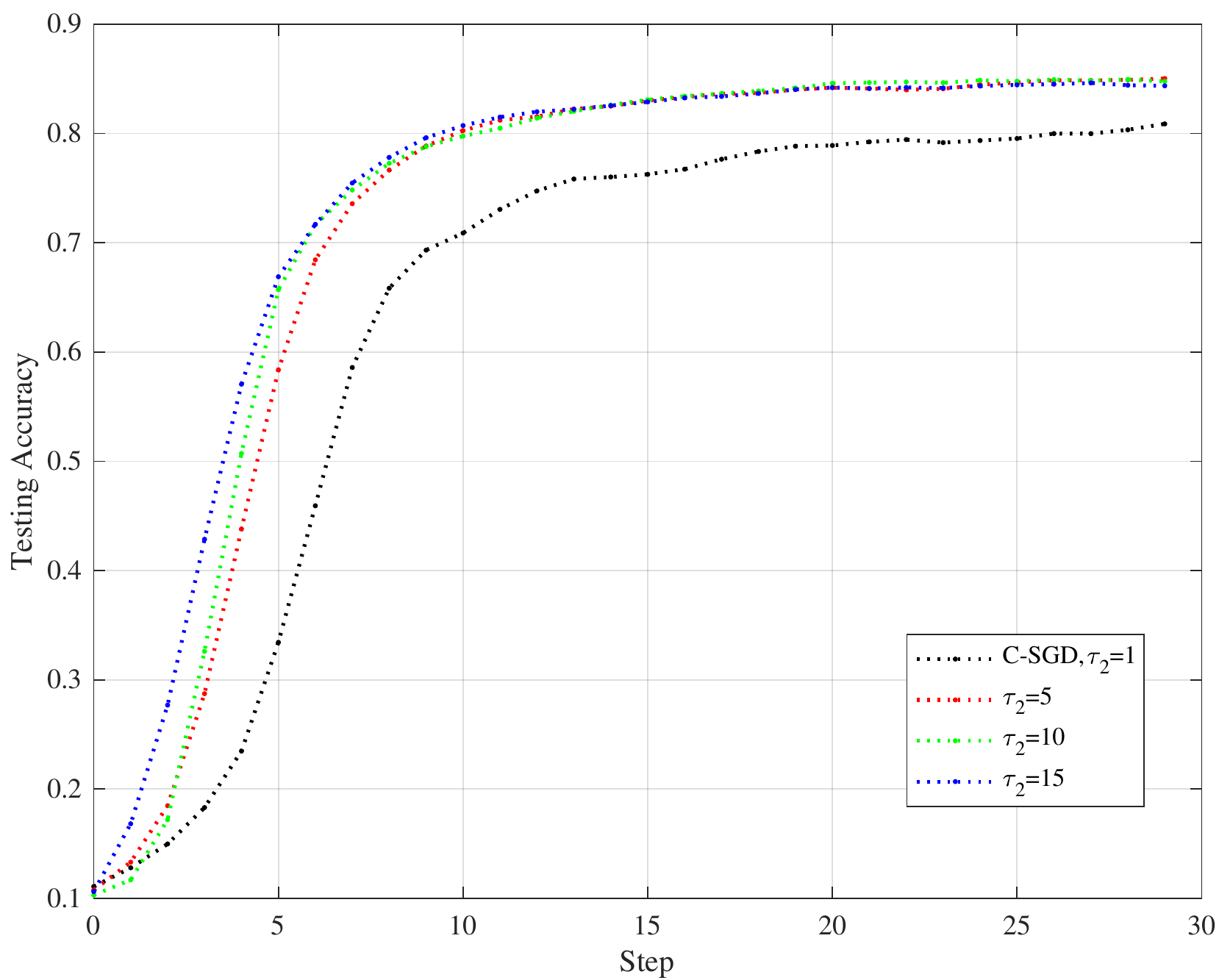}\label{torus-acc}}
	\caption{Experiments on MNIST and CIFAR-10 datasets with CNN under different $\tau_2$. There are 10 nodes for ring topology and quasi-ring topology. We set the computation frequency $\tau_1=4$. Ring topology uses both MNIST and CIFAR-10 as the training dataset, and quasi-ring topology is trained with MNIST. For the training on MNIST, we set the learning rate $\eta=0.002$. For the training on CIFAR-10, we set the learning rate $\eta=0.008$.}
	\label{convergencecurve}
\end{figure*}

\subsubsection{The accelerated convergence of DFL}
We present the accelerated convergence of DFL and use C-SGD as the benchmark, where C-SGD is a special case of DFL with communication frequency $\tau_2=1$. We set the fixed computation frequency $\tau_1=4$ for DFL and C-SGD. For ring topology, both DFL and C-SGD are performed based on two CNN models which use MNIST and CIFAR-10 for training and testing, respectively. For quasi-ring topology, DFL and C-SGD algorithms are only based on the CNN model with MNIST.  Under the same dataset, the global loss functions of the two algorithms are identical.

The curves of training loss and testing accuracy are presented in Fig. \ref{convergencecurve}. We can see that the training loss and testing accuracy curves converge gradually with steps based on the two datasets, thereby validating the efficient convergence of DFL.  We can further find out that the training loss and testing accuracy of DFL outperform these of C-SGD under the same condition. 
This verifies that DFL accelerates the convergence compared with C-SGD. 

\subsubsection{Effect of $\tau_2$}
As shown in Fig. \ref{convergencecurve}, we can find that under the same number of steps, the case of $\tau_2=15$ shows the best convergence of training loss and prediction performance and the case of $\tau_2=1$ shows the worst. A larger $\tau_2$ leads to a better convergence whether in training loss or testing accuracy, thereby verifying our analysis that the convergence upper bound decreases with $\tau_2$ in Remark \ref{remark1}. Thus, more inter-node communications in a round can improve the learning performance of DFL. 

\begin{figure}[!t]
	\centering
	\includegraphics[scale=0.31]{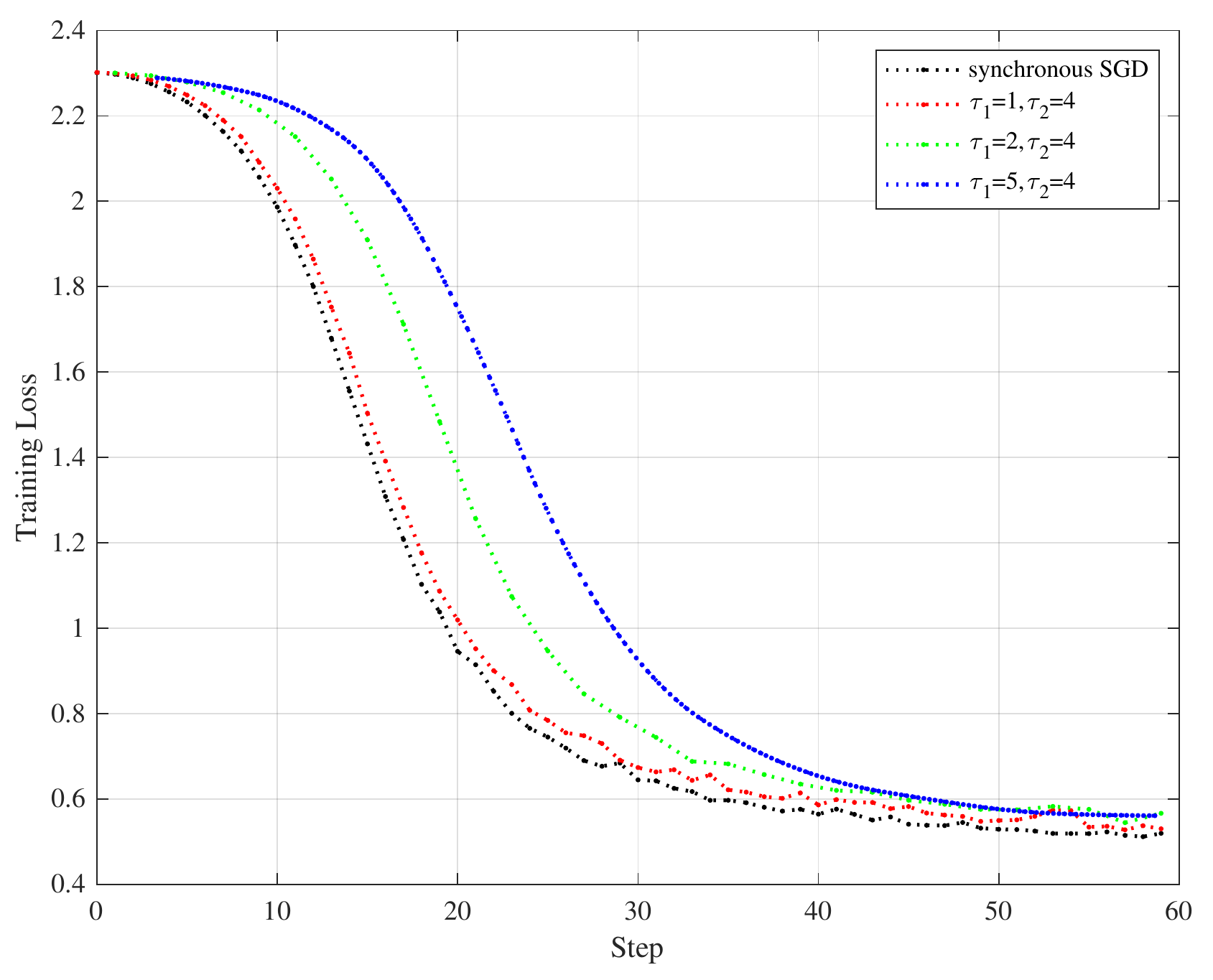}        
	\caption{Training loss curves with different $\tau_1$. The experiment is based on MNIST dataset. There are 10 edge nodes and this experiment is based on ring topology. The learning rate is $\eta=0.002$ and the communication frequency is $\tau_2=4$.}
	\label{tau1}
\end{figure}

\subsubsection{Effect of $\tau_1$}
Based on MNIST dataset, we simulate the training loss curves with different computation frequency $\tau_1$.  This experiment is based on ring topology.
We use synchronous SGD as the benchmark where the computational frequency is set as $\tau_1=1$ and the confusion matrix is set as $\textbf{C}=\textbf{J}$ to achieve complete model averaging.
The experimental result is presented in Fig. \ref{tau1}. We can find out that under the same number of steps, synchronous SGD outperforms DFL, thereby verifying Corollary \ref{cor1} that local update once with sufficient model averaging in a round leads to the optimal convergence of DFL. We can see that as $\tau_1$ goes from $1$ to $10$, the convergence of training loss becomes much worse. This is because more local updates lead to a worse local drift. Therefore, the simulation results confirm that the convergence bound increases with $\tau_1$ as discussed in Remark \ref{remark1}.

\begin{figure}[!t]
	\centering
	\includegraphics[scale=0.31]{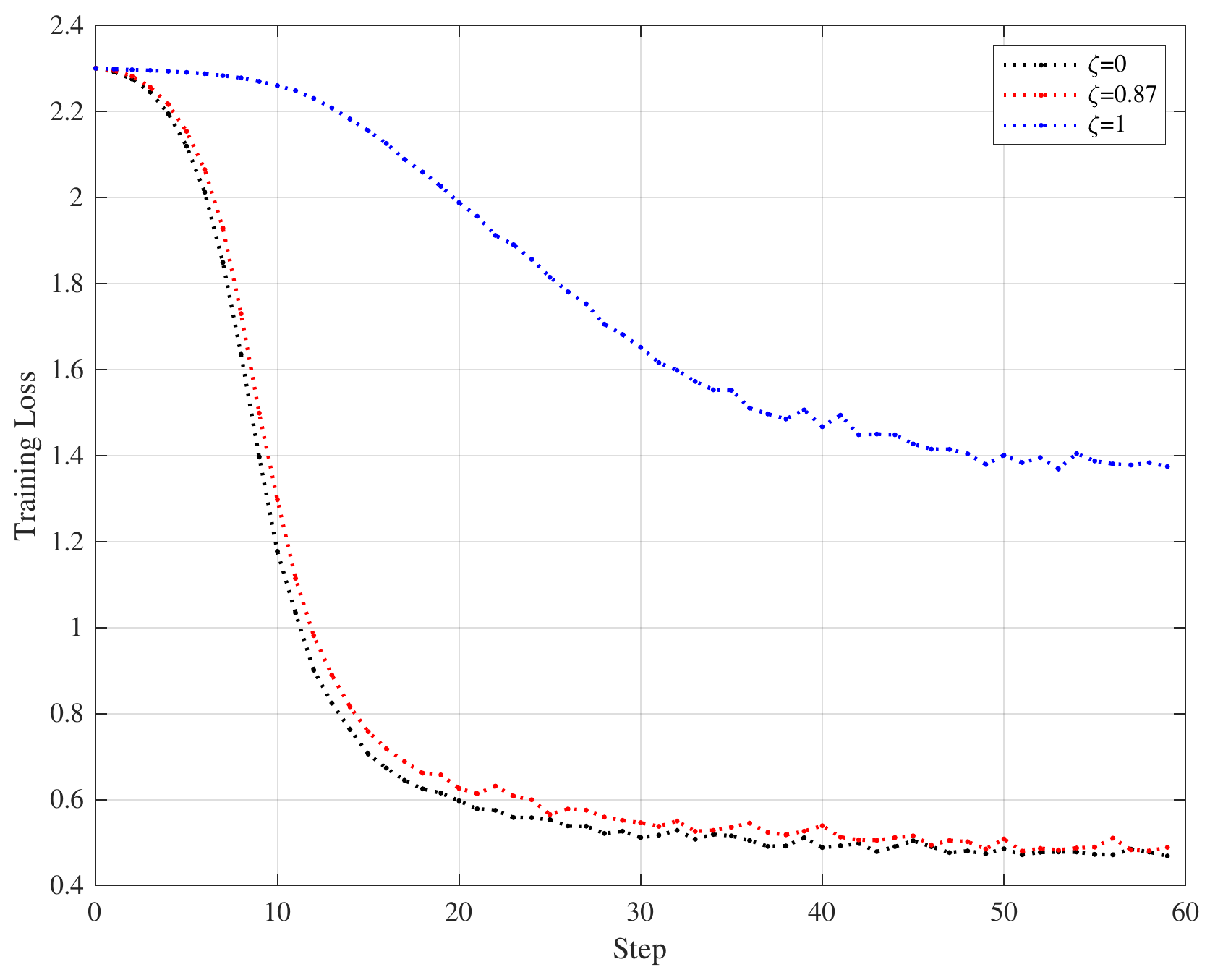}        
	\caption{Training loss curves with different $\zeta$. The experiment is based on MNIST dataset. There are 10 edge nodes. We set the learning rate $\eta=0.002$, the computation frequency $\tau_1=2$ and the communication frequency $\tau_2=4$.}
	\label{zeta}
\end{figure}
\begin{figure}[!t]
	\centering
	\subfloat[Training loss with wall-clock time]{\includegraphics[scale=0.31]{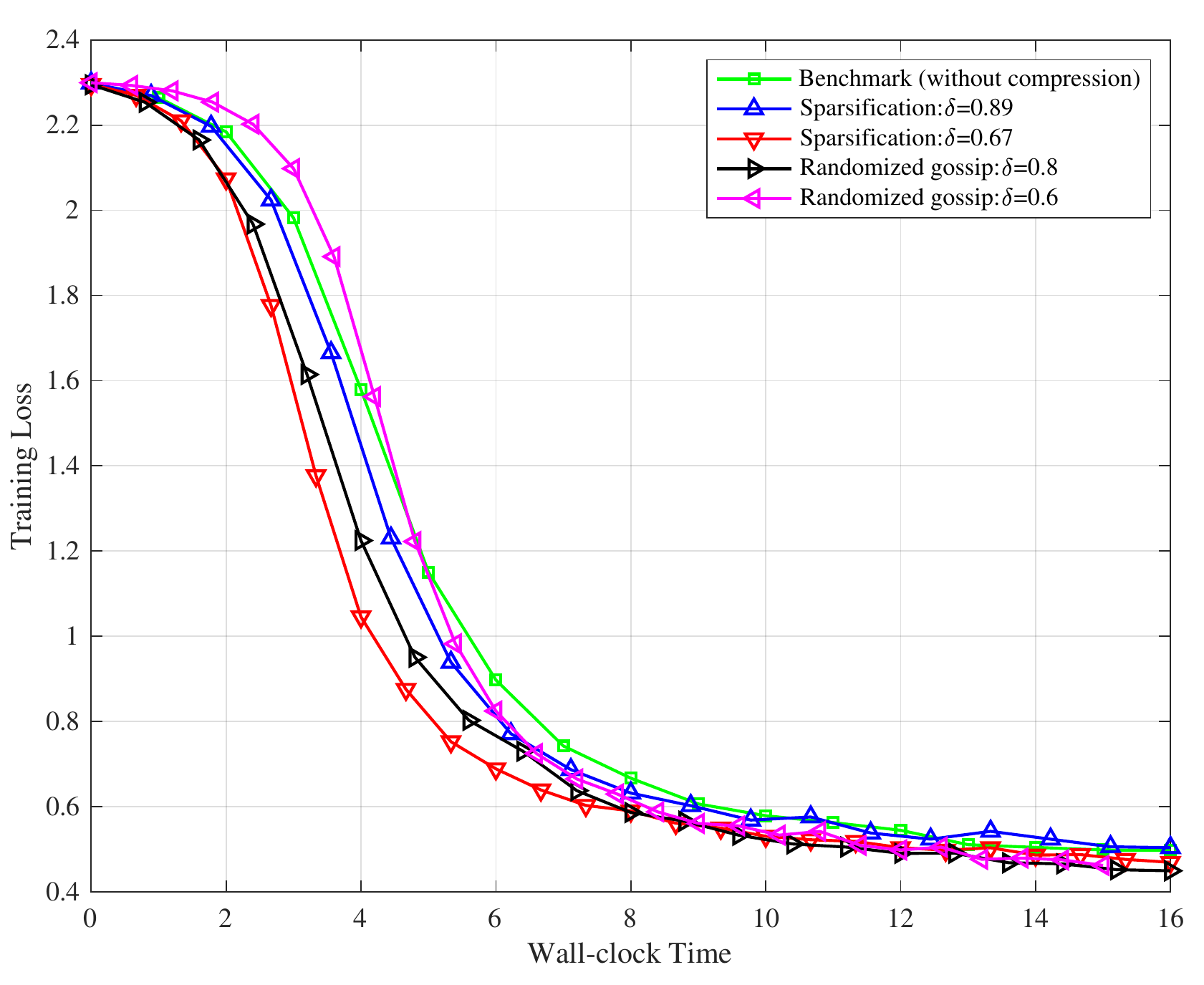}\label{time}}
	\\
	\subfloat[Training loss with step]{\includegraphics[scale=0.31]{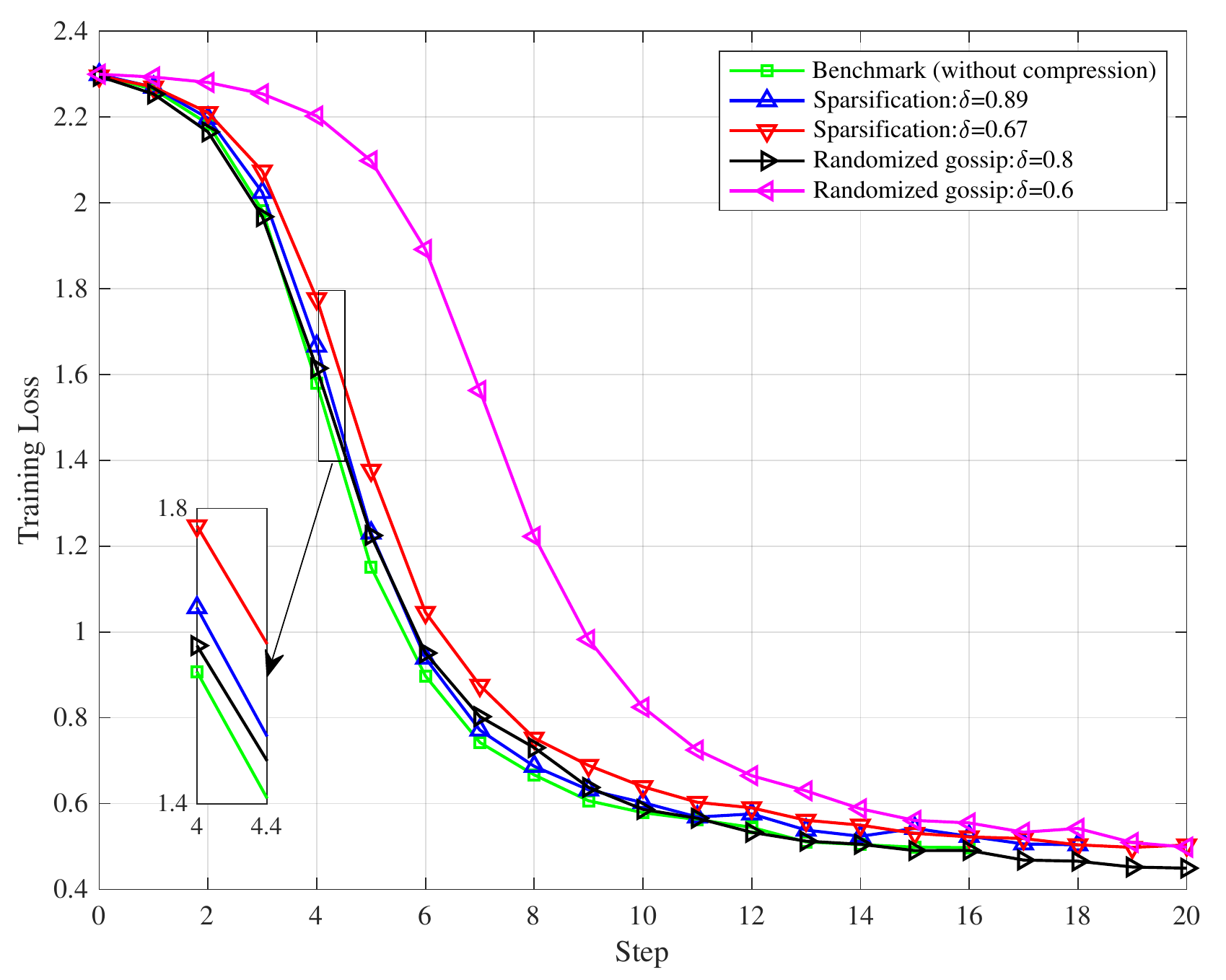}\label{epoch}}
	\caption{C-DFL training on MNIST with CNN. The ring topology is used with $10$ nodes. The computation frequency $\tau_1$ and communication frequency $\tau_2$ are set as $4$ and $4$, respectively. The consensus step size is $\gamma=1$ and the learning rate $\eta$ is $0.002$. }
	\label{compressionfigure}
\end{figure}

\subsubsection{Effect of $\zeta$}
We also explore the impact of $\zeta$ on DFL convergence based on MNIST dataset. In this experiment, we set the computation frequency $\tau_1=2$ and the communication frequency $\tau_2=4$. We use DFL with $\zeta=0$ as the benchmark.
From Fig. \ref{zeta}, we can see that DFL with $\zeta=0$ exhibits the best convergence performance compared with DFL with $\zeta>0$. From the perspective of network topology, condition $\zeta=0$ corresponds to the confusion matrix $\textbf{C}=\textbf{J}$ where the network has the densest inter-node connection. So model averaging is achieved, meaning that there is no local drift caused by insufficient inter-node communication and the convergence of synchronous SGD outperforms that of DFL with $\zeta>0$. We can further find out that a larger $\zeta$ results in a worse convergence of DFL. This is because a larger $\zeta$ corresponds to a sparser confusion matrix or network topology. Then less model averaging is accomplished in the sparser topology. Therefore, the simulation of Fig. \ref{zeta}
verifies the impact of $\zeta$ in Remark \ref{remark2}.

\subsubsection{C-DFL}
In this simulation, we use the sparsification and randomized gossip as compression operators which are introduced in Section \ref{con-csdfl}. To show the improved communication efficiency of C-DFL, we use DFL (without compression) as the benchmark. The training is based on MNIST with CNN and the decentralized network is a ring topology. The computation frequency $\tau_1$ and communication frequency $\tau_2$ are set as $4$ and $4$, respectively. We set the consensus step size $\gamma=1$ and the learning rate $\eta=0.002$.

From Fig. \ref{compressionfigure}, we can see that C-DFL converges under the two compression operators, which are Sparsification and Randomized gossip, respectively. Focusing on Fig. \ref{compressionfigure}(a), we find that C-DFL with Sparsification and Randomized gossip converges at a higher rate than DFL. For example, when wall-clock time is $4$, we obtain the convergences of Sparsification with $\delta=0.89$ and $\delta=0.67$ are increased by $16.9\%$ and $74.6\%$ than DFL, respectively. The convergence performance of Randomized gossip with $\delta=0.8$ rises by $50.7\%$ compared with DFL. This shows the communication efficiency of C-DFL because C-DFL transmits fewer bits and accelerates inter-node communication practically.
Note that when wall-clock time is smaller than $5$, the convergence of Randomized gossip with $\delta=0.6$ is worse than DFL.  We can find that Randomized gossip with $\delta=0.6$ shows the worst convergence in Fig. \ref{compressionfigure}(b). The reason is because the compression ratio $\delta$ is relatively small and too much information is lost in compression so that the benefits of reduced communication overhead brought by compression can not offset the inferior convergence performance due to information loss. Therefore, a smaller compression ratio could result in inferior convergence performance.
Focusing on Fig. \ref{compressionfigure}(b), we can see that convergence of C-DFL is inferior to that of DFL and a smaller compression ratio lets the convergence get worse. This is consistent with the convergence analysis of C-DFL in Proposition \ref{convergenceofCSDFL}.

\section{Conclusion}
In this paper, we have proposed a general decentralized machine learning framework named DFL to jointly consider model consensus and communication efficiency. DFL performs multiple local updates and
multiple inter-node communications. We have established the strong convergence guarantees for DFL without the assumption of convex objective.
The analysis about the convergence upper bound of DFL shows that DFL possesses superior convergence properties compared with C-SGD. Then, we have proposed C-DFL with compressed communication to further improve communication efficiency of DFL. The convergence analysis of C-DFL has indicated a linear convergence behavior.
Finally, based on CNN with MNIST and CIFAR-10 datasets, experimental simulations have been made. The results have validated the theoretical analysis of DFL and C-DFL. 

\small
\bibliographystyle{unsrt} 
\bibliography{bare_jrnl}
\ifCLASSOPTIONcaptionsoff
  \newpage
\fi

\appendices
\newtheorem{lemma}{Lemma}
\section{Proof of Proposition 1}\label{appendix2}
In order to prove Proposition 1, we first present the preliminaries of some notations for facilitating the proof, and then
give a supporting lemma of DFL convergence where the error convergence bound of the synchronous SGD and the local drift brought by DFL are presented. Lastly, we provide the proof of Proposition 1 based on the supporting lemma.
\subsection{Preliminaries of Proof}
For convenience, we first define $\Xi_t$ to denote the set $\{\xi_t^{(1)},\xi_t^{(2)},\dots,\xi_t^{(N)}\}$ of mini-batches at $N$ nodes in the $t$-th iteration step. We use $\textbf{E}_t$ to denote the conditional expectation $\mathbb{E}_{\Xi_t|\textbf{X}_t}$.
In addition, for $t\in[k]_1$, we define the averaged stochastic gradient and the averaged global gradient as the following:
\begin{align}
\label{GGt}\mathcal{G}_t&=\frac{1}{N}\sum_{i=1}^{N}g(\textbf{w}_t^{(i)}),\\
\label{HHt}\mathcal{H}_t&=\frac{1}{N}\sum_{i=1}^{N}\nabla F(\textbf{w}_t^{(i)}).
\end{align}
And when $t\in[k]_2$, DFL performs inter-node communication without gradient descent so that $\mathcal{G}_t=0,\mathcal{H}_t=0$ for $t\in[k]_2$.
Similar to the definition of $\textbf{X}_t,\textbf{G}_t$, we define the global gradient matrix $\nabla F(\textbf{X}_t)$ as follows:
\begin{align}
\nabla F(\textbf{X}_t) \triangleq [\nabla F(\textbf{w}_t^{(1)}),\nabla F(\textbf{w}_t^{(2)}),\dots,\nabla F(\textbf{w}_t^{(N)})].
\end{align}
We further define the Frobenius norm and the operator norm of matrix as follows to make analysis easier.
\begin{definition}[\textbf{The Frobenius norm} \cite{horn2012matrix}]\label{fronorm}
  \rm The Frobenius norm defined for matrix $\textbf{\textup{A}}\in M_n$ by
  \begin{align}
  \|\textbf{\textup{A}}\|_{\textup{F}}^2=|\textup{Tr}(\textbf{A}\textbf{A}^\top)|=\sum_{i,j=1}^{n}|a_{ij}|^2.
  \end{align}
\end{definition}
\begin{definition}[\textbf{The operator norm} \cite{horn2012matrix}]\label{op-norm}
  \rm The operator norm defined for matrix $\textbf{\textup{A}}\in M_n$ by
  \begin{align}
  \|\textbf{A}\|_{\textup{op}}=\max_{\|\textbf{x}\|=1}\|\textbf{A}\textbf{x}\|=\sqrt{\lambda_{\max}(\textbf{A}^\top \textbf{A})}.
  \end{align}
\end{definition}
According to the definition \ref{fronorm}, the Frobenius norm of the global gradient matrix is 
$$\|\nabla F(\textbf{X}_t)\|^2_{\textup{F}}=\sum_{i=1}^{N}\|\nabla F(\textbf{w}_t^{(i)})\|^2.$$

\subsection{The Supporting Lemma}
We provide the supporting lemma for the proof of Proposition 1. The lemma establishes the basic convergence of DFL and the error convergence bound consists of two parts which are caused by the synchronous SGD and the local drift, respectively. We present the supporting lemma as follows:
\begin{lemma}\label{lem1}
  Under Assumption 1, if the learning rate satisfies the condition $\eta L(1+\frac{\beta}{N})\leq 1$ and all local models have the same initial point $\textbf{\textup{w}}_1$, then for Algorithm 1 of DFL, the average-squared gradient of $T$ steps is bounded as 
  \begin{align}
  \mathbb{E}\left[\frac{1}{T}\sum_{t=1}^{T}\|\nabla F(\textbf{\textup{u}}_t)\|^2\right]\leq \underbrace{\frac{2[F(\textbf{\textup{w}}_1)-F_{\textup{inf}}]}{\eta T}+\frac{\eta L \sigma^2}{N}}_{\mbox{\footnotesize synchronous SGD}}+\notag\\
  \underbrace{\frac{L^2}{T}\sum_{t=1}^{T}\frac{\mathbb{E}\|\textbf{\textup{X}}_t(\textbf{\textup{I}}-\textbf{\textup{J}})\|^2_{\textup{F}}}{N}}_{\mbox{\footnotesize local drift}},
  \end{align} 
  where $\textbf{\textup{u}}_t$ is defined in \textup{(17)} and matrices $\textbf{\textup{I}},\textbf{\textup{J}}\in \mathbb{R}^{N\times N}$. 
\end{lemma}

In order to prove Lemma \ref{lem1}, we first present $3$ supporting lemmas as follows. Then based on the presented lemmas, we finish the proof of Lemma \ref{lem1}.

\begin{lemma}\label{lemma2}
  Under the condition 4 and 5 of Assumption 1, we have the following upper bound for the expectation of the discrepancy between the averaged stochastic gradient $\mathcal{G}_t$ and the averaged global gradient $\mathcal{H}_t$ as
  \begin{align}\label{38}
  \mathbb{E}_{\Xi_t|\textbf{\textup{X}}_t}[\|\mathcal{G}_t-\mathcal{H}_t\|^2]\leq\frac{\beta}{N^2}\|\nabla F(\textbf{\textup{X}}_t)\|^2_{\textup{F}}+\frac{\sigma^2}{N}.
  \end{align}
\end{lemma}
\begin{proof}
  According to the definition of $\mathcal{G}_t$ and $\mathcal{H}_t$ in \eqref{GGt} and \eqref{HHt}, respectively, we have 
  
  \begin{align*}
  &\mathbb{E}_{\Xi_t|\textbf{\textup{X}}_t}[\|\mathcal{G}_t-\mathcal{H}_t\|^2]\\
  =&\mathbb{E}_{\Xi_t|\textbf{\textup{X}}_t} \left\|\frac{1}{N}\sum_{i=1}^{N}\left[g(\textbf{w}_t^{(i)})-\nabla F(\textbf{w}_t^{(i)})\right]\right\|^2\\
  =&\frac{1}{N^2}\mathbb{E}_{\Xi_t|\textbf{\textup{X}}_t}\left[\sum_{i=1}^{N}\left\|g(\textbf{w}_t^{(i)})-\nabla F(\textbf{w}_t^{(i)})\right\|^2\right]+\\
  &\frac{1}{N^2}\mathbb{E}_{\Xi_t|\textbf{\textup{X}}_t}\sum_{j\neq l}^{N}  \left \langle  g(\textbf{w}_t^{(j)})-\nabla F(\textbf{w}_t^{(j)}), g(\textbf{w}_t^{(l)})-\nabla F(\textbf{w}_t^{(l)}) \right\rangle\\
  =&\frac{1}{N^2}\sum_{i=1}^{N}\mathbb{E}_{\xi_t^{(i)}|\textbf{X}_t}\left\|g(\textbf{w}_t^{(i)})-\nabla F(\textbf{w}_t^{(i)})\right\|^2+\\
  &\frac{1}{N^2}\sum_{j\neq l}^{N} \left \langle \mathbb{E}_{\xi_t^{(j)}|\textbf{X}_t}\left[g(\textbf{w}_t^{(j)})-\nabla F(\textbf{w}_t^{(j)})\right],\right.
  \\
  \phantom{=\;\;}
  &\left.\mathbb{E}_{\xi_t^{(l)}|\textbf{X}_t}\left[g(\textbf{w}_t^{(l)})-\nabla F(\textbf{w}_t^{(l)})\right]\right\rangle,
  \end{align*}
  where the last equation is due to $\{\xi_t^{(i)}\}$ are independent random variables. Then, according to the condition 4 in Assumption 1, we can obtain that the inner product term of all across terms in the last equation is zero. Thus, we have 
  \begin{align*}
  \mathbb{E}_{\Xi_t|\textbf{\textup{X}}_t}[\|\mathcal{G}_t-\mathcal{H}_t\|^2]&\leq \frac{1}{N^2}\sum_{i=1}^{N}\left[\beta\left\|\nabla F(\textbf{w}_t^{(i)})\right\|^2+\sigma^2\right]\\
  &=\frac{\beta}{N}\frac{\|\nabla F(\textbf{X}_t)\|^2_{\textup{F}}}{N}+\frac{\sigma^2}{N},
  \end{align*}
  where the last equation is due to the condition 5 in Assumption 1 and the definition of the Frobenius norm of $\nabla F(\textbf{X}_t)$. 
\end{proof}

\begin{lemma}\label{lem3}
  Based on Assumption 1, the expectation of the inner product between stochastic gradient and global gradient can expanded as 
  \begin{align}\label{inner-pro-bet}
  \textbf{\textup{E}}_t[\langle\nabla F(\textbf{\textup{u}}_t),\mathcal{G}_t\rangle]=&\frac{1}{2}\|\nabla F(\textbf{\textup{u}}_t)\|^2 +\frac{1}{2N}\sum_{i=1}^{N}\left\|\nabla F(\textbf{\textup{w}}_t^{(i)})\right\|^2\notag\\
  &-\frac{1}{2N}\sum_{i=1}^{N}\left\|\nabla F(\textbf{\textup{u}}_t)-\nabla F(\textbf{\textup{w}}_t^{(i)})\right\|^2,
  \end{align} 
  where $\textbf{\textup{E}}_t$ is the conditional expectation $\mathbb{E}_{\Xi_t|\textbf{\textup{X}}_t}$.
\end{lemma}
\begin{proof}
  We can directly prove equation \eqref{inner-pro-bet} as follows:
  \begin{align}
  &\textbf{\textup{E}}_t[\langle\nabla F(\textbf{u}_t),\mathcal{G}_t\rangle]\notag\\
  =&\textbf{\textup{E}}_t\left[\left\langle\nabla F(\textbf{u}_t),\frac{1}{N}\sum_{i=1}^{N}g(\textbf{w}_t^{(i)})\right\rangle\right]\notag\\
  =&\frac{1}{N}\sum_{i=1}^{N}\left\langle\nabla F(\textbf{u}_t),\nabla F(\textbf{w}_t^{(i)})\right\rangle\notag\\
  \label{44}=&\frac{1}{2N}\!\sum_{i=1}^{N}\!\left[\!\|\nabla F(\textbf{u}_t)\|^2\!+\!\left\|\nabla F(\textbf{w}_t^{(i)})\right\|^2\!-\!\left\|\nabla F(\textbf{u}_t)\!-\!\nabla F(\textbf{w}_t^{(i)})\right\|^2\!\right]\\
  =&\frac{1}{2}\|\nabla F(\textbf{u}_t)\|^2+\frac{1}{2N}\sum_{i=1}^{N}\left\|\nabla F(\textbf{w}_t^{(i)})\right\|^2\notag\\
  &-\frac{1}{2N}\sum_{i=1}^{N}\left\|\nabla F(\textbf{u}_t)\!-\!\nabla F(\textbf{w}_t^{(i)})\right\|^2,
  \end{align}
  where equation \eqref{44} is because $\textbf{a}^\top \textbf{b}=\frac{1}{2}(\|\textbf{a}\|^2+\|\textbf{b}\|^2-\|\textbf{a}-\textbf{b}\|^2)$.
\end{proof}

\begin{lemma}\label{lem4}
  Based on the Assumption 1, the squared norm of stochastic gradient can be bounded as
  \begin{align}\label{42}
  \textbf{\textup{E}}_t\left[\|\mathcal{G}_t\|^2\right]\leq\left(\frac{\beta}{N}+1\right)\frac{\|\nabla F(\textbf{\textup{X}}_t)\|_{\textup{F}}^2}{N}+\frac{\sigma^2}{N}.
  \end{align}
\end{lemma}
\begin{proof}
  According to the condition 4 of Assumption 1, we have $\textbf{\textup{E}}_t[\mathcal{G}_t]=\mathcal{H}_t$. Then we obtain
  \begin{align}
  \textbf{\textup{E}}_t\left[\|\mathcal{G}_t\|^2\right]&=\textbf{\textup{E}}_t\left[\|\mathcal{G}_t-\textbf{E}_t[\mathcal{G}_t]\|^2\right]+\|\textbf{\textup{E}}_t[\mathcal{G}_t]\|^2\\
  &=\textbf{E}_t\left[\|\mathcal{G}_t-\mathcal{H}_t\|^2\right]+\|\mathcal{H}_t\|^2\\
  \label{45}&\leq\frac{\beta}{N}\frac{\|\nabla F(\textbf{X}_t)\|^2_{\textup{F}}}{N}+\frac{\sigma^2}{N}+\|\mathcal{H}_t\|^2\\
  \label{46}&\leq\frac{\beta}{N}\frac{\|\nabla F(\textbf{X}_t)\|^2_{\textup{F}}}{N}+\frac{\sigma^2}{N}+\frac{1}{N}\|\nabla F(\textbf{X}_t)\|^2_{\textup{F}}\\
  &=\left(\frac{\beta}{N}+1\right)\frac{\|\nabla F(\textbf{X}_t)\|^2_{\textup{F}}}{N}+\frac{\sigma^2}{N},
  \end{align}
  where \eqref{45} is from Lemma \ref{lemma2}, and \eqref{46} is because the convexity of vector norm and Jensen's inequality:
  \begin{align*}
  \|\mathcal{H}_t\|^2&=\left\|\frac{1}{N}\sum_{i=1}^{N}\nabla F(\textbf{w}_t^{(i)})\right\|^2\\
  &\leq\frac{1}{N}\sum_{i=1}^{N}\left\|\nabla F(\textbf{w}_t^{(i)})\right\|^2\\
  &=\frac{1}{N}\|\nabla F(\textbf{X}_t)\|^2_{\textup{F}}.
  \end{align*}
\end{proof}
Obviously, Lemmas \ref{lemma2}, \ref{lem3}, \ref{lem4} are established under the condition $t\in[k]_1$.
However, Lemmas \ref{lemma2}, \ref{lem3}, \ref{lem4} still hold under the communication stage for $t\in [k]_2$. We argue that as follows in detail.
For the equation \eqref{38} of Lemma \ref{lemma2}, when $t\in[k]_2$, we can get the left side of the inequality $\mathbb{E}_{\Xi_t|\textbf{\textup{X}}_t}[\|\mathcal{G}_t-\mathcal{H}_t\|^2]=0$, and the right side of the inequality is greater than $0$. Thus, Lemma \ref{lemma2} is true for $t\in [k]_2$. Considering equation \eqref{inner-pro-bet} of Lemma \ref{lem3}, for $t\in [k]_2$, we get the left side of the equality $\textbf{\textup{E}}_t[\langle\nabla F(\textbf{\textup{u}}_t),\mathcal{G}_t\rangle]=0$ and the right side is $\frac{1}{2}\|\nabla F(\textbf{\textup{u}}_t)\|^2-\frac{1}{2N}\sum_{i=1}^{N}\left\|\nabla F(\textbf{\textup{u}}_t)\right\|^2=0$ due to 
there exists no gradient at inter-node communication stage. Therefore, Lemma \ref{lem3} is established for $t\in [k]_2$. Focusing on Lemma \ref{lem4}, for $t\in[k]_2$, we can get the left side of \eqref{42} $\textbf{\textup{E}}_t\left[\|\mathcal{G}_t\|^2\right] =0$, and the right side 
$\left(\frac{\beta}{N}+1\right)\frac{\|\nabla F(\textbf{\textup{X}}_t)\|_{\textup{F}}^2}{N}+\frac{\sigma^2}{N}>0$. Thus for $t\in [k]_2$, Lemma \ref{lem4} holds.

After presenting the three lemmas, we begin to prove Lemma \ref{lem1}. 

\textbf{Proof of Lemma \ref{lem1}:} According to the $L$-smooth assumption from Assumption 1, we have
\begin{align}\label{48}
\textbf{E}_t[F(\textbf{u}_{t+1})]\!-\!F(\textbf{u}_t)\leq-\eta\textbf{E}_t[\langle\nabla F(\textbf{u}_t),\mathcal{G}_t\rangle]\!+\!\frac{\eta^2L}{2}\textbf{E}_t\left[\|\mathcal{G}_t\|^2\right].
\end{align}
Note that $\textbf{u}_{t+1}=\textbf{u}_{t}$ for $t\in [k]_2$ and $\textbf{E}_t[F(\textbf{u}_{t+1})]\!-\!F(\textbf{u}_t)=0$. Because $\mathcal{G}_t=0$ for $t\in [k]_2$, the right side of \eqref{48} is $0$. Thus \eqref{48} holds for $t\in [k]_2$.
Combining with Lemma \ref{lem3} and \ref{lem4}, one can get 
\begin{align*}
&\textbf{E}_t[F(\textbf{u}_{t+1})]-F(\textbf{u}_t)\\
\leq&-\frac{\eta}{2}\|\nabla F(\textbf{u}_t)\|^2-\frac{\eta}{2N}\sum_{i=1}^{N}\left\|\nabla F(\textbf{w}_t^{(i)})\right\|^2+\\
&\frac{\eta}{2N}\sum_{i=1}^{N}\left\|\nabla F(\textbf{u}_t)-\nabla F(\textbf{w}_t^{(i)})\right\|^2+\\
&\frac{\eta^2L}{2N}\sum_{i=1}^{N}\left\|\nabla F(\textbf{w}_t^{(i)})\right\|^2\cdot\left(\frac{\beta}{N}+1\right)+\frac{\eta^2L\sigma^2}{2N}\\
\leq&-\frac{\eta}{2}\|\nabla F(\textbf{u}_t)\|^2\!\!-\!\!\frac{\eta}{2}\left[1\!\!-\!\!\eta L\left(\frac{\beta}{N}\!+\!1\right)\right]\!\!\cdot\!\!\frac{1}{N}\sum_{i=1}^{N}\left\|\nabla F(\textbf{w}_t^{(i)})\right\|^2\\
&+\frac{\eta^2L\sigma^2}{2N}+\frac{\eta L^2}{2N}\sum_{i=1}^{N}\left\|\textbf{u}_t-\textbf{w}_t^{(i)}\right\|^2.
\end{align*}
By rearranging the above inequality and according to the definition of Frobenius norm, we further get
\begin{align}
&\|\nabla F(\textbf{u}_t)\|^2\leq\frac{2[F(\textbf{u}_t)-\textbf{E}_t[F(\textbf{u}_{t+1})]]}{\eta}+\frac{\eta L\sigma^2}{N}+\notag\\
&\frac{L^2}{N}\sum_{i=1}^{N}\left\|\textbf{u}_t-\textbf{w}_t^{(i)}\right\|^2-\left[1-\eta L\left(\frac{\beta}{N}+1\right)\right]\frac{1}{N}\|\nabla F(\textbf{X}_t)\|^2_{\textup{F}}.
\end{align}
Taking the total expectation and averaging over all iteration steps, we can get 
\begin{align}\label{50}
&\mathbb{E}\left[\frac{1}{T}\sum_{t=1}^{T}\|\nabla F(\textbf{u}_t)\|^2\right]\leq\frac{2[F(\textbf{u}_1)-F_{\textup{inf}}]}{\eta T}+\frac{\eta L\sigma^2}{N}\notag\\
&+\frac{L^2}{TN}\sum_{t=1}^{T}\sum_{i=1}^{N}\mathbb{E}\left\|\textbf{u}_t-\textbf{w}_t^{(i)}\right\|^2\notag\\
&-\frac{1}{NT}\left[1-\eta L\left(\frac{\beta}{N}+1\right)\right]\sum_{t=1}^{T}\mathbb{E}\|\nabla F(\textbf{X}_t)\|^2_{\textup{F}}.
\end{align}
Recalling the definition $\textbf{u}_t=\textbf{X}_t\frac{\textbf{1}}{N}$, we can obtain 
\begin{align}\label{51}
\sum_{i=1}^{N}\left\|\textbf{u}_t-\textbf{w}_t^{(i)}\right\|^2=&\left\|\textbf{u}_t \textbf{1}^\top-\textbf{X}_t\right\|^2_{\textup{F}}\notag\\
=&\left\|\textbf{X}_t\frac{\textbf{1}\textbf{1}^\top}{N}-\textbf{X}_t\right\|^2_{\textup{F}}\notag\\
=&\|\textbf{X}_t(\textbf{I}-\textbf{J})\|^2_{\textup{F}}.
\end{align}
By substituting \eqref{51} into \eqref{50} and according to the condition $\eta L\left(1+\frac{\beta}{N}\right)\leq1$, we complete the proof of Lemma \ref{lem1}.

\subsection{Proof of Proposition 1}
We first provide three lemmas for assisting proof as follows.
\begin{lemma}\label{lem5}
  For two real matrices $\textbf{\textup{A}}\in \mathbb{R}^{d\times m}$ and $\textbf{\textup{B}}^{m\times m}$, if $\textbf{\textup{B}}$ is  symmetric, then we have 
  \begin{align}
  \|\textbf{\textup{A}}\textbf{\textup{B}}\|_{\textup{F}}\leq \|\textbf{\textup{B}}\|_{\textup{op}}\|\textbf{\textup{A}}\|_{\textup{F}}.
  \end{align} 
\end{lemma}
\begin{proof}
  Assuming that $\textbf{a}_1^\top,\dots,\textbf{a}_d^\top$ denote the rows of matrix $\textbf{A}$ and $\mathcal{I}=\{i\in[1,d]:\|\textbf{a}_i\|\neq 0\}$. Then we obtain
  \begin{align*}
  \|\textbf{A}\textbf{B}\|^2_{\textup{F}}&=\sum_{i=1}^{d}\|\textbf{a}_i^\top \textbf{B}\|^2=\sum_{i\in\mathcal{I}}^{d}\|\textbf{B}\textbf{a}_i\|^2\\
  &=\sum_{i\in\mathcal{I}}^{d}\frac{\|\textbf{B}\textbf{a}_i\|^2}{\|\textbf{a}_i\|^2}\|\textbf{a}_i\|^2\\
  &\leq\sum_{i\in\mathcal{I}}^{d}\|\textbf{B}\|^2_{\textup{op}}\|\textbf{a}_i\|^2=\|\textbf{B}\|^2_{\textup{op}}\sum_{i\in\mathcal{I}}^{d}\|\textbf{a}_i\|^2=\|\textbf{B}\|^2_{\textup{op}}\|\textbf{A}\|^2_{\textup{F}},
  \end{align*}
  where the last inequality is due to Definition \ref{op-norm} about matrix operator norm. 
\end{proof}
\begin{lemma}\label{lem6}
  Consider two matrices $\textbf{\textup{A}}\in\mathbb{R}^{m\times n}$ and $\textbf{\textup{B}}\in \mathbb{R}^{n\times m}$. We have 
  \begin{align}
  |\textup{Tr}(\textbf{\textup{A}}\textbf{\textup{B}})|\leq\|\textbf{\textup{A}}\|_{\textup{F}}\|\textbf{\textup{B}}\|_{\textup{F}}.
  \end{align}
\end{lemma}
\begin{proof}
  Assume $\textbf{a}_{i}^\top\in \mathbb{R}^{n}$ is the $i$-th row of matrix $\textbf{A}$ and $\textbf{b}_{i}^\top\in \mathbb{R}^{n}$ is the $i$-th column of matrix $\textbf{B}$. According to the definition of matrix trace, we have
  \begin{align}
  \textup{Tr}(\textbf{A}\textbf{B})&=\sum_{i=1}^{m}\sum_{j=1}^{n}\textbf{A}_{ij}\textbf{B}_{ji}\notag\\
  \label{54}&=\sum_{i=1}^{m}\textbf{a}_{i}^\top\textbf{b}_{i}.
  \end{align}
  Then according to Cauchy-Schwartz inequality, we further obtain
  \begin{align}
  |\sum_{i=1}^{m}\textbf{a}_{i}^\top\textbf{b}_{i}|^2&\leq\left(\sum_{i=1}^{m}\|\textbf{a}_{i}\|^2\right)\left(\sum_{i=1}^{m}\|\textbf{b}_{i}\|^2\right)\notag\\
  \label{55}&=\|\textbf{\textup{A}}\|_{\textup{F}}^2\|\textbf{\textup{B}}\|_{\textup{F}}^2.
  \end{align}
  Then combining \eqref{54} and \eqref{55}, we finish the proof.
\end{proof}
\begin{lemma}\label{lem7}
  Consider a matrix $\textbf{\textup{C}}\in \mathbb{R}^{m\times m}$ which satisfies condition 6 of Assumption 1. Then we have 
  \begin{align}
  \|\textbf{\textup{C}}^j-\textbf{\textup{J}}\|_{\textup{op}}=\zeta^j,
  \end{align}
  where $\zeta=\max\{|\lambda_2(\textbf{\textup{C}})|,|\lambda_m|(\textbf{\textup{C}})\}$.
\end{lemma}
\newcommand{\bm}[1]{\mbox{\boldmath{$#1$}}}
\begin{proof}
  Because $\textbf{C}$ is a real symmetric matrix, it can be decomposed as $\textbf{C}=\textbf{Q}\bm{\Lambda}\textbf{Q}^\top$, where $\textbf{Q}$ is an orthogonal matrix and $\textbf{\bm{\Lambda}}=\textup{diag}\{\lambda_1(\textbf{C}),\lambda_2(\textbf{C}),\dots,\lambda_m(\textbf{C})\}$. Similarly, matrix $\textbf{J}$ can be decomposed as $\textbf{J}=\textbf{Q}\bm{\Lambda}_0\textbf{Q}^\top$ where $\bm{\Lambda}_0=\textup{diag}\{1,0,\dots,0\}$. Then we have 
  \begin{align}\label{57}
  \textbf{C}^j-\textbf{J}=(\textbf{Q}\bm{\Lambda}\textbf{Q}^\top)^j-\textbf{J}=\textbf{Q}(\bm{\Lambda}^j-\bm{\Lambda}_0)\textbf{Q}^\top.
  \end{align} 
  According to the definition of the matrix operator norm, we further obtain
  \begin{align*}
  \|\textbf{C}^j-\textbf{J}\|_{\textup{op}}&=\sqrt{\lambda_{\max}((\textbf{C}^j-\textbf{J})^\top(\textbf{C}^j-\textbf{J}))}\\
  &=\sqrt{\lambda_{\max}(\textbf{C}^{2j}-\textbf{J})},
  \end{align*}
  where the last equality is due to
  \begin{align*}
  (\textbf{C}^j-\textbf{J})^\top(\textbf{C}^j-\textbf{J})&=(\textbf{C}^j-\textbf{J})(\textbf{C}^j-\textbf{J})\\
  &=\textbf{C}^{2j}+\textbf{J}^2-\textbf{C}^j\textbf{J}-\textbf{J}\textbf{C}^j\\
  &=\textbf{C}^{2j}+\textbf{J}-2\textbf{J}=\textbf{C}^{2j}-\textbf{J}.
  \end{align*}
  Because $\textbf{C}^{2j}-\textbf{J}=\textbf{Q}(\bm{\Lambda}^{2j}-\bm{\Lambda}_0)\textbf{Q}^\top$ from \eqref{57}, the maximum eigenvalue is $\max\{0,\lambda_2(\textbf{C})^{2j},\dots,\lambda_m(\textbf{C})^{2j}\}=\zeta^{2j}$ where $\zeta$ is the second largest eigenvalue of $\textbf{C}$. Therefore, we have 
  $$\|\textbf{C}^j-\textbf{J}\|_{\textup{op}}=\sqrt{\lambda_{\max}(\textbf{C}^{2j}-\textbf{J})}=\zeta^j.$$
\end{proof}
\textbf{Proof of Proposition 1:} After introducing the three lemmas, we begin to prove Proposition 1.
We recall the result of Lemma \ref{lem1} as follows:
\begin{align}
\mathbb{E}\left[\frac{1}{T}\sum_{t=1}^{T}\|\nabla F(\textbf{\textup{u}}_t)\|^2\right]\leq \frac{2[F(\textbf{\textup{w}}_1)-F_{\textup{inf}}]}{\eta T}+\frac{\eta L \sigma^2}{N}+\notag\\
\frac{L^2}{TN}\sum_{t=1}^{T}\mathbb{E}\|\textbf{\textup{X}}_t(\textbf{\textup{I}}-\textbf{\textup{J}})\|^2_{\textup{F}}.
\end{align}
Our target is to find a upper bound of the local drift term $\frac{L^2}{TN}\sum_{t=1}^{T}\mathbb{E}\|\textbf{\textup{X}}_t(\textbf{\textup{I}}-\textbf{\textup{J}})\|^2_{\textup{F}}$. First of all, we consider the term $\mathbb{E}\|\textbf{\textup{X}}_t(\textbf{\textup{I}}-\textbf{\textup{J}})\|^2_{\textup{F}}$.

According to the learning strategy of DFL in (5), we have 
\begin{align*}
\textbf{\textup{X}}_t(\textbf{\textup{I}}-\textbf{\textup{J}})&=(\textbf{X}_{t-1}-\eta \textbf{G}^{\prime}_{t-1})\textbf{C}_{t-1}(\textbf{I}-\textbf{J})\\
&=\textbf{X}_{t-1}(\textbf{I}-\textbf{J})\textbf{C}_{t-1}-\eta\textbf{G}^{\prime}_{t-1}(\textbf{C}_{t-1}-\textbf{J}),
\end{align*}
where the last equality is due to $\textbf{C}_{t-1}\textbf{J}=\textbf{J}\textbf{C}_{t-1}=\textbf{J}$ and hence $(\textbf{I}-\textbf{J})\textbf{C}_{t-1}=\textbf{C}_{t-1}(\textbf{I}-\textbf{J})$. We further expand the expression of $\textbf{X}_{t-1}$ as follows:
\begin{align*}
\textbf{X}_{t}(\textbf{I}-\textbf{J})=&[\textbf{X}_{t-2}(\textbf{I}-\textbf{J})\textbf{C}_{t-2}-\eta\textbf{G}^{\prime}_{t-2}(\textbf{C}_{t-2}-\textbf{J})]\textbf{C}_{t-1}\\
&-\eta\textbf{G}^{\prime}_{t-1}(\textbf{C}_{t-1}-\textbf{J})\\
=&\textbf{X}_{t-2}(\textbf{I}-\textbf{J})\textbf{C}_{t-2}\textbf{C}_{t-1}-\eta\textbf{G}^{\prime}_{t-2}(\textbf{C}_{t-2}\textbf{C}_{t-1}-\textbf{J})\\&-\eta\textbf{G}^{\prime}_{t-1}(\textbf{C}_{t-1}-\textbf{J}).
\end{align*}
Repeating the same expansion for $\textbf{X}_{t-2},\textbf{X}_{t-3},\dots,\textbf{X}_{2}$, finally we obtain
\begin{align}
\textbf{X}_{t}(\textbf{I}-\textbf{J})=\textbf{X}_{1}(\textbf{I}-\textbf{J})\bm{\Phi}_{1,t-1}-\eta\sum_{s=1}^{t-1}\textbf{G}^{\prime}_s(\bm{\Phi}_{s,t-1}-\textbf{J}),
\end{align}
where $\bm{\Phi}_{s,t-1}=\prod_{l=s}^{t-1}\textbf{C}_l$. As parameters of all local models are initialized at the same point $\textbf{X}_1(\textbf{I}-\textbf{J})=0$, the squared norm of the local drift term can be rewritten as 
\begin{align}
\mathbb{E}\|\textbf{X}_{t}(\textbf{I}-\textbf{J})\|^2_{\textup{F}}=\eta^2\mathbb{E}\left\|\sum_{s=1}^{t-1}\textbf{G}^{\prime}_s(\bm{\Phi}_{s,t-1}-\textbf{J})\right\|^2_{\textup{F}}.
\end{align}
Then we list the expression of $\bm{\Phi}_{s,t-1}$ for any $s\leq t-1$. Without loss of generality, we assume $t=j\tau+i$ where $i\leq\tau_1$ according to Algorithm 1. Note that $j$ denotes the index of iteration round and $i$ denote the index of local updates. The expression of matrix $\bm{\Phi}_{s,t-1}$ is presented as follows:
\begin{align}
\bm{\Phi}_{s,t-1}=
\begin{cases}
\textbf{I},&j\tau< s\leq j\tau+i \\
\textbf{C}^{j\tau-s},&(j-1)\tau+\tau_1<s\leq j\tau\\
\textbf{C}^{\tau_2},&(j-1)\tau<s\leq(j-1)\tau+\tau_1\\
\textbf{C}^{\tau_2}\cdot\textbf{C}^{(j-1)\tau-s},&(j-2)\tau+\tau_1<s\leq(j-1)\tau\\
\textbf{C}^{2\tau_2},&(j-2)\tau<s\leq(j-2)\tau+\tau_1\\
\vdots\\
\textbf{C}^{(j-1)\tau_2}\cdot\textbf{C}^{\tau-s},&\tau_1<s\leq\tau\\
\textbf{C}^{j\tau_2},&0<s\leq\tau_1.
\end{cases}
\end{align}
For the simplification of writing, we define the sum of stochastic gradient within one round as $\textbf{Y}_r=\sum_{s=r\tau+1}^{r\tau+\tau_1}\textbf{G}^{\prime}_s=\sum_{s=r\tau+1}^{r\tau+\tau_1}\textbf{G}_s$ for $0\leq r<j$ and $\textbf{Y}_j=\sum_{s=j\tau+1}^{j\tau+i-1}\textbf{G}_s$. Similarly, we define accumulated global gradient $\textbf{Q}_r=\sum_{s=r\tau+1}^{r\tau+\tau_1}\nabla F(\textbf{X}_s)$ for $0\leq r<j$ and $\textbf{Q}_j=\sum_{s=j\tau+1}^{j\tau+i-1}\nabla F(\textbf{X}_s)$. Thus, we have 
\begin{align*}
&\sum_{s=1}^{\tau}\textbf{G}^{\prime}_s(\bm{\Phi}_{s,t-1}-\textbf{J})=\sum_{s=1}^{\tau_1}\textbf{G}_s(\textbf{C}^{j\tau_2}-\textbf{J})=\textbf{Y}_0(\textbf{C}^{j\tau_2}-\textbf{J})\\
&\sum_{s=\tau+1}^{2\tau}\textbf{G}^{\prime}_s(\bm{\Phi}_{s,t-1}-\textbf{J})=\textbf{Y}_1(\textbf{C}^{(j-1)\tau_2}-\textbf{J})\\
&\vdots\\
&\sum_{s=j\tau+1}^{j\tau+i-1}\textbf{G}^{\prime}_s(\bm{\Phi}_{s,t-1}-\textbf{J})=\textbf{Y}_j(\textbf{I}-\textbf{J}).
\end{align*}
Hence, we sum all these terms and obtain
\begin{align}
\sum_{s=1}^{t-1}\textbf{G}^{\prime}_s(\bm{\Phi}_{s,t-1}-\textbf{J})=\sum_{r=0}^{j}\textbf{Y}_r(\textbf{C}^{(j-r)\tau_2}-\textbf{J}).
\end{align}
We further decompose the local drift term into two parts as follows:
\begin{align}
&\mathbb{E}\|\textbf{X}_t(\textbf{I}-\textbf{J})\|^2_{\textup{F}}\notag\\
=&\eta^2\mathbb{E}\left\|\sum_{r=0}^{j}\textbf{Y}_r(\textbf{C}^{(j-r)\tau_2}-\textbf{J})\right\|^2_{\textup{F}}\notag\\
=&\eta^2\mathbb{E}\left\|\sum_{r=0}^{j}(\textbf{Y}_r\!-\!\textbf{Q}_r)(\textbf{C}^{(j-r)\tau_2}\!-\!\textbf{J})\!+\!\sum_{r=0}^{j}\textbf{Q}_r(\textbf{C}^{(j-r)\tau_2}\!-\!\textbf{J})\right\|^2_{\textup{F}}\notag\\
\leq&\underbrace{2\eta^2\mathbb{E}\left\|\sum_{r=0}^{j}(\textbf{Y}_r\!-\!\textbf{Q}_r)(\textbf{C}^{(j-r)\tau_2}\!-\!\textbf{J})\right\|^2_{\textup{F}}}_{A_1}\notag\\
\label{63}&+\underbrace{2\eta^2\mathbb{E}\left\|\sum_{r=0}^{j}\textbf{Q}_r(\textbf{C}^{(j-r)\tau_2}\!-\!\textbf{J})\right\|^2_{\textup{F}}}_{A_2},
\end{align}
where \eqref{63} is due to $\|\textbf{a}+\textbf{b}\|^2\leq2\|\textbf{a}\|^2+2\|\textbf{b}\|^2$. Based on \eqref{63}, we aim to separately derive the upper bounds for $A_1$ and $A_2$ before completing the proof of Proposition 1. 

\textbf{The Bound of $A_1$:} For the term $A_1$ in \eqref{63}, we have 
\begin{align}
&A_1\notag\\
=&2\eta^2\mathbb{E}\left\|\sum_{r=0}^{j}(\textbf{Y}_r\!-\!\textbf{Q}_r)(\textbf{C}^{(j-r)\tau_2}\!-\!\textbf{J})\right\|^2_{\textup{F}}\notag\\
=&2\eta^2\sum_{r=0}^{j}\mathbb{E}\|(\textbf{Y}_r\!-\!\textbf{Q}_r)(\textbf{C}^{(j-r)\tau_2}\!-\!\textbf{J})\|^2_{\textup{F}}\notag\\
\!+\!&2\eta^2\!\!\sum_{n=0}^{j}\!\sum_{l=0\!,\!l\neq n}^{j}\!\!\!\!\!\mathbb{E}\![\textup{Tr}\!(\!(\!\textbf{C}^{(j-n)\tau_2}\!-\!\textbf{J})(\textbf{Y}_n\!\!\!-\!\!\!\textbf{Q}_n)\!\!^\top\!\!(\textbf{Y}_l\!\!-\!\!\textbf{Q}_l)(\textbf{C}^{(j\!-\!l)\tau_2}\!\!-\!\!\textbf{J}\!)\!)\!]\notag\\
\label{64}=&2\eta^2\sum_{r=0}^{j}\mathbb{E}\|(\textbf{Y}_r\!-\!\textbf{Q}_r)(\textbf{C}^{(j-r)\tau_2}\!-\!\textbf{J})\|^2_{\textup{F}}\\
\label{65}\leq&2\eta^2\sum_{r=0}^{j}\mathbb{E}\|\textbf{Y}_r\!-\!\textbf{Q}_r\|^2_{\textup{F}}\|\textbf{C}^{(j-r)\tau_2}\!-\!\textbf{J}\|^2_{\textup{op}}\\
\label{66}=&2\eta^2\sum_{r=0}^{j}\mathbb{E}\|\textbf{Y}_r\!-\!\textbf{Q}_r\|^2_{\textup{F}}\zeta^{2\tau_2(j-r)}\\
\label{67}=&2\eta^2\sum_{r=0}^{j-1}\mathbb{E}\|\textbf{Y}_r\!-\!\textbf{Q}_r\|^2_{\textup{F}}\zeta^{2\tau_2(j-r)}+2\eta^2\mathbb{E}\|\textbf{Y}_j\!-\!\textbf{Q}_j\|^2_{\textup{F}},
\end{align}
where \eqref{65} comes from Lemma \ref{lem6} and \eqref{66} is due to Lemma \ref{lem7}. 
Consider the inner product of the cross terms $g(\textbf{w}_s^{(i)})-\nabla F(\textbf{w}_s^{(i)})$ and $g(\textbf{w}_l^{(j)})-\nabla F(\textbf{w}_l^{(j)})$ for any $s\neq l$ and any $i,j$, we can obtain
\begin{align}
&\mathbb{E}\left[\left\langle g(\textbf{w}_s^{(i)})-\nabla F(\textbf{w}_s^{(i)}),g(\textbf{w}_l^{(j)})-\nabla F(\textbf{w}_l^{(j)})\right\rangle\right]\notag\\
=&\mathbb{E}_{\textbf{w}_s^{(i)}\!,\!\xi_s^{(i)}\!,\!\textbf{w}_l^{(j)}}\mathbb{E}_{\xi_l^{(j)}\!|\!\textbf{w}_s^{(i)}\!,\!\xi_s^{(i)}\!,\!\textbf{w}_l^{(j)}}\!\!\left[\!\!\left\langle\!g(\textbf{w}_s^{(i)})\!\!-\!\!\nabla F(\textbf{w}_s^{(i)})\!,\!g(\textbf{w}_l^{(j)})\!\!-\!\!\nabla F(\textbf{w}_l^{(j)})\!\right\rangle\!\!\right]\notag\\
=&\mathbb{E}\left[\left\langle g(\textbf{w}_s^{(i)})\!\!-\!\!\nabla F(\textbf{w}_s^{(i)}),\mathbb{E}_{\xi_l^{(j)}\!|\!\textbf{w}_s^{(i)}\!,\!\xi_s^{(i)}\!,\!\textbf{w}_l^{(j)}}\left[g(\textbf{w}_l^{(j)})\!\!-\!\!\nabla F(\textbf{w}_l^{(j)})\right]\right\rangle\right]\notag\\
\label{68}=&\mathbb{E}\left[\left\langle g(\textbf{w}_s^{(i)})\!\!-\!\!\nabla F(\textbf{w}_s^{(i)}),0\right\rangle\right]\\
\label{69}=&0,
\end{align}
where \eqref{68} comes from condition 4 of Assumption 1. Therefore, according to \eqref{69}, we can easily obtain  $\mathbb{E}[(\textbf{Y}_n-\textbf{Q}_n)^\top(\textbf{Y}_l-\textbf{Q}_l)]=0$. Thus \eqref{64} holds. 

Following the result of \eqref{67}, for and $0\leq r<j$, we have 
\begin{align}
&\mathbb{E}\left[\|\textbf{Y}_r-\textbf{Q}_r\|^2_{\textup{F}}\right]\notag\\
=&\mathbb{E}\left[\left\|\sum_{s=r\tau+1}^{r\tau+\tau_1}[\textbf{G}_s-\nabla F(\textbf{X}_s)]\right\|^2_{\textup{F}}\right]\\
=&\sum_{i=1}^{N}\mathbb{E}\left[\left\|\sum_{s=r\tau+1}^{r\tau+\tau_1}[g(\textbf{w}_s^{(i)})-\nabla F(\textbf{w}_s^{(i)})]\right\|^2\right]\\
=&\sum_{i=1}^{N}\mathbb{E}\left[\sum_{s=r\tau+1}^{r\tau+\tau_1}\left\|g(\textbf{w}_s^{(i)})-\nabla F(\textbf{w}_s^{(i)})\right\|^2+\right.\notag
\\
\phantom{=\;\;}
&\left.\sum_{s\neq l}\left\langle g(\textbf{w}_s^{(i)})-\nabla F(\textbf{w}_s^{(i)}),g(\textbf{w}_l^{(i)})-\nabla F(\textbf{w}_l^{(i)})\right\rangle\right]\\
\label{73}=&\sum_{i=1}^{N}\mathbb{E}\left[\sum_{s=r\tau+1}^{r\tau+\tau_1}\left\|g(\textbf{w}_s^{(i)})-\nabla F(\textbf{w}_s^{(i)})\right\|^2\right]\\
=&\mathbb{E}\left[\sum_{s=r\tau+1}^{r\tau+\tau_1}\sum_{i=1}^{N}\left\|g(\textbf{w}_s^{(i)})-\nabla F(\textbf{w}_s^{(i)})\right\|^2\right]\\
\label{75}\leq&\beta\sum_{s=r\tau+1}^{r\tau+\tau_1}\sum_{i=1}^{N}\mathbb{E}\left[\left\|\nabla F(\textbf{w}_s^{(i)}) \right\|^2\right]+\tau_1 N\sigma^2\\
\label{76}=&\beta\sum_{s=r\tau+1}^{r\tau+\tau_1}\|\nabla F(\textbf{X}_s)\|^2_{\textup{F}}+\tau_1 N\sigma^2,
\end{align}
where \eqref{73} comes from \eqref{69} and \eqref{75} is according to condition 5 of Assumption 1. Similarly,  we can obtain 
\begin{align}
\label{77}\mathbb{E}\left[\|\textbf{Y}_j-\textbf{Q}_j\|^2_{\textup{F}}\right]\leq \beta\sum_{s=j\tau+1}^{j\tau+i-1}\|\nabla F(\textbf{X}_s)\|^2_{\textup{F}}+(i-1)N\sigma^2.
\end{align}
Substituting \eqref{76} and \eqref{77} into \eqref{67}, we obtain 
\begin{align}
A_1&\leq 2\eta^2 \sum_{r=0}^{j-1}\left[\zeta^{2\tau_2(j-r)}\left(\beta\sum_{s=r\tau+1}^{r\tau+\tau_1}\|\nabla F(\textbf{X}_s)\|^2_{\textup{F}}+\tau_1 N\sigma^2\right)\right]\notag\\
&+2\eta^2\beta \sum_{s=j\tau+1}^{j\tau+i-1}\|\nabla F(\textbf{X}_s)\|^2_{\textup{F}}+2\eta^2(i-1)N\sigma^2\\
\label{79}&\leq 2\eta^2 N\sigma^2\left[\frac{\tau_1\zeta^{2\tau_2}}{1-\zeta^{2\tau_2}}+i-1\right]+\notag\\
&2\eta^2\beta\sum_{r=0}^{j-1}\left[\zeta^{2\tau_2(j-r)}\left(\sum_{s=r\tau+1}^{r\tau+\tau_1}\|\nabla F(\textbf{X}_s)\|^2_{\textup{F}}\right)\right]+\notag\\
&2\eta^2 \beta\sum_{s=j\tau+1}^{j\tau+i-1}\|\nabla F(\textbf{X}_s)\|^2_{\textup{F}}\\
&\leq 2\eta^2 N\sigma^2\left[\frac{\tau_1\zeta^{2\tau_2}}{1-\zeta^{2\tau_2}}+\tau_1-1\right]+\notag\\
&2\eta^2\beta\sum_{r=0}^{j-1}\left[\zeta^{2\tau_2(j-r)}\left(\sum_{s=r\tau+1}^{r\tau+\tau_1}\|\nabla F(\textbf{X}_s)\|^2_{\textup{F}}\right)\right]+\notag\\
\label{80}&2\eta^2 \beta\sum_{s=j\tau+1}^{j\tau+\tau_1}\|\nabla F(\textbf{X}_s)\|^2_{\textup{F}},
\end{align}
where \eqref{79} is because 
\begin{align*}
\sum_{r=0}^{j-1}\zeta^{2\tau_2(j-r)}=\frac{\zeta^{2\tau_2}(1-(\zeta^{2\tau_2})^j)}{1-\zeta^{2\tau_2}}\leq\frac{\zeta^{2\tau_2}}{1-\zeta^{2\tau_2}},
\end{align*} 
and \eqref{80} comes from $i\leq\tau_1$. 
Next, we sum over all iteration steps in one round from $i=1$ to $i=\tau$ and obtain
\begin{align}
\sum_{i=1}^{\tau}A_1&\leq 2\eta^2 N\sigma^2\tau\left[\frac{\tau_1\zeta^{2\tau_2}}{1-\zeta^{2\tau_2}}+\tau_1-1\right]\notag\\
&+2\eta^2\beta\tau\sum_{r=0}^{j-1}\left[\zeta^{2\tau_2(j-r)}\left(\sum_{s=r\tau+1}^{r\tau+\tau_1}\|\nabla F(\textbf{X}_s)\|^2_{\textup{F}}\right)\right]\notag\\
&+2\eta^2\beta\tau\sum_{s=j\tau+1}^{j\tau+\tau_1}\|\nabla F(\textbf{X}_s)\|^2_{\textup{F}}\\
&=2\eta^2 N\sigma^2\tau\left[\frac{\tau_1\zeta^{2\tau_2}}{1-\zeta^{2\tau_2}}+\tau_1-1\right]\notag\\
&+2\eta^2\beta\tau\sum_{r=0}^{j}\left[\zeta^{2\tau_2(j-r)}\left(\sum_{s=r\tau+1}^{r\tau+\tau_1}\|\nabla F(\textbf{X}_s)\|^2_{\textup{F}}\right)\right].
\end{align}
Then we sum over all rounds from $j=0$ to $j=\frac{T}{\tau}-1$, where $T$ is the total iteration steps and we get
\begin{align}
&\sum_{j=0}^{\frac{T}{\tau}-1}\sum_{i=1}^{\tau}A_1\notag\\
&\leq 2\eta^2 N\sigma^2 T\left[\frac{\tau_1\zeta^{2\tau_2}}{1-\zeta^{2\tau_2}}+\tau_1-1\right]+\notag\\
&2\eta^2\beta\tau\sum_{j=0}^{\frac{T}{\tau}-1}\sum_{r=0}^{j}\left[\zeta^{2\tau_2(j-r)}\left(\sum_{s=r\tau+1}^{r\tau+\tau_1}\|\nabla F(\textbf{X}_s)\|^2_{\textup{F}}\right)\right]\notag\\
&=2\eta^2 N\sigma^2 T\left[\frac{\tau_1\zeta^{2\tau_2}}{1-\zeta^{2\tau_2}}+\tau_1-1\right]+\notag\\
&2\eta^2\beta\tau\sum_{r=0}^{\frac{T}{\tau}-1}\left[\left(\sum_{s=r\tau+1}^{r\tau+\tau_1}\|\nabla F(\textbf{X}_s)\|^2_{\textup{F}}\right)\left(\sum_{j=r}^{\frac{T}{\tau}-1}\zeta^{2\tau_2(j-r)}\right)\right]\notag\\
&\leq 2\eta^2 N\sigma^2 T\left[\frac{\tau_1\zeta^{2\tau_2}}{1-\zeta^{2\tau_2}}+\tau_1-1\right]+\notag\\
&\frac{2\eta^2\beta\tau}{1-\zeta^{2\tau_2}}\sum_{r=0}^{\frac{T}{\tau}-1}\sum_{s=r\tau+1}^{r\tau+\tau_1}\|\nabla F(\textbf{X}_s)\|^2_{\textup{F}}\notag\\
&=2\eta^2 N\sigma^2 T\left(\frac{\tau_1}{1-\zeta^{2\tau_2}}-1 \right)+\notag\\
\label{83}&\frac{2\eta^2\beta\tau}{1-\zeta^{2\tau_2}}\sum_{r=0}^{\frac{T}{\tau}-1}\sum_{s=r\tau+1}^{r\tau+\tau_1}\|\nabla F(\textbf{X}_s)\|^2_{\textup{F}}.
\end{align}
Thus we finish the first part for deriving the bound of $A_1$. 

\textbf{The bound of $A_2$:}
From the second term in \eqref{63}, we have 
\begin{align}
A_2&=2\eta^2 \mathbb{E}\left\|\sum_{r=0}^{j}\textbf{Q}_r(\textbf{C}^{(j-r)\tau_2}-\textbf{J})\right\|^2_{\textup{F}}\notag\\
&=2\eta^2 \sum_{r=0}^{j}\mathbb{E}\|\textbf{Q}_r(\textbf{C}^{(j-r)\tau_2}-\textbf{J})\|^2_{\textup{F}}+\notag\\
\label{84}&2\eta^2 \sum_{n=0}^{j}\sum_{l=0,l\neq n}^{j}\mathbb{E}[\textup{Tr}((\textbf{C}^{\tau_2(j-n)}-\textbf{J})\textbf{Q}_n^\top \textbf{Q}_l(\textbf{C}^{\tau_2(j-l)}-\textbf{J}))].
\end{align}
Considering the second term in equality \eqref{84}, we can find the bound of the trace as follows:
\begin{align}
&|\textup{Tr}((\textbf{C}^{\tau_2(j-n)}-\textbf{J})\textbf{Q}_n^\top \textbf{Q}_l(\textbf{C}^{\tau_2(j-l)}-\textbf{J}))|\notag\\
\label{85}\leq&\|(\textbf{C}^{\tau_2(j-n)}-\textbf{J})\textbf{Q}_n^\top\|_{\textup{F}}\|\textbf{Q}_l(\textbf{C}^{\tau_2(j-l)}-\textbf{J})\|_{\textup{F}}\\
\label{86}\leq&\|\textbf{C}^{\tau_2(j-n)}-\textbf{J}\|_{\textup{op}}\|\textbf{Q}_n\|_{\textup{F}}\|\textbf{Q}_l\|_{\textup{F}}\|\textbf{C}^{\tau_2(j-l)}-\textbf{J}\|_{\textup{op}}\\
\label{87}\leq&\frac{1}{2}\zeta^{\tau_2(2j-n-l)}\left[\|\textbf{Q}_n\|_{\textup{F}}^2+\|\textbf{Q}_l\|_{\textup{F}}^2\right],
\end{align}
where \eqref{85} is because of Lemma \ref{lem6}, \eqref{86} comes from Lemma \ref{lem5} and \eqref{87} is due to Lemma \ref{lem7} and $2ab\leq a^2+b^2$. Then we have 
\begin{align}
&A_2\leq 2\eta^2 \sum_{r=0}^{j}\mathbb{E}\|\textbf{Q}_r\|_{\textup{F}}^2\|\textbf{C}^{\tau_2(j-r)}-\textbf{J}\|^2_{\textup{op}}+\notag\\
&\eta^2 \sum_{n=0}^{j}\sum_{l=0,l\neq n}^{j}\zeta^{\tau_2(2j-n-l)}\mathbb{E}\left[\|\textbf{Q}_n\|_{\textup{F}}^2+\|\textbf{Q}_l\|_{\textup{F}}^2\right]\\
\label{89}=&\!2 \eta^2\!\! \sum_{r=0}^{j}\!\!\zeta^{2\tau_2 (j-r)}\mathbb{E}\|\textbf{Q}_r\|^2_{\textup{F}}\!\!+\!\!2\eta^2\sum_{n=0}^{j}\sum_{l=0,l\neq n}^{j}\!\!\!\!\!\!\zeta^{\tau_2(2j-n-l)}\mathbb{E}\|\textbf{Q}_n\|^2_{\textup{F}}\\
=&\!2 \eta^2\!\! \sum_{r=0}^{j}\!\!\zeta^{2\tau_2 (j-r)}\mathbb{E}\|\textbf{Q}_r\|^2_{\textup{F}}\!+\!2\eta^2\!\!\sum_{n=0}^{j}\zeta^{\tau_2(j-n)}\mathbb{E}\|\textbf{Q}_n\|^2_{\textup{F}}\!\!\!\!\!\sum_{l=0,l\neq n}^{j}\!\!\!\!\!\!\zeta^{\tau_2(j-l)}\\
=&2\eta^2\!\!\!\left[\sum_{r=0}^{j-1}\zeta^{2\tau_2 (j-r)}\mathbb{E}\|\textbf{Q}_r\|^2_{\textup{F}}\!\!+\!\!\sum_{n=0}^{j-1}\zeta^{\tau_2(j-n)}\mathbb{E}\|\textbf{Q}_n\|^2_{\textup{F}}\!\!\!\!\!\sum_{l=0,l\neq n}^{j}\!\!\!\!\!\!\zeta^{\tau_2(j-l)}\right.\notag
\\
\phantom{=\;\;}
&\left.+\mathbb{E}\|\textbf{Q}_j\|^2_{\textup{F}}+\mathbb{E}\|\textbf{Q}_j\|^2_{\textup{F}}\sum_{l=0}^{j-1}\zeta^{\tau_2(j-l)}\right]\\
\label{92}\leq& 2\eta^2\left[\sum_{r=0}^{j-1}\zeta^{2\tau_2 (j-r)}\mathbb{E}\|\textbf{Q}_r\|^2_{\textup{F}}+\sum_{n=0}^{j-1}\frac{\zeta^{\tau_2(j-n)}}{1-\zeta^{\tau_2}}\mathbb{E}\|\textbf{Q}_n\|^2_{\textup{F}}\right.\notag
\\
\phantom{=\;\;}
&\left.+\mathbb{E}\|\textbf{Q}_j\|^2_{\textup{F}}+\mathbb{E}\|\textbf{Q}_j\|^2_{\textup{F}}\frac{\zeta^{\tau_2}}{1-\zeta^{\tau_2}}\right],
\end{align}
where \eqref{89} is because indices $n$ and $l$ are symmetric and \eqref{92} is due to 
\begin{align*}
&\sum_{l=0,l\neq n}^{j}\zeta^{\tau_2(j-l)}\leq \sum_{l=0}^{j}\zeta^{\tau_2 l}=\frac{1-\zeta^{\tau_2(j+1)}}{1-\zeta^{\tau_2}}<\frac{1}{1-\zeta^{\tau_2}}\\
&\sum_{l=0}^{j-1}\zeta^{\tau_2(j-l)}=\sum_{l=1}^{j}\zeta^{\tau_2 l}=\frac{\zeta^{\tau_2}(1-\zeta^{\tau_2 j})}{1-\zeta^{\tau_2}}<\frac{\zeta^{\tau_2}}{1-\zeta^{\tau_2}}.
\end{align*}
By rearranging \eqref{92}, we obtain 
\begin{align}
A_2&\leq 2\eta^2\sum_{r=0}^{j-1}\left[\left(\zeta^{2\tau_2 (j-r)}+\frac{\zeta^{\tau_2(j-r)}}{1-\zeta^{\tau_2}}\right) \mathbb{E}\|\textbf{Q}_r\|^2_{\textup{F}}\right]+\notag\\
&\frac{2\eta^2}{1-\zeta^{\tau_2}}\mathbb{E}\|\textbf{Q}_j\|^2_{\textup{F}}\\
&=2\eta^2\!\!\sum_{r=0}^{j-1}\left[\left(\zeta^{2\tau_2 (j-r)}\!\!+\!\!\frac{\zeta^{\tau_2(j-r)}}{1-\zeta^{\tau_2}}\right)\!\!\mathbb{E}\left\| \sum_{s=1}^{\tau_1}\nabla F(\textbf{X}_{r\tau+s})\right\|^2_{\textup{F}}\right]\!\!\notag\\
&+\frac{2\eta^2}{1-\zeta^{\tau_2}}\mathbb{E}\left\|\sum_{s=1}^{i-1}\nabla F(\textbf{X}_{j\tau+s})\right\|^2_{\textup{F}}\\
&\label{95}\leq 2\eta^2\tau_1\sum_{r=0}^{j-1}\left[\left(\zeta^{2\tau_2 (j-r)}\!\!+\!\!\frac{\zeta^{\tau_2(j-r)}}{1-\zeta^{\tau_2}}\right)\!\!\sum_{s=1}^{\tau_1}\mathbb{E}\|\nabla F(\textbf{X}_{r\tau+s})\|^2_{\textup{F}}\right]\notag\\
&+\frac{2\eta^2(i-1)}{1-\zeta^{\tau_2}}\sum_{s=1}^{i-1}\mathbb{E}\|\nabla F(\textbf{X}_{j\tau+s})\|^2_{\textup{F}},
\end{align}
where \eqref{95} is because of the convexity of Frobenius norm and Jensen's inequality. Then by summing over all iteration steps in $j$-th round from $i=1$ to $i=\tau$, we have 
\begin{align}
&\sum_{i=1}^{\tau}A_2\notag\\
&\leq 2\eta^2\tau_1\tau\sum_{r=0}^{j-1}\left[\left(\zeta^{2\tau_2 (j-r)}\!\!+\!\!\frac{\zeta^{\tau_2(j-r)}}{1-\zeta^{\tau_2}}\right)\sum_{s=1}^{\tau_1}\mathbb{E}\|\nabla F(\textbf{X}_{r\tau+s})\|^2_{\textup{F}}\right]\notag\\
&+\frac{\eta^2\tau(\tau-1)}{1-\zeta^{\tau_2}}\sum_{s=1}^{\tau_1-1}\mathbb{E}\|\nabla F(\textbf{X}_{j\tau+s})\|^2_{\textup{F}}.
\end{align}
Then, summing over all rounds from $j=0$ to $j=\frac{T}{\tau}-1$, where $T$ is the total iteration steps, we can obtain
\begin{align}
&\sum_{j=0}^{\frac{T}{\tau}-1}\sum_{i=1}^{\tau}A_2\notag\\
\leq&2\eta^2\tau_1\tau \sum_{j=0}^{\frac{T}{\tau}-1}\sum_{r=0}^{j-1}\left[\left(\zeta^{2\tau_2 (j-r)}\!\!+\!\!\frac{\zeta^{\tau_2(j-r)}}{1-\zeta^{\tau_2}}\right)\sum_{s=1}^{\tau_1}\mathbb{E}\|\nabla F(\textbf{X}_{r\tau+s})\|^2_{\textup{F}}\right]\notag\\
&+\frac{\eta^2\tau(\tau-1)}{1-\zeta^{\tau_2}}\sum_{j=0}^{\frac{T}{\tau}-1}\sum_{s=1}^{\tau_1-1}\mathbb{E}\|\nabla F(\textbf{X}_{j\tau+s})\|^2_{\textup{F}}\\
=&2\eta^2\tau_1\tau\!\!\sum_{r=0}^{\frac{T}{\tau}-2}\!\!\left[\!\!\left(\sum_{s=1}^{\tau_1}\mathbb{E}\|\nabla F(\textbf{X}_{r\tau+s})\|^2_{\textup{F}}\!\!\right)\!\!\!\!\sum_{j=1}^{\frac{T}{\tau}-1-r}\!\!\!\!\left(\zeta^{2\tau_2 j}\!\!+\!\!\frac{\zeta^{\tau_2j}}{1-\zeta^{\tau_2}}\right)\!\!\right]\notag\\
&+\frac{\eta^2\tau(\tau-1)}{1-\zeta^{\tau_2}}\sum_{j=0}^{\frac{T}{\tau}-1}\sum_{s=1}^{\tau_1-1}\mathbb{E}\|\nabla F(\textbf{X}_{j\tau+s})\|^2_{\textup{F}}\\
\leq&2\eta^2\tau_1\tau\left(\frac{\zeta^{2\tau_2}}{1-\zeta^{2\tau_2}}+\frac{\zeta^{\tau_2}}{(1-\zeta^{\tau_2})^2}\right)\sum_{r=0}^{\frac{T}{\tau}-2}
\sum_{s=1}^{\tau_1}\mathbb{E}\|\nabla F(\textbf{X}_{r\tau+s})\|^2_{\textup{F}}\notag\\
&+\frac{\eta^2\tau(\tau-1)}{1-\zeta^{\tau_2}}\sum_{j=0}^{\frac{T}{\tau}-1}\sum_{s=1}^{\tau_1-1}\mathbb{E}\|\nabla F(\textbf{X}_{j\tau+s})\|^2_{\textup{F}}\\
\leq&\frac{\eta^2\tau}{1-\zeta^{\tau_2}}\!\!\left(\frac{2\tau_1\zeta^{2\tau_2}}{1+\zeta^{\tau_2}}\!\!+\!\!\frac{2\tau_1\zeta^{\tau_2}}{1-\zeta^{\tau_2}}\!\!+\!\tau\!-\!1\!
\right)\!\!\sum_{r=0}^{\frac{T}{\tau}-1}\!\sum_{s=1}^{\tau_1}\mathbb{E}\|\nabla F(\textbf{X}_{r\tau+s})\|^2_{\textup{F}}\\
\label{101}=&\frac{\eta^2\tau}{1-\zeta^{\tau_2}}\!\!\left(\frac{2\tau_1\zeta^{2\tau_2}}{1+\zeta^{\tau_2}}\!\!+\!\!\frac{2\tau_1\zeta^{\tau_2}}{1-\zeta^{\tau_2}}\!\!+\!\tau\!-\!1\!
\right)\!\!\sum_{r=0}^{\frac{T}{\tau}-1}\!\sum_{s=r\tau+1}^{r\tau+\tau_1}\!\!\mathbb{E}\|\nabla F(\textbf{X}_{s})\|^2_{\textup{F}}.
\end{align}
Therefore, we finish the second part for deriving the bound of $A_2$.

\textbf{The final proof of Proposition 1:} Based on inequalities \eqref{63}, \eqref{83} and \eqref{101}, local drift has an upper bound as follows:
\begin{align}
&\frac{L^2}{TN}\sum_{t=1}^{T}\mathbb{E}\|\textbf{\textup{X}}_t(\textbf{\textup{I}}-\textbf{\textup{J}})\|^2_{\textup{F}}\notag\\
\leq&\frac{L^2}{TN}\sum_{j=0}^{\frac{T}{\tau}-1}\sum_{i=1}^{\tau}(A_1+A_2)\\
\label{103}\leq& 2 L^2\eta^2\sigma^2\!\!\left(\frac{\tau_1}{1-\zeta^{2\tau_2}}\!\!-\!\!1\right)\!\!+\!\!\frac{L^2}{TN}\frac{2\eta^2\beta\tau}{1-\zeta^{2\tau_2}}\!\!\sum_{r=0}^{\frac{T}{\tau}-1}\!\!\sum_{s=r\tau+1}^{r\tau+\tau_1}\!\!\!\!\|\nabla F(\textbf{X}_s)\|^2_{\textup{F}}\notag\\
\!+\!&\frac{L^2}{TN}\!\frac{\eta^2\tau}{1-\zeta^{\tau_2}}\!\!\left(\frac{2\tau_1\zeta^{2\tau_2}}{1+\zeta^{\tau_2}}\!\!+\!\!\frac{2\tau_1\zeta^{\tau_2}}{1-\zeta^{\tau_2}}\!\!+\!\tau\!-\!1\!
\right)\!\!\sum_{r=0}^{\frac{T}{\tau}-1}\!\sum_{s=r\tau+1}^{r\tau+\tau_1}\!\!\mathbb{E}\|\!\nabla \!F\!(\textbf{X}_{s})\!\|^2_{\textup{F}}.
\end{align} 
Substituting \eqref{103} into \eqref{50}, we have
\begin{align}
&\mathbb{E}\left[\frac{1}{T}\sum_{t=1}^{T}\|\nabla F(\textbf{u}_t)\|^2\right]\leq\frac{2[F(\textbf{u}_1)-F_{\textup{inf}}]}{\eta T}+\frac{\eta L\sigma^2}{N}\notag\\
+&2 L^2\eta^2\sigma^2\!\!\left(\frac{\tau_1}{1-\zeta^{2\tau_2}}\!\!-\!\!1\right)\!\!+\!\!\frac{L^2}{TN}\frac{2\eta^2\beta\tau}{1-\zeta^{2\tau_2}}\!\!\sum_{r=0}^{\frac{T}{\tau}-1}\!\!\sum_{s=r\tau+1}^{r\tau+\tau_1}\!\!\!\!\|\nabla F(\textbf{X}_s)\|^2_{\textup{F}}\notag\\
\!+\!&\frac{L^2}{TN}\!\frac{\eta^2\tau}{1-\zeta^{\tau_2}}\!\!\left(\frac{2\tau_1\zeta^{2\tau_2}}{1+\zeta^{\tau_2}}\!\!+\!\!\frac{2\tau_1\zeta^{\tau_2}}{1-\zeta^{\tau_2}}\!\!+\!\tau\!-\!1\!
\right)\!\!\sum_{r=0}^{\frac{T}{\tau}-1}\!\sum_{s=r\tau+1}^{r\tau+\tau_1}\!\!\mathbb{E}\|\!\nabla \!F\!(\textbf{X}_{s})\!\|^2_{\textup{F}}\notag\\
-&\frac{1}{NT}\left[1-\eta L\left(\frac{\beta}{N}+1\right)\right]\sum_{t=1}^{T}\mathbb{E}\|\nabla F(\textbf{X}_t)\|^2_{\textup{F}}.
\end{align}
Because there exists no gradient during inter-node communication period, i.e., $t\in[k]_2$, we have $\sum_{t=1}^{T}\mathbb{E}\|\nabla F(\textbf{X}_t)\|^2_{\textup{F}}=\sum_{r=0}^{\frac{T}{\tau}-1}\!\sum_{s=r\tau+1}^{r\tau+\tau_1}\!\!\mathbb{E}\|\!\nabla \!F\!(\textbf{X}_{s})\!\|^2_{\textup{F}}$. So if the learning rate satisfies the following inequality:
\begin{align}
\eta L\!\left(\!\!\frac{\beta}{N}\!\!+\!\!1\!\right)\!\!+\!\!\frac{2\eta^2 L^2\beta\tau}{1-\zeta^{2\tau_2}}\!\!+\!\!\frac{\eta^2 L^2\tau}{1-\zeta^{\tau_2}}\!\!\left(\frac{2\tau_1\zeta^{2\tau_2}}{1+\zeta^{\tau_2}}\!\!+\!\!\frac{2\tau_1\zeta^{\tau_2}}{1-\zeta^{\tau_2}}\!\!+\!\tau\!-\!1\!
\right)\!\leq \!1,
\end{align}
we can obtain 
\begin{align}
\mathbb{E}\left[\frac{1}{T}\sum_{t=1}^{T}\|\nabla F(\textbf{u}_t)\|^2\right]&\leq\frac{2[F(\textbf{u}_1)-F_{\textup{inf}}]}{\eta T}+\frac{\eta L\sigma^2}{N}\notag\\
&+2 L^2\eta^2\sigma^2\!\!\left(\frac{\tau_1}{1-\zeta^{2\tau_2}}\!\!-\!\!1\right).
\end{align}
Therefore, the proof of Proposition 1 is finished.

\section{Proof of Proposition 2}\label{appendix3}
\subsection{Preliminaries of Proof}
Firstly, we present some basic definition and inequalities for the proof.
\begin{remark}
  We define $\bar{\textup{\textbf{X}}}_t\triangleq \textup{\textbf{X}}_t\textup{\textbf{J}}$, then $\bar{\textup{\textbf{X}}}_t=\left[\textup{\textbf{u}}_{t}, ..., \textup{\textbf{u}}_{t}\right]$.
\end{remark}
\begin{remark}\label{lsmooth}
  If $F(\cdot)$ is $L$-smooth for parameter $L\geq0$, we have 
  $$F(\textup{\textbf{y}})\leq F(\textup{\textbf{x}})+\langle\nabla F(\textup{\textbf{x}}), \textup{\textbf{y}}-\textup{\textbf{x}}\rangle+\frac{L}{2}\|\textup{\textbf{y}}-\textup{\textbf{x}}\|^2,$$ $\forall \textup{\textbf{x}}, \textup{\textbf{y}}\in \mathbb{R}^d.$
\end{remark}
\begin{remark}\label{llsmooth}
  If $F(\cdot)$ is $L$-smooth with minimizer $\textup{\textbf{x}}^*$ (i.e., $\nabla F(\textup{\textbf{x}}^*)=0$), then 
  $$\|\nabla F(\textup{\textbf{x}})\|^2=\|\nabla F(\textup{\textbf{x}})-\nabla F(\textup{\textbf{x}}^*)\|^2\leq 2L(F(\textup{\textbf{x}})-F(\textup{\textbf{x}}^*)).$$
\end{remark}
\begin{remark}\label{mustrong}
  If $F(\cdot)$ is $\mu$-strong for parameter $\mu\geq0$, we have
  $$F(\textup{\textbf{y}})\geq F(\textup{\textbf{x}}) +\langle\nabla F(\textup{\textbf{x}}), \textup{\textbf{y}}-\textup{\textbf{x}}\rangle+\frac{\mu}{2}\|\textup{\textbf{y}}-\textup{\textbf{x}}\|^2,$$ $\forall \textup{\textbf{x}}, \textup{\textbf{y}}\in\mathbb{R}^d$.
\end{remark}
\begin{remark}
  For $A\in \mathbb{R}^{d\times N}$, $B\in\mathbb{R}^{N\times N}$, $\|AB\|_F\leq \|A\|_F\|B\|_2$. 
\end{remark}
\begin{remark}
  If $\mathbb{E}_{\xi|\textbf{\textup{w}}}\|g(\textup{\textbf{w}}^{(i)})\|^2\leq G^2$, then for $\textbf{\textup{G}}=\left[g(\textbf{\textup{w}}^{(1)}), ..., g(\textbf{\textup{w}}^{(N)})\right]$, we have $$\mathbb{E}_{\xi_t^{(1)},...,\xi_t^{(N)}}\|\textbf{\textup{G}}\|_F^2\leq NG^2.$$
\end{remark}
\begin{remark}
  If for $F_i(\cdot)$ with $\mathbb{E}_{\xi|\textbf{\textup{w}}}\|g(\textup{\textbf{w}}^{(i)})-\nabla F_i(\textup{\textbf{w}})\|^2\leq \sigma_i^2$, then 
  \begin{align*}
  \mathbb{E}_{\xi_t^{(1)},...,\xi_t^{(N)}}\left\|\frac{1}{N}\sum_{i=1}^{N}\left(\nabla F_i(\textup{\textbf{w}}_t^{(i)})-g(\textup{\textbf{w}}^{(i)})\right)\right\|^2\leq\frac{\bar{\sigma}^2}{N},
  \end{align*} 
  where $\bar{\sigma}^2=\frac{\sum_{i=1}^{N}\sigma_i^2}{N}$.
\end{remark}
\begin{proof}
  Assuming $Y_i=\nabla F_i(\textup{\textbf{w}}_t^{(i)})-g(\textup{\textbf{w}}^{(i)})$, then 
  \begin{align*}
  \mathbb{E}\left\|\frac{1}{N}\sum_{i=1}^{N}Y_i\right\|^2&=\frac{1}{N^2}\left(\sum_{i=1}^{N}\mathbb{E}\|Y_i\|^2+\sum_{i\neq j}\mathbb{E}\langle Y_i, Y_j\rangle\right)\\
  &=\frac{1}{N^2}\sum_{i=1}^{N}\mathbb{E}\|Y_i\|^2\leq \frac{1}{N^2}\sum_{i=1}^{N}\sigma_i^2=\frac{\bar{\sigma}^2}{N}. 
  \end{align*} 
  Expectation of the product of $Y_i$ and $Y_j$ is zero because $\xi^{(i)}$ is independent of $\xi^{(j)}$ for $i\neq j$. 
\end{proof}
\begin{remark}\label{remark10}
  For given two vectors $\textup{\textbf{a}}, \textup{\textbf{b}}\in\mathbb{R}^d$,
  \begin{align*}
  \|\textup{\textbf{a}}+\textup{\textbf{b}}\|^2\leq(1+\theta)\|\textup{\textbf{a}}\|^2+(1+\theta^{-1})\|\textup{\textbf{b}}\|^2, \forall \theta>0.
  \end{align*}
  This inequality also holds for the sum of two matrices $A,B\in\mathbb{R}^{d\times N}$ in Frobenius norm.
\end{remark}

\subsection{The supporting lemmas}
We present several supporting lemmas to prove Proposition 2 as follows.
We define $\textbf{u}^{(k)}\triangleq \textbf{u}_{k\tau}$ for simplification, which is the average of model parameters after $k$ iteration rounds. 
\begin{lemma}\label{lem8}
  If there exists $a_k\in\mathbb{R}$ which satisfies $\alpha_k (F(\textbf{\textup{u}}^{(k)})-F^*)\geq \frac{1}{\tau_1}\sum_{j=0}^{\tau_1-1}(F(\textbf{\textup{u}}_{k\tau+j})-F^*)$, then the averages $\textbf{\textup{u}}^{(k)}$ of Algorithm 2 satisfy 
  \begin{align}
  &\mathbb{E}_{\xi_{k\tau}^{(1)},...,\xi_{k\tau}^{(N)}}\|\textbf{\textup{u}}^{(k+1)}-\textbf{\textup{u}}^{*}\|^2\leq\left(1-\frac{\mu\eta_{k}\tau_1}{2}\right)\|\textbf{\textup{u}}^{(k)}-\textbf{\textup{u}}^{*}\|^2 +\notag \\
  &\eta_k^2\tau_1^2\frac{\bar{\sigma}^2}{N} + \frac{(L+\mu)\eta_k^3\bar{\sigma}^2(\tau_1-1)(2\tau_1-1)\tau_1}{6N} +\notag\\ &\!(\!F(\textbf{\textup{u}}^{(k)})\!-\!F^*\!)\!\left[\!\!4L\eta_k^2\tau_1^2\alpha_k\!\!-\!\!2\eta_k\tau_1\!\!+\!\!\frac{4}{3}(L\!\!+\!\!\mu)\eta_{k}^3L\alpha_k(\tau_1\!\!-\!\!1)(2\tau_1\!\!-\!\!1)\tau_1\!\!\right]\notag\\
  &\!+\! \!\frac{2L^2\eta_k^2\tau_1}{N}\!\Phi_k \!\!+\!\! \frac{2(L+\mu)\eta_k}{N}\Phi_k\!\!+\!\!\frac{(L+\mu)\eta_k}{N}\!\!\sum_{j=0}^{\tau_1-1}\sum_{i=1}^{N}\frac{4\eta_k^2L^2j^2}{N}\Phi_{k,j},
  \end{align} 
  where $\Phi_k=\sum_{j=0}^{\tau_1-1}\sum_{i=1}^{N}\|\textbf{\textup{u}}_{k\tau+j}-\textbf{\textup{w}}_{k\tau+j}^{(i)}\|^2$ and $\Phi_{k,j}=\sum_{p=0}^{j-1}\sum_{i=1}^{N}\|\textbf{\textup{u}}_{k\tau+p}-\textbf{\textup{w}}_{k\tau+p}^{(i)}\|^2$.
\end{lemma}
\begin{proof}
  According to the update of model averaging, we have
  \begin{align*}
  &\|\textbf{u}^{(k+1)}-\textbf{u}^*\|^2\\
  =&\left\|\textbf{u}^{(k)}-\textbf{u}^*-\eta_{k}\sum_{j=0}^{\tau_1-1}\frac{1}{N}\sum_{q=1}^{N}g(\textbf{w}_{k\tau+j}^{(q)})\right\|^2\\
  =&\left\|\textbf{u}^{(k)}-\textbf{u}^*-\sum_{j=0}^{\tau_1-1}\frac{\eta_{k}}{N}\sum_{i=1}^{N}\nabla F_i(\textbf{w}_{k\tau+j}^{(i)})+\right.\\
  &\left.\sum_{j=0}^{\tau_1-1}\frac{\eta_{k}}{N}\sum_{i=1}^{N}\nabla F_i(\textbf{w}_{k\tau+j}^{(i)})
  -\sum_{j=0}^{\tau_1-1}\frac{\eta_{k}}{N}\sum_{q=1}^{N}g(\textbf{w}_{k\tau+j}^{(q)})\right\|^2\\ 
  =&\left\|\textbf{u}^{(k)}-\textbf{u}^*-\sum_{j=0}^{\tau_1-1}\frac{\eta_{k}}{N}\sum_{i=1}^{N}\nabla F_i(\textbf{w}_{k\tau+j}^{(i)})\right\|^2+\\
  &\left\|\sum_{j=0}^{\tau_1-1}\frac{\eta_{k}}{N}\sum_{i=1}^{N}\nabla F_i(\textbf{w}_{k\tau+j}^{(i)})-\sum_{j=0}^{\tau_1-1}\frac{\eta_{k}}{N}\sum_{q=1}^{N}g(\textbf{w}_{k\tau+j}^{(q)})\right\|^2\\
  &+\frac{2\eta_k}{N}\left\langle\textbf{u}^{(k)}-\textbf{u}^*-\sum_{j=0}^{\tau_1-1}\frac{\eta_{k}}{N}\sum_{i=1}^{N}\nabla F_i(\textbf{w}_{k\tau+j}^{(i)}),\right.\\
  &\left.\sum_{j=0}^{\tau_1-1}\sum_{i=1}^{N}\nabla F_i(\textbf{w}_{k\tau+j}^{(i)})-\sum_{j=0}^{\tau_1-1}\sum_{q=1}^{N}g(\textbf{w}_{k\tau+j}^{(q)})\right\rangle.
  \end{align*}
  The last term is zero because $\mathbb{E}_{\xi^{(i)}|\textup{\textbf{w}}^{(i)}}[g(\textup{\textbf{w}}_t^{(i)})]=\nabla F(\textup{\textbf{w}}_t^{(i)})$. The second term is less than $\eta_{k}^2\tau_1^2\frac{\bar{\sigma}^2}{N}$. We rewrite the first term as 
  \begin{align*}
  &\left\|\textbf{u}^{(k)}-\textbf{u}^*-\sum_{j=0}^{\tau_1-1}\frac{\eta_{k}}{N}\sum_{i=1}^{N}\nabla F_i(\textbf{w}_{k\tau+j}^{(i)})\right\|^2\\
  =&\|\textbf{u}^{(k)}-\textbf{u}^*\|^2+\eta_{k}^2\underbrace{\left\|\sum_{j=0}^{\tau_1-1}\frac{1}{N}\sum_{i=1}^{N}\nabla F_i(\textbf{w}_{k\tau+j}^{(i)})\right\|^2}_{\mbox{\footnotesize $T_1$}}\\
  &-\underbrace{2\eta_{k}\left\langle\textbf{u}^{(k)}-\textbf{u}^*,\sum_{j=0}^{\tau_1-1}\frac{1}{N}\sum_{i=1}^{N}\nabla F_i(\textbf{w}_{k\tau+j}^{(i)})\right\rangle}_{\mbox{\footnotesize $T_2$}}.
  \end{align*}
  Considering $T_1$, we have 
  \begin{align}
  &T_1\notag\\
  =&\left\|\!\sum_{j=0}^{\tau_1-1}\!\!\frac{1}{N}\!\!\sum_{i=1}^{N}\!\!\left(\nabla F_i(\textbf{w}_{k\tau+j}^{(i)})\!\!-\!\!\nabla F_i(\textbf{u}_{k\tau+j})\!\!+\!\!\nabla F_i(\textbf{u}_{k\tau+j})\right.\right.\notag\\
  &\left.\left. -\nabla F_i(\textbf{w}^*)\right)\right\|^2\notag\\
  \leq & 2\tau_1 \sum_{j=0}^{\tau_1}\left\|\frac{1}{N}\sum_{i=1}^{N}\left(\nabla F_i(\textbf{w}_{k\tau+j}^{(i)})\!\!-\!\!\nabla F_i(\textbf{u}_{k\tau+j})\right)\right\|^2+ \notag\\
  &2\tau_1\sum_{j=0}^{\tau_1-1}\left\|\frac{1}{N}\sum_{i=1}^{N}\left(\nabla F_i(\textbf{u}_{k\tau+j})\!\!-\!\!\nabla F_i(\textbf{w}^*)\right)\right\|^2\notag\\
  \label{aboutlsmooth}\leq&\frac{2L^2}{N}\tau_1\sum_{j=0}^{\tau_1-1}\sum_{i=1}^{N}\|\textbf{w}_{k\tau+j}^{(i)}-\textbf{u}_{k\tau+j}\|^2\notag\\
  &+4L\tau_1\sum_{j=0}^{\tau_1-1}(F(\textbf{u}_{k\tau+j})-F^*)\\
  \label{ak}\leq&\frac{2L^2}{N}\tau_1\sum_{j=0}^{\tau_1-1}\sum_{i=1}^{N}\|\textbf{w}_{k\tau+j}^{(i)}-\textbf{u}_{k\tau+j}\|^2+4L\tau_1^2 \alpha_k(F(\textbf{u}^{(k)})-F^*).
  \end{align} 
  According to condition $1$ of Assumption 1 and Remark 5, we obtain \eqref{aboutlsmooth}. And \eqref{ak} is due to the assumption of $\alpha_k$: $\alpha_k (F(\textbf{\textup{u}}^{(k)})-F^*)\geq \frac{1}{\tau_1}\sum_{j=0}^{\tau_1-1}(F(\textbf{\textup{u}}_{k\tau+j})-F^*)$.
  
  Then considering $T_2$, we have
  \begin{align}
  &-\frac{1}{\eta_{k}}T_2\notag\\
  =&-\frac{2}{N}\sum_{j=0}^{\tau_1-1}\sum_{i=1}^{N}\langle\textbf{u}^{(k)}-\textbf{u}^*, \nabla F_i(\textbf{w}_{k\tau+j}^{(i)})\rangle\notag\\
  =&-\frac{2}{N}\sum_{j=0}^{\tau_1-1}\sum_{i=1}^{N}\left[\langle \textbf{u}^{(k)}-\textbf{w}_{k\tau+j}^{(i)}, \nabla F_i(\textbf{w}_{k\tau+j}^{(i)})\rangle\right.\notag\\
  &\left.+\langle \textbf{w}_{k\tau+j}^{(i)}-\textbf{u}^*,\nabla F_i(\textbf{w}_{k\tau+j}^{(i)})\rangle\right]\notag\\
  \label{remark46}\leq&-\frac{2}{N}\sum_{j=0}^{\tau_1-1}\sum_{i=1}^{N}\left[F_i(\textbf{u}^{(k)})-F_i(\textbf{w}_{k\tau+j}^{(i)})-\frac{L}{2}\|\textbf{u}^{(k)}-\textbf{w}_{k\tau+j}^{(i)}\|^2\right.\notag\\
  &\left.+F_i(\textbf{w}_{k\tau+j}^{(i)})-F_i(\textbf{u}^*)+\frac{\mu}{2}\|\textbf{w}_{k\tau+j}^{(i)}-\textbf{u}^*\|^2\right]\\
  \label{remarkno}\leq&\sum_{j=0}^{\tau_1-1}\left[-2(F(\textbf{u}^{(k)})-F(\textbf{u}^*))+\frac{L+\mu}{N}\sum_{i=1}^{N}\|\textbf{u}^{(k)}-\textbf{w}_{k\tau+j}^{(i)}\|^2\right.\notag\\
  &\left.-\frac{\mu}{2}\|\textbf{u}^{(k)}-\textbf{u}^*\|^2\right]\\
  \label{nan}=&-2\tau_1(F(\textbf{u}^{(k)})-F^*)+\frac{L+\mu}{N}\sum_{j=0}^{\tau_1-1}\sum_{i=1}^{N}\|\textbf{u}^{(k)}-\textbf{w}_{k\tau+j}^{(i)}\|^2\notag\\
  &-\frac{\mu\tau_1}{2}\|\textbf{u}^{(k)}-\textbf{u}^*\|^2,
  \end{align}
  where \eqref{remark46} is because of Remark 4 and 6, and \eqref{remarkno} is because $\left\|\sum_{i=1}^{n}\textbf{a}_i\right\|^2\leq n\sum_{i=1}^{n}\|\textbf{a}_i\|^2$. Focusing on the second term of \eqref{nan}, we have
  \begin{align}\label{abcd}
  &\|\textbf{u}^{(k)}-\textbf{w}_{k\tau+j}^{(i)}\|^2\notag\\
  =&\left\|\textbf{u}_{k\tau+j}-\textbf{w}_{k\tau+j}^{(i)}+\frac{\eta_{k}}{N}\sum_{p=0}^{j-1}\sum_{i=1}^{N}g(\textbf{w}_{k\tau+p}^{(i)})\right\|^2\notag\\
  =&\left\|\textbf{u}_{k\tau+j}-\textbf{w}_{k\tau+j}^{(i)}+\frac{\eta_{k}}{N}\sum_{p=0}^{j-1}\sum_{i=1}^{N}g(\textbf{w}_{k\tau+p}^{(i)})\right.\notag\\
  &\left.-\frac{\eta_{k}}{N}\sum_{p=0}^{j-1}\sum_{i=1}^{N}\nabla F_i(\textbf{w}_{k\tau+p}^{(i)})+\frac{\eta_{k}}{N}\sum_{p=0}^{j-1}\sum_{i=1}^{N}\nabla F_i(\textbf{w}_{k\tau+p}^{(i)})\right\|^2\notag\\
  =&\left\|\textbf{u}_{k\tau+j}-\textbf{w}_{k\tau+j}^{(i)}+\frac{\eta_{k}}{N}\sum_{p=0}^{j-1}\sum_{i=1}^{N}\nabla F_i(\textbf{w}_{k\tau+p}^{(i)})\right\|^2+\notag\\
  &\eta_{k}^2\left\|\frac{1}{N}\sum_{p=0}^{j-1}\sum_{i=1}^{N}(g(\textbf{w}_{k\tau+p}^{(i)})-\nabla F_i(\textbf{w}_{k\tau+p}^{(i)}))\right\|^2\notag\\
  \leq& \left\|\textbf{u}_{k\tau+j}-\textbf{w}_{k\tau+j}^{(i)}+\frac{\eta_{k}}{N}\sum_{p=0}^{j-1}\sum_{i=1}^{N}\nabla F_i(\textbf{w}_{k\tau+p}^{(i)})\right\|^2+\frac{\eta_{k}^2 j^2 \bar{\sigma}^2}{N}\notag\\
  \leq&2\|\textbf{u}_{k\tau+j}-\textbf{w}_{k\tau+j}^{(i)}\|^2+2\eta_{k}^2\left\|\frac{1}{N}\sum_{p=0}^{j-1}\sum_{i=1}^{N}\nabla F_i(\textbf{w}_{k\tau+p}^{(i)})\right\|^2\notag\\
  &+\frac{\eta_{k}^2 j^2 \bar{\sigma}^2}{N}.
  \end{align}
  Because the second term in the last inequality has a similar form with $T_1$, then following \eqref{ak}, we can obtain 
  \begin{align}\label{t1}
  &\left\|\frac{1}{N}\sum_{p=0}^{j-1}\sum_{i=1}^{N}\nabla F_i(\textbf{w}_{k\tau+p}^{(i)})\right\|^2\notag\\
  \leq&\frac{2L^2}{N}j\sum_{p=0}^{j-1}\sum_{i=1}^{N}\|\textbf{w}_{k\tau+p}^{(i)}-\textbf{u}_{k\tau+p}\|^2+4Lj^2 \alpha_k(F(\textbf{u}^{(k)})-F^*).
  \end{align}
  By substituting \eqref{t1} into \eqref{abcd}, we have 
  \begin{align}\label{115}
  &\|\textbf{u}^{(k)}-\textbf{w}_{k\tau+j}^{(i)}\|^2\notag\\
  \leq&2\|\textbf{u}_{k\tau+j}-\textbf{w}_{k\tau+j}^{(i)}\|^2+\frac{4\eta_{k}^2 L^2 j}{N}\sum_{p=0}^{j-1}\sum_{i=1}^{N}\|\textbf{w}_{k\tau+p}^{(i)}-\textbf{u}_{k\tau+p}\|^2+\notag\\
  &8\eta_{k}^2 L j^2 \alpha_k(F(\textbf{u}^{(k)})-F^*)+\frac{\eta_{k}^2 j^2 \bar{\sigma}^2}{N}.
  \end{align}
  Then by putting \eqref{115} into \eqref{nan}, we have the upper bound of $-\frac{1}{\eta_{k}}T_2$ as
  \begin{align}\label{116}
  &-\frac{1}{\eta_{k}}T_2\leq\notag\\
  &-2\tau_1(F(\textbf{u}^{(k)})-F^*)+\frac{L+\mu}{N}\sum_{j=0}^{\tau_1-1}\sum_{i=1}^{N}\|\textbf{u}^{(k)}-\textbf{w}_{k\tau+j}^{(i)}\|^2\notag\\
  &-\frac{\mu\tau_1}{2}\|\textbf{u}^{(k)}-\textbf{u}^*\|^2\notag\\
  &\leq -2\tau_1(F(\textbf{u}^{(k)})-F^*)+\frac{2(L+\mu)}{N}\sum_{j=0}^{\tau_1-1}\sum_{i=1}^{N}\|\textbf{u}_{k\tau+j}-\textbf{w}_{k\tau+j}^{(i)}\|^2\notag\\
  &+\frac{L+\mu}{N}\sum_{j=0}^{\tau_1-1}\sum_{i=1}^{N}\frac{4\eta_{k}^2 L^2 j}{N}\sum_{p=0}^{j-1}\sum_{i=1}^{N}\|\textbf{w}_{k\tau+p}^{(i)}-\textbf{u}_{k\tau+p}\|^2+\notag\\
  &\frac{4}{3}(L+\mu)L\eta_{k}^2\alpha_k(\tau_1-1)(2\tau_1-1)\tau_1(F(\textbf{u}^{(k)})-F^*)+\notag\\
  &\frac{(L+\mu)\eta_{k}^2\bar{\sigma}^2(\tau_1-1)(2\tau_1-1)\tau_1}{6N}-\frac{\mu\tau_1}{2}\|\textbf{u}^{(k)}-\textbf{u}^*\|^2.
  \end{align}
  Therefore, we can finish the proof by putting the upper bounds of $T_1$ and $T_2$ into the original inequality as follows.
  \begin{align*}
  &\|\textbf{u}^{(k+1)}-\textbf{u}^*\|^2\\
  \leq& \|\textbf{u}^{(k)}-\textbf{u}^*\|^2+\frac{2L^2\eta_{k}^2}{N}\tau_1\sum_{j=0}^{\tau_1-1}\sum_{i=1}^{N}\|\textbf{w}_{k\tau+j}^{(i)}-\textbf{u}_{k\tau+j}\|^2\\
  &+4\eta_{k}^2L\tau_1^2 \alpha_k(F(\textbf{u}^{(k)})-F^*)-2\eta_{k}\tau_1(F(\textbf{u}^{(k)})-F^*)+\\
  &\frac{2(L+\mu)\eta_{k}}{N}\sum_{j=0}^{\tau_1-1}\sum_{i=1}^{N}\|\textbf{u}_{k\tau+j}-\textbf{w}_{k\tau+j}^{(i)}\|^2+\\
  &\frac{(L+\mu)\eta_{k}}{N}\sum_{j=0}^{\tau_1-1}\sum_{i=1}^{N}\frac{4\eta_{k}^2 L^2 j}{N}\sum_{p=0}^{j-1}\sum_{i=1}^{N}\|\textbf{w}_{k\tau+p}^{(i)}-\textbf{u}_{k\tau+p}\|^2+\\
  &\frac{4}{3}(L+\mu)L\eta_{k}^3\alpha_k(\tau_1-1)(2\tau_1-1)\tau_1(F(\textbf{u}^{(k)})-F^*)+\\
  &\frac{(L+\mu)\eta_{k}^3\bar{\sigma}^2(\tau_1-1)(2\tau_1-1)\tau_1}{6N}-\frac{\mu\eta_{k}\tau_1}{2}\|\textbf{u}^{(k)}-\textbf{u}^*\|^2\\
  &+\eta_{k}^2\tau_1^2\frac{\bar{\sigma}^2}{N}\\
  =&(1-\frac{\mu\eta_{k}\tau_1}{2})\|\textbf{u}^{(k)}-\textbf{u}^*\|^2+\eta_{k}^2\tau_1^2\frac{\bar{\sigma}^2}{N}+\\
  &\frac{(L+\mu)\eta_{k}^3\bar{\sigma}^2(\tau_1-1)(2\tau_1-1)\tau_1}{6N}+\\
  &(F(\textbf{u}^{(k)})-F^*)\left[-2\eta_{k}\tau_1+4\eta_{k}^2L\tau_1^2 \alpha_k+\right.\\
  &\left.\frac{4}{3}(L+\mu)L\eta_{k}^3\alpha_k(\tau_1-1)(2\tau_1-1)\tau_1\right]+\\
  &\frac{2(L+\mu)\eta_{k}}{N}\sum_{j=0}^{\tau_1-1}\sum_{i=1}^{N}\|\textbf{u}_{k\tau+j}-\textbf{w}_{k\tau+j}^{(i)}\|^2+\\
  &\frac{2L^2\eta_{k}^2}{N}\tau_1\sum_{j=0}^{\tau_1-1}\sum_{i=1}^{N}\|\textbf{w}_{k\tau+j}^{(i)}-\textbf{u}_{k\tau+j}\|^2+\\
  &\frac{(L+\mu)\eta_{k}}{N}\sum_{j=0}^{\tau_1-1}\sum_{i=1}^{N}\frac{4\eta_{k}^2 L^2 j}{N}\sum_{p=0}^{j-1}\sum_{i=1}^{N}\|\textbf{w}_{k\tau+p}^{(i)}-\textbf{u}_{k\tau+p}\|^2.
  \end{align*} 
\end{proof}

We define the function $h: \mathbb{R}^{d\times N}\times\mathbb{R}^{d\times N}\to\mathbb{R}^{d\times N}\times\mathbb{R}^{d\times N}$ denotes the average consensus operator, where $(\textbf{X}_{t+1},\textbf{Y}_{t+1})=h(\textbf{X}_t,\textbf{Y}_t)$. Therefore, for CHOCO-G algorithm which is used in C-DFL, the operator $h$ is
\begin{align*}
&\textbf{X}_{t+1}=\textbf{X}_t+\gamma \textbf{Y}_t(\textbf{C}-\textbf{I}),\\
&\textbf{Y}_{t+1}=\textbf{Y}_t+Q(\textbf{X}_{t+1}-\textbf{Y}_t),
\end{align*}
where $\textbf{Y}_t=\hat{\textbf{X}}_t$. For convenience, we use $\textbf{X}^{(k)},\bar{\textbf{X}}^{(k)}$ to denote $\textbf{X}_{k\tau}, \bar{\textbf{X}}_{k\tau}$, respectively.

\begin{lemma}[\cite{koloskova2019decentralized}]\label{lem9}
  CHOCO-G converges linearly for average consensus:
  \begin{align*}
  e_t=\left(1-p\right)^te_0,
  \end{align*}
  where $p=\frac{\rho^2\delta}{82}$, where $\delta$ is the compression ratio as in Assumption 2, $e_t=\mathbb{E}_Q\sum_{i=1}^{N}\left(\|\textup{\textbf{w}}_t^{(i)}-\textup{\textbf{u}}_t\|^2+\|\textup{\textbf{w}}_t^{(i)}-\hat{\textup{\textbf{w}}}_t^{(i)}\|^2\right)$ and using the step size $\gamma=\frac{\rho^2\delta}{16\rho+\rho^2+4\beta^2+2\rho\beta^2-8\rho\delta}$.
\end{lemma}

From Lemma \ref{lem9}, we have 
\begin{align}\label{117}
&\mathbb{E}\left[\|\textbf{X}_{t+1}-\bar{\textbf{X}}_{t+1}\|^2_F+\|\textbf{X}_{t+1}-\textbf{Y}_{t+1}\|^2_F\right]\leq\notag\\
&(1-p)\mathbb{E}\left[\|\textbf{X}_{t}-\bar{\textbf{X}}_{t}\|^2_F+\|\textbf{X}_{t}-\textbf{Y}_{t}\|^2_F\right].
\end{align}

\begin{lemma}\label{lem10}
  The iterates of C-DFL $\{\textup{\textbf{X}}^{(k)}\}$ satisfy 
  \begin{align*}
  &\left\|\textup{\textbf{X}}^{(k)}-\bar{\textup{\textbf{X}}}^{(k)}\right\|^2_{F}\\ \leq&\frac{2(1+\theta^{-1})\eta_{k}^2\tau_1^2NG^2(1-p)^{\tau_2}}{(1+\theta)(1-p)^{\tau_2}-1}[((1+\theta)(1-p)^{\tau_2})^{k}-1].
  \end{align*}
\end{lemma}
\begin{proof}
  Using \eqref{117} which shows the linear convergence of CHOCO-G, we have
  \begin{align}
  &\mathbb{E}\left[\left\|\textbf{X}^{(k+1)}-\bar{\textbf{X}}^{(k+1)}\right\|^2_F+\left\|\textbf{X}^{(k+1)}-\textbf{Y}^{(k+1)}\right\|^2_F\right]\leq\notag\\
  &(1-p)^{\tau_2}\mathbb{E}\left\|\textbf{X}_{k\tau+\tau_1}\!-\!\bar{\textbf{X}}_{k\tau+\tau_1}\right\|^2_F\!+\!(1\!-\!p)^{\tau_2}\mathbb{E}\left\|\textbf{X}_{k\tau+\tau_1}\!-\!\textbf{Y}_{k\tau}\right\|^2_{F}\notag\\
  &=(1-p)^{\tau_2}\mathbb{E}\left\|\textbf{X}_{k\tau}-\bar{\textbf{X}}_{k\tau}+\eta_{k}\sum_{j=0}^{\tau_1-1}\textbf{G}_{k\tau+j}(\textbf{J}-\textbf{I})\right\|^2_{F}+\notag\\
  &(1-p)^{\tau_2}\mathbb{E}\left\|\textbf{X}_{k\tau}-\textbf{Y}_{k\tau}-\eta_{k}\sum_{j=0}^{\tau_1-1}\textbf{G}_{k\tau+j}\right\|^2_{F}\notag\\
  &\leq(1\!+\!\theta)(1-p)^{\tau_2}\mathbb{E}\left(\left\|\textbf{X}_{k\tau}-\bar{\textbf{X}}_{k\tau}\right\|^2_{F}+\left\|\textbf{X}_{k\tau}-\textbf{Y}_{k\tau}\right\|^2_{F}\right)+\notag\\
  \label{dsdds}&(1\!+\!\theta^{-1})(1\!-\!p)^{\tau_2}\eta_{k}^2\mathbb{E}\left(\left\|\sum_{j=0}^{\tau_1-1}\textbf{G}_{k\tau+j}(\textbf{J}-\textbf{I})\right\|^2_F\!\!\!\!+\!\!\left\|\sum_{j=0}^{\tau_1-1}\textbf{G}_{k\tau+j}\right\|^2_{F}\!\!\right)\\
  \label{118}&\leq(1+\theta)(1-p)^{\tau_2}\mathbb{E}\left(\left\|\textbf{X}_{k\tau}-\bar{\textbf{X}}_{k\tau}\right\|^2_{F}+\left\|\textbf{X}_{k\tau}-\textbf{Y}_{k\tau}\right\|^2_{F}\right)\notag\\
  &+2(1+\theta^{-1})(1-p)^{\tau_2}\eta_{k}^2\tau_1^2NG^2,
  \end{align} where \eqref{dsdds} follows from Remark \ref{remark10} and \eqref{118} is from Remark 7 and 8.
  For sequence $\{e_k\}$, if $e_{k+1}\leq Ae_k+B$ where $A$ and $B$ are constants, then we can easily obtain 
  $$e_{k}\leq\frac{B}{A-1}(A^k-1),$$ 
  when $e_0=0$. We set $e_k=\mathbb{E}\left(\left\|\textbf{X}_{k\tau}-\bar{\textbf{X}}_{k\tau}\right\|^2_{F}+\left\|\textbf{X}_{k\tau}-\textbf{Y}_{k\tau}\right\|^2_{F}\right)$ with $e_0=0$, $A=2(1-p)^{\tau_2}$ and $B=4(1-p)^{\tau_2}\eta_{k}^2\tau_1^2NG^2$. Then
  \begin{align*}
  &\mathbb{E}\left(\left\|\textbf{X}_{k\tau}-\bar{\textbf{X}}_{k\tau}\right\|^2_{F}+\left\|\textbf{X}_{k\tau}-\textbf{Y}_{k\tau}\right\|^2_{F}\right)\\
  \leq&\frac{2(1+\theta^{-1})\eta_{k}^2\tau_1^2NG^2(1-p)^{\tau_2}}{(1+\theta)(1-p)^{\tau_2}-1}[((1+\theta)(1-p)^{\tau_2})^{k}-1].
  \end{align*}
  We finish the proof of Lemma \ref{lem10} from
  \begin{align*}
  \mathbb{E}\left(\left\|\textup{\textbf{X}}^{(k)}-\bar{\textup{\textbf{X}}}^{(k)}\right\|^2_{F}\right)\leq\mathbb{E}\left(\left\|\textbf{X}_{k\tau}-\bar{\textbf{X}}_{k\tau}\right\|^2_{F}+\left\|\textbf{X}_{k\tau}-\textbf{Y}_{k\tau}\right\|^2_{F}\!\right)\!.
  \end{align*} 
\end{proof}

\begin{lemma}\label{lem11}
  Let $\Phi_k=\sum_{j=0}^{\tau_1-1}\sum_{i=1}^{N}\|\textbf{\textup{u}}_{k\tau+j}-\textbf{\textup{w}}_{k\tau+j}^{(i)}\|^2$ and $\Phi_{k,j}=\sum_{p=0}^{j-1}\sum_{i=1}^{N}\|\textbf{\textup{u}}_{k\tau+p}-\textbf{\textup{w}}_{k\tau+p}^{(i)}\|^2$. We have 
  \begin{align}
  \label{119}\Phi_k&\leq\frac{4(1+\theta^{-1})\eta_{k}^2\tau_1^3NG^2(1-p)^{\tau_2}}{(1+\theta)(1-p)^{\tau_2}-1}[((1+\theta)(1-p)^{\tau_2})^{k}-1]\notag\\
  &+\frac{1}{3}\eta_{k}^2NG^2(\tau_1-1)(2\tau_1-1)\tau_1,\\
  \label{120}\Phi_{k,j}&\leq\frac{4(1+\theta^{-1})\eta_{k}^2j^3NG^2(1-p)^{\tau_2}}{(1+\theta)(1-p)^{\tau_2}-1}[((1+\theta)(1-p)^{\tau_2})^{k}-1]\notag\\
  &+\frac{1}{3}\eta_{k}^2NG^2(j-1)(2j-1)j.
  \end{align}
\end{lemma}
\begin{proof}
  Straightforwardly, We have
  \begin{align*}
  &\left\|\textbf{X}_{k\tau+j}-\bar{\textbf{X}}_{k\tau+j}\right\|^2_{F}\\
  =&\left\|\textbf{X}_{k\tau}-\bar{\textbf{X}}_{k\tau}+\eta_{k}\sum_{p=0}^{j-1}\textbf{G}_{k\tau+p}(\textbf{J}-\textbf{I})\right\|^2_{F}\\
  \leq&2\left\|\textbf{X}_{k\tau}-\bar{\textbf{X}}_{k\tau}\right\|^2_F+2\eta_{k}^2\left\|\sum_{p=0}^{j-1}\textbf{G}_{k\tau+p}(\textbf{J}-\textbf{I})\right\|^2_{F}\\
  \leq&2\left\|\textbf{X}_{k\tau}-\bar{\textbf{X}}_{k\tau}\right\|^2_F+2\eta_{k}^2j^2NG^2.
  \end{align*}
  Therefore, we have 
  \begin{align*}
  \Phi_k&=\sum_{j=0}^{\tau_1}\left\|\textbf{X}_{k\tau+j}-\bar{\textbf{X}}_{k\tau+j}\right\|^2_{F}\\
  &\leq2\sum_{j=0}^{\tau_1}\left\|\textbf{X}_{k\tau}-\bar{\textbf{X}}_{k\tau}\right\|^2_F+2\eta_{k}^2NG^2\sum_{j=0}^{\tau_1}j^2\\
  &\leq\frac{4(1+\theta^{-1})\eta_{k}^2\tau_1^3NG^2(1-p)^{\tau_2}}{(1+\theta)(1-p)^{\tau_2}-1}[((1+\theta)(1-p)^{\tau_2})^{k}-1]\\
  &+\frac{1}{3}\eta_{k}^2NG^2(\tau_1-1)(2\tau_1-1)\tau_1,
  \end{align*}
  where the last inequality follows from Lemma \ref{lem10}. Similarly, we can prove \eqref{120} by following the way of proving \eqref{119}.
\end{proof}

\begin{lemma}[\cite{stich2018local}]\label{lemma12}
  Let $\{a_t\}_{t\geq0}$, $a_t\leq 0$, $\{e_t\}_{t\geq0}$, $e_t\geq0$, be sequences satisfying
  \begin{align*}
  a_{t+1}\leq\left(1-\mu\eta_{t}\right)a_t-\eta_{t}e_tA+\eta_{t}^2B+\eta_{t}^3C,
  \end{align*}
  for $\eta_{t}=\frac{4}{\mu(a+t)}$ and constants $A>0$, $B$, $C\geq0$, $\mu>0$, $a>1$. Then
  \begin{align*}
  \frac{A}{S_T}\sum_{t=0}^{T-1}w_te_t\leq\frac{\mu a^3}{4S_T}a_0+\frac{2T(T+2a)}{\mu S_T}B+\frac{16T}{\mu^2 S_T}C,
  \end{align*}
  for $w_t=(a+t)^2$ and $S_T:=\sum_{t=0}^{T-1}w_t=\frac{T}{6}(2T^2+6aT-3T+6a^2-6a+1)\geq\frac{1}{3}T^3$.
\end{lemma}
\subsection{Proof of Proposition 2}
\begin{lemma}[Convergence of C-DFL]\label{lemma13}
  Under Assumption 1 and 2 for $p=\frac{\rho^2\delta}{82}$, if 
  \begin{align*}
  4L\eta_{k}^2\alpha_k-2\eta_{k}\tau_1+\frac{8}{3}(L+\mu)\eta_{k}^3L\alpha_k\tau_1^3\leq-\eta_{k},
  \end{align*}
  Algorithm 2 with step size $\eta_{k}=\frac{4}{\mu(a+k)}$, for parameter $a\geq 16\kappa$, $\kappa=\frac{L}{\mu}$ converges at rate
  \begin{align*}
  &F(\textup{\textbf{w}}_{avg}^{(K)})-F^*\leq\frac{\mu\tau_1a^3}{8S_K}\left\|\textup{\textbf{u}}_0-\textup{\textbf{u}}^*\right\|^2+\frac{4K(K+2a)}{\mu S_K}\frac{\tau_1\bar{\sigma}^2}{N}+\\
  &\frac{64K}{\mu^2S_K}\left(\frac{(L+\mu)\bar{\sigma}^2\tau_1}{3N}+D_1+D_2+D_3\right),
  \end{align*}
  where $\textup{\textbf{w}}_{avg}^{(K)}=\frac{1}{S_K}\sum_{k=0}^{K}w_k\textup{\textbf{u}}_k$ for weights $w_k=(a+k)^2$, $S_K=\sum_{k=0}^{K}w_k\geq\frac{1}{3}T^2$, and 
  \begin{align*}
  D_1&=\frac{2(1+\theta^{-1})L\tau_1^2G^2(1-p)^{\tau_2}}{(1+\theta)(1-p)^{\tau_2}-1}[((1+\theta)(1-p)^{\tau_2})^k-1]\\
  &+\frac{1}{3}L\tau_1^2G^2,\\
  D_2\!&=\!\frac{8(1+\theta^{-1})(L+\mu)\tau_1G^2(1-p)^{\tau_2}}{(1+\theta)(1-p)^{\tau_2}-1}[((1+\theta)(1-p)^{\tau_2})^k\!-\!1]\!\\
  &+\!\frac{4}{3}(L+\mu)\tau_1G^2,\\
  D_3&=\frac{(1+\theta^{-1})(L+\mu)\tau_1^4G^2(1-p)^{\tau_2}}{6((1+\theta)(1-p)^{\tau_2}-1)}[((1+\theta)(1-p)^{\tau_2})^k\!-\!1]\!\\
  &+\!\frac{1}{36}(L+\mu)\tau_1^4G^2.
  \end{align*}
\end{lemma}
\begin{proof}
  Combining the results of lemma \ref{lem8} and \ref{lem11}, by using $a_k$ to denote $\mathbb{E}\|\textbf{u}^{(k)}-\textbf{u}^*\|^2$, i.e., $a_k:=\mathbb{E}\|\textbf{u}^{(k)}-\textbf{u}^*\|^2$, we have
  \begin{align}\label{121}
  &a_{k+1}\leq\left(1-\frac{\mu\eta_{k}\tau_1}{2}\right)a_k+\frac{\eta_{k}^2\tau_1^2\bar{\sigma}^2}{N}+\notag\\
  &(F(\textbf{u}^{(k)})-F^*)\left[-2\eta_{k}\tau_1+4\eta_{k}^2L\tau_1^2 \alpha_k+\frac{8}{3}(L+\mu)L\eta_{k}^3\alpha_k\tau_1^3\right]\notag\\
  &\!+\!\!\eta_{k}^3\left[\frac{(L+\mu)\bar{\sigma}^2\tau_1^3}{3N}\right.\notag\\
  &\left.+\frac{8(1\!+\!\theta^{-1})L^2\eta_{k}\tau_1^4G^2(1\!-\!p)^{\tau_2}}{(1+\theta)(1-p)^{\tau_2}-1}(((1+\theta)(1-p)^{\tau_2})^k\!\!-\!\!1)\right.\notag\\
  &\left.+\frac{4}{3}L^2\eta_{k}\tau_1^4G^2+\!\right.\notag\\
  &\left.\frac{8(1+\theta^{-1})(L+\mu)\tau_1^3G^2(1-p)^{\tau_2}}{(1+\theta)(1-p)^{\tau_2}-1}(((1+\theta)(1-p)^{\tau_2})^k\!-\!1)\!\right.\notag\\
  &\left.+\!\frac{4}{3}(L+\mu)\tau_1^3G^2+\right.\notag\\
  &\left.\frac{8(1\!+\!\theta^{-1}\!)\!L^2\!(L\!+\!\mu)\eta_{k}^2\tau_1^6G^2(1\!-\!p)^{\tau_2}}{3(1+\theta)(1-p)^{\tau_2}-1}(((1\!+\!\theta)(1\!-\!p)^{\tau_2}\!\!)^k\!\!-\!\!1\!)\!\right.\notag\\
  &\left.+\!\frac{4}{9}L^2(L+\mu)\eta_{k}^2\tau_1^6G^2\right]\\
  &\leq\left(1-\frac{\mu\eta_{k}\tau_1}{2}\right)a_k+\frac{\eta_{k}^2\tau_1^2\bar{\sigma}^2}{N}+\notag\\
  &(F(\textbf{u}^{(k)})-F^*)\left[-2\eta_{k}\tau_1+4\eta_{k}^2L\tau_1^2 \alpha_k+\frac{8}{3}(L+\mu)L\eta_{k}^3\alpha_k\tau_1^3\right]\notag\\
  &+\eta_{k}^3\left[\frac{(L+\mu)\bar{\sigma}^2\tau_1^3}{3N}+\right.\notag\\
  &\left.\frac{2(1+\theta^{-1})L\tau_1^4G^2(1-p)^{\tau_2}}{(1+\theta)(1-p)^{\tau_2}-1}(((1+\theta)(1-p)^{\tau_2})^k-1)\right.\notag\\
  &\left.+\frac{1}{3}L\tau_1^4G^2+\right.\notag\\
  &\left.\!\frac{8(1+\theta^{-1})(L+\mu)\tau_1^3G^2(1-p)^{\tau_2}}{(1+\theta)(1-p)^{\tau_2}-1}(((1+\theta)(1-p)^{\tau_2})^k\!-\!1)\!\right.\notag\\
  \label{122}&\left.+\!\frac{4}{3}(L+\mu)\tau_1^3G^2+\right.\notag\\
  &\left.\frac{(1+\theta^{-1})(L+\mu)\tau_1^6G^2(1-p)^{\tau_2}}{6((1+\theta)(1-p)^{\tau_2}-1)}(((1+\theta)(1-p)^{\tau_2})^k\!-\!1)\!\right.\notag\\
  &\left.+\!\frac{1}{36}(L+\mu)\tau_1^6G^2\right],
  \end{align}
  where inequality \eqref{121} holds because $(\tau_1-1)(2\tau_1-1)\tau_1\leq2\tau_1^3$,
  \begin{align*}
  \sum_{j=0}^{\tau_1-1}j^5=\frac{(\tau_1-1)^2\tau_1^2(2\tau^2-2\tau_1-1)}{12}\leq\frac{\tau_1^6}{6},
  \end{align*} 
  which is in the last two terms, 
  and \eqref{122} holds because $\eta_{k}L\leq\frac{1}{4}$ (this holds, as $a\geq 16\kappa$).
  From Lemma \ref{lemma12}, we prove this lemma.
\end{proof}
\textbf{Proof of Proposition 2.} The proof of Proposition 2 follows from Lemma \ref{lemma13} with using the inequality $\mathbb{E}\mu\|\textbf{u}_0-\textbf{u}^*\|\leq 2G$ derived in \citep[Lemma 2]{rakhlin2012making} for the upper bound of the first term.
\section{Structure Description of CNNs}\label{appendix4}
In this part, we present the structure of the two CNN models which are trained based on MNIST and CIFAR-10, respectively. Table \ref{table1} shows the CNN models applied in the simulation.
\renewcommand\arraystretch{1.5}
\begin{table}[!h]  
  \caption{The structure of CNNs}
  \centering
  \label{table1}
  \renewcommand\arraystretch{1.5}
  \begin{threeparttable}
    \begin{tabular}{cp{5.8cm}}  
      
      \toprule[1pt]   
      
      \multicolumn{1}{c}{\textbf{CNN trained on MNIST}} & \multicolumn{1}{c}{\textbf{CNN trained on CIFAR-10}} \\ 
      
      \midrule   
      
      \multicolumn{1}{c}{Convolution layer\tnote{1} [1,16,3$\times$3]} &   \multicolumn{1}{c}{Convolution layer [3,64,5$\times$5]}  \\  
      
      \multicolumn{1}{c}{ReLU layer} &   \multicolumn{1}{c}{ReLU layer}  \\ 
      
      \multicolumn{1}{c}{MaxPool layer\tnote{2} [2$\times$2]} &   \multicolumn{1}{c}{MaxPool layer [3$\times$3]}  \\ 
      
      \multicolumn{1}{c}{Convolution layer [16,32,3$\times$3]} &   \multicolumn{1}{c}{Convolution layer [64,64,5$\times$5]}  \\ 
      
      \multicolumn{1}{c}{ReLU layer} &   \multicolumn{1}{c}{ReLU layer}  \\ 
      
      \multicolumn{1}{c}{MaxPool layer [2$\times$2]} &   \multicolumn{1}{c}{MaxPool layer [3$\times$3]}  \\ 
      
      \multicolumn{1}{c}{Dense layer\tnote{3} [32$\times$7$\times$7,10]} &   \multicolumn{1}{c}{Dense layer [64$\times$4$\times$4,384]}\\
      
      \multicolumn{1}{c}{} &   \multicolumn{1}{c}{ReLU layer}\\
      
      \multicolumn{1}{c}{} &   \multicolumn{1}{c}{Dense layer [384,192]}\\
      
      \multicolumn{1}{c}{} &   \multicolumn{1}{c}{ReLU layer}\\
      
      \multicolumn{1}{c}{} &   \multicolumn{1}{c}{Dense layer [192,10]}\\
      
      
      \bottomrule[1pt]  
      
    \end{tabular}
    \begin{tablenotes}
      \footnotesize
      \item[1] Convolution layer [input channels, output channels, kernel size]
      \item[2] Maxpool layer [kernel size]
      \item[3] Dense layer [size of each input sample, size of each output sample]
    \end{tablenotes}
  \end{threeparttable}
\end{table}

\end{document}